\documentclass{article}

\usepackage{microtype}
\usepackage{graphicx}
\usepackage{subcaption}
\usepackage{booktabs} 

\usepackage{amsmath}
\usepackage{amssymb}
\usepackage{mathtools}
\usepackage{amsthm}
\usepackage{thmtools}

\usepackage[hyphens]{url} 

\usepackage{times}

\usepackage[utf8]{inputenc} 
\usepackage[T1]{fontenc}    
\usepackage{hyperref}       
\usepackage{url}            
\usepackage{booktabs}       
\usepackage{amsfonts}       
\usepackage{nicefrac}       
\usepackage{microtype}      
\usepackage{xcolor} 

\usepackage{graphicx} 
\usepackage{amsmath}
\usepackage{color}
\newcommand{\grad}{\nabla}
\newcommand{\gradmod}{\widetilde{\nabla}}
\renewcommand{\vec}[1]{\mathbf{#1}}
\usepackage{url}
\usepackage{multirow}

\usepackage{array}
\usepackage{longtable}
\usepackage{soul}
\usepackage{graphicx} 
\usepackage{subcaption}
\usepackage{float}
\usepackage{xcolor}
\usepackage{hyperref}
\hypersetup{
   colorlinks = false,
   pdfborder={0 0 0}}

\allowdisplaybreaks

\newcommand{\enc}{\mathsf{Enc}}

\usepackage{tcolorbox}
\renewcommand\vec{\mathbf}

\newif\ifsubmission
\submissionfalse

\newif\ifsubmissionshort
\submissionshorttrue
\usepackage{adjustbox}

\ifsubmissionshort{
\theoremstyle{plain}
\newtheorem{theorem}{Theorem}[section]

\newtheorem{lemma}[theorem]{Lemma}
\newtheorem{corollary}[theorem]{Corollary}
\newtheorem{claim}[theorem]{Claim}
\theoremstyle{definition}
\newtheorem{definition}[theorem]{Definition}
\newtheorem{assumption}[theorem]{Assumption}
\theoremstyle{remark}
\newtheorem{remark}[theorem]{Remark}}
\else{
\newtheorem{definition}{Definition}[section]
\newtheorem{theorem}{Theorem}

\newtheorem{claim}[theorem]{Claim}

\newtheorem{assumption}{Assumption}
\newtheorem{remark}{Remark}
}
\fi

\usepackage[round]{natbib}

\usepackage{environ}
\newcommand{\acksection}{\section*{Acknowledgments and Disclosure of Funding}}
\NewEnviron{ack}{%
  \acksection
  \BODY
}

\usepackage[textsize=tiny]{todonotes}
\usepackage[preprint]{icml2026}

\def\commentson{1} 

\ifnum\commentson=1
\newcommand{\yv}[1]{\textcolor{red}{Yvonne: #1}}
\newcommand{\my}[1]{\textcolor{orange}{Mingyu: #1}}
\newcommand{\dnote}[1]{\textcolor{teal}{Dana: #1}}
\newcommand{\anti}[1]{\textcolor{magenta}{Antigoni: #1}}
\newcommand{\wm}[1]{\textcolor{blue}{ Min: #1}}
\else
\newcommand{\yv}[1]{}
\newcommand{\my}[1]{}
\newcommand{\dnote}[1]{}
\newcommand{\anti}[1]{}
\newcommand{\wm}[1]{}
\fi

\begin{document}

\twocolumn[
\icmltitle{Revisiting ML Training under Fully Homomorphic Encryption: Convergence Guarantees, Differential Privacy, and Efficient Algorithms}
\icmltitlerunning{Revisiting ML Training under FHE: Convergence Guarantees, Differential Privacy, and Efficient Algorithms}

\begin{icmlauthorlist}
\icmlauthor{Yvonne Zhou}{umd}
\icmlauthor{Mingyu Liang}{umd}
\icmlauthor{Ivan Brugere}{jp}
\icmlauthor{Danial Dervovic}{jp}
\icmlauthor{Yue Guo}{cryp,jp}
\icmlauthor{Antigoni Polychroniadou}{cryp,jp}
\icmlauthor{Min Wu}{umd}
\icmlauthor{Dana Dachman-Soled}{umd}
\end{icmlauthorlist}

\icmlaffiliation{umd}{University of Maryland, College Park, MD 20742}
\icmlaffiliation{jp}{J.P. Morgan AI Research, New York, NY, 10017 }
\icmlaffiliation{cryp}{AlgoCRYPT CoE }
\icmlcorrespondingauthor{Yvonne Zhou}{skyzhou@umd.edu}
\vskip 0.3in
]
\printAffiliationsAndNotice{}

\begin{abstract}
We present the first theoretical convergence analysis of machine learning training under fully homomorphic encryption (FHE), combined with a differentially private (DP) training algorithm tailored to encrypted computation. Our approach improves computational efficiency over standard differentially private gradient descent (DP-GD) while achieving comparable utility. In particular, we prove convergence of approximate gradient descent using polynomial approximations of activation and loss functions, which are required for FHE compatibility. To preserve privacy in downstream tasks, we integrate differential privacy without relying on costly per-sample gradient clipping, enabling scalable encrypted learning. We also provide data-independent hyperparameter selection and theoretically grounded strategies for polynomial approximation which can be of independent interest. Together, these contributions advance the feasibility of efficient, private, and secure machine learning on sensitive data.

\end{abstract}

\section{Introduction}
\sloppy

As machine learning (ML) is increasingly applied to sensitive domains, protecting
the privacy and security of training data has become critical, especially when
resource-constrained clients outsource training to external servers. In such
settings, privacy risks arise from: 
(1) a potentially
semi-honest \emph{server} that follows the protocol but attempts to infer
information about the data; 
and
(2) \emph{end-users} who, through black-box or white-box access to a trained
model, may extract information about the training data.

To address these threats, we combine \emph{Fully Homomorphic Encryption (FHE)}
and \emph{Differential Privacy (DP)} to provide end-to-end privacy guarantees.
FHE ensures that the server performs training entirely on encrypted data and
never observes raw inputs, while DP limits information leakage to 
end-users after deployment, mitigating attacks such as model
inversion~\cite{fredrikson2015model}. Together, FHE and DP offer strong
protection throughout the ML pipeline.

As illustrated in Figure~\ref{fig:FHE_DP-GD}, the workflow proceeds as follows. The client generates an FHE key pair and encrypts both its dataset and pre-sampled noise, which are then sent to the server. The server performs homomorphic DP-GD, adding the encrypted noise to the encrypted gradients at each iteration, and returns the final encrypted model to the client. A semi-honest server learns nothing about the client’s data, while the client obtains DP guarantees against downstream users, even after decrypting and releasing the model.
This framework naturally extends to a multi-client setting using cryptographic tools such as multi-key FHE~\cite{ananth2020multi, lopez2012fly, mukherjee2016two}, where each client encrypts its data and noise under its own public key.\footnote{Each client independently contributes noise, so the variance grows linearly with the number of clients. Alternatively, the server could sample noise homomorphically under FHE, but this is significantly more computationally expensive.} The protocol requires only a single round of interaction with the server, and
a single round of interaction among the clients to jointly decrypt the final model.

\begin{figure}[h]
    \centering
    \includegraphics[width=0.8\linewidth]{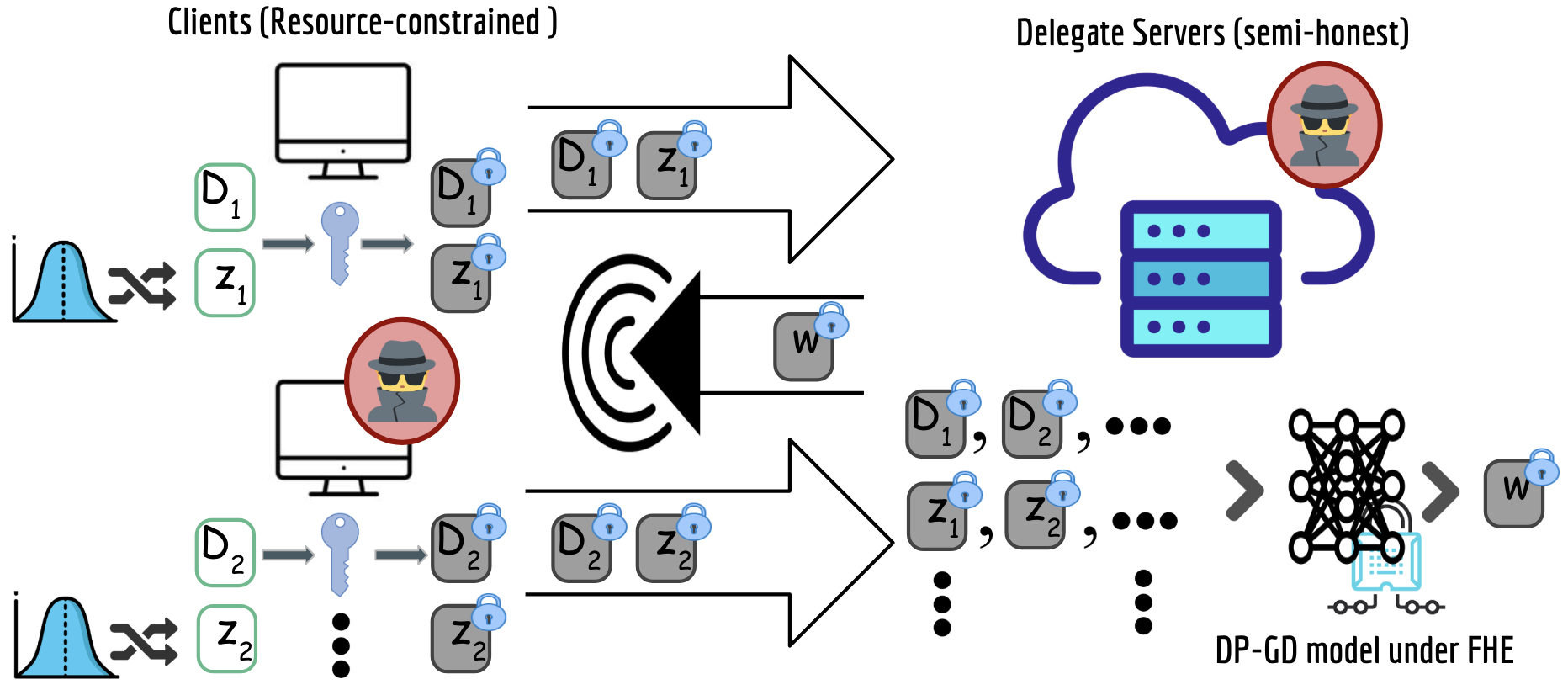}
    \vskip 0.1in
    \caption{Secure and DP training with encrypted dataset}
    \vskip -0.1in
    \label{fig:FHE_DP-GD}
\end{figure}

When performing gradient descent under FHE, the true gradient is often replaced with a polynomial approximation so that the arithmetic operations required to compute the gradient are compatible with the available FHE operations of ``$+$'' and ``$\times$'' over the native FHE ring.
To illustrate our methodology, we concretely examine how polynomial approximation is employed in the case of neural networks.

\vspace{-0.5em}\paragraph{Polynomial approximation for neural networks.}
Given a dataset \(D\) of \(N\) samples \((\vec{x}, y)\), consider the objective
$f(\vec{w}) = \frac{1}{N} \sum_{(\vec{x}, y)\in D} f_{(\vec{x}, y)}(\vec{w})$,
where $f_{(\vec{x}, y)}(\vec{w}) = L_y\!\left(M_{\vec{w}}(\vec{x})\right)$,
\(M_{\vec{w}}\) is an \(L\)-layer feedforward  neural network with activation \(\psi\),
\(L_y\) is the loss for label \(y\), and \(\nabla f(\vec{w})\) denotes the
corresponding gradient.

A polynomial approximation to \(\nabla f(\vec{w})\) is obtained by rerunning
backpropagation while replacing each occurrence of \(\psi\) and its derivative
\(\psi'\) (except in the final layer) with polynomial approximations \(p\) and
\(p'\). In the final layer, the term
$\psi'(z_L)\, L_y'(a_L) = (L_y \circ \psi)'(z_L)$,
where \(z_L\) and \(a_L\) denote the input and output of the final activation,
is replaced by a polynomial approximation \(\tilde{p}'_y\) of the combined
derivative. 
Any continuous function can be approximated arbitrarily well by a polynomial of
sufficient degree~\cite{bernstein1912proof,Roulier70}.
However, such approximations guarantee bounded error only on specified
\emph{closed intervals}, which prior work typically selected heuristically.
A central goal of this work is to identify approximation intervals together with
formal guarantees that all inputs encountered during gradient descent remain
within them.

Crucially, the substitution described above yields a vector-valued function that is itself the
gradient of another objective \(g\) (see Section~\ref{sec:poly_grad}). Thus,
running approximate gradient descent using the substituted gradient is
equivalent to running exact gradient descent on \(g\).

\vspace{-0.5em}\paragraph{Convergence of Approximate Gradient Descent.}

Although approximate gradient descent has performed well in practice~\cite{kim2018secure, han2018efficient}, its theoretical justification remains unclear. In particular, even if the gradients of the original and approximate objective functions, $\nabla f$ and $\nabla g$, are close in norm (i.e., within $\zeta$), this does not guarantee that essential properties of $f$ are preserved by $g$. For instance, $g$ may no longer be convex, even if $f$ is, which implies that gradient descent on $g$ may converge to a local minimum far from the global minimum of $f$.
Thus, it is not immediately evident what can be said about the convergence of gradient descent to the minimum of $f$ when optimization is performed with respect to the surrogate objective $g$. 
\vspace{-0.5em}\paragraph{Privacy-Preserving Learning via FHE and DP}
\emph{Fully Homomorphic Encryption} (FHE) is an advanced cryptographic primitive that enables computations to be performed directly on encrypted data, producing encrypted outputs that, upon decryption, match the results of the corresponding computations carried out on the plaintext data. FHE enables a client to outsource the training process to an untrusted server without interaction: the client sends encrypted data to the server, which computes directly on ciphertexts and returns the final encrypted model weights.
While FHE ensures the server learns nothing about the training data during the computation, it does not prevent information leakage \emph{after} training. 
Once the client decrypts the final model, it may be used in downstream tasks or released publicly, potentially exposing sensitive information about the training data.

One might consider releasing the encrypted model together with the decryption key to trusted users. However, this reduces the setting to a closed-access deployment model and is incompatible with our goal of a \emph{public-release} setting for broad use.
To address this limitation, we incorporate \emph{differential privacy} (DP), which provides formal guarantees that the released model does not overly depend on any single training example, and remains secure under arbitrary post-processing, without requiring trusted access or cryptographic keys. A widely used approach for achieving differential privacy in ML training is \emph{differentially private gradient descent}  (DP-GD),  where \emph{gradient perturbation} is performed--i.e.~Gaussian noise is added to the gradient at each iteration of the optimization process (see Algorithm~\ref{fig:DP_grad_desc}).

In our framework, the client encrypts the training data, and the server applies an \emph{approximate} version of DP-GD, in which gradients of the original loss function $f$, 
are replaced with gradients of a polynomial approximation $g$, 
as discussed above. We first investigate 
the convergence guarantees of 
this noisy, approximate training procedure.

DP-GD relies on an additional step, \emph{gradient clipping}, which limits the gradient's sensitivity by scaling it whenever its norm exceeds
a chosen threshold. This operation controls the noise
required to achieve DP. 
Performing gradient clipping \emph{homomorphically} under FHE
is highly inefficient \footnote{The more advanced \emph{adaptive gradient clipping}~\cite{andrew2021differentially} is even more expensive under FHE due to the additional logic required for gradient monitoring and threshold updates.}. It entails evaluating operations such as comparisons,
square roots, and divisions directly on ciphertexts, each of which is costly
due to bit-level computation and ciphertext expansion. Even after replacing
these operations with polynomial approximations, the resulting circuit depth
(the primary measure of FHE complexity) is about 24 per iteration--far
larger than the depth of roughly 8 needed to implement an approximate version
of standard (no-DP) gradient descent.

To overcome this limitation, we design an FHE-tailored training algorithm that avoids gradient clipping while providing formal DP guarantees and comparable model utility, significantly reducing computational overhead and improving scalability.

\vspace{-0.8em}\paragraph{Gradient Perturbation vs.\ Output Perturbation.}
An alternative to gradient perturbation (DP-GD) is
\emph{output perturbation} (Output-GD), which enforces DP by adding calibrated
noise to the final model parameters \cite{dwork2006calibrating,chaudhuri2011differentially}.
Both prior work \cite{yu2019gradient} and our experiments (Figure~\ref{fig:dp_no_fhe_accuracy} and Table~\ref{tb:no-clipping_nofhe}) show that Output-GD yields significantly worse utility under the same privacy budget, since it injects a single large noise term rather than distributing noise across iterations, which enables better composition, noise averaging, and stability. Further, the accuracy of Output-GD models exhibit high variance across
runs, while gradient-perturbed models achieve consistent performance.
Finally, Output-GD’s sensitivity, and thus its noise scale, depends on the maximum weight magnitude attained during training. Under FHE, this would require either an expensive encrypted \textbf{Maximum} computation or sending
the encrypted squared norms of all intermediate weights to the client for post-processing, incurring substantial communication overhead.
and retaining privacy guarantees only in the single-client setting. 
For these reasons, we adopt gradient perturbation to achieve strong DP guarantees with higher practical utility.

\subsection{Related Work}
\label{sec:related_works}
\vspace{-0.5em}\paragraph{Secure gradient descent training using FHE}
Prior work has explored FHE for secure machine learning. \cite{gilad2016cryptonets} study encrypted neural network inference using FHE. \cite{kim2018secure} implement logistic regression under the BGV scheme, \cite{han2018efficient} propose an efficient logistic regression method with CKKS, and \cite{mihara2020neural} implement neural network training using CKKS. These studies focus on empirical evaluations and lack theoretical guarantees.

\vspace{-0.7em}\paragraph{FHE + DP in federated learning}
Recent work combines homomorphic encryption (HE) and differential privacy (DP) for secure machine learning \cite{8241854,aziz2023exploring,10320195}, typically applying HE to model updates (e.g., gradients) and requiring interactive client–server rounds. In contrast, our method encrypts the training data itself, enabling secure outsourcing to untrusted servers, supporting clients with limited resources. This allows the server to train large-scale models entirely without client participation, while preserving data privacy.

\vspace{-0.7em}\paragraph{Outsourcing DP computation}
There is growing interest in outsourcing differential privacy (DP) computation.
For example, \cite{dagher2020privacy} propose a framework for confidential
range-count queries using attribute-based encryption and DP, while
\cite{ye2020secure} combine order-preserving encryption with DP to support
secure outsourced data release. However, these works focus on DP data querying
rather than machine learning tasks.

\subsection{Our Contributions} 

To the best of our knowledge, this work presents the first theoretical convergence analysis of machine learning training within the context of FHE, along with a secure, DP training algorithm tailored for FHE. Our algorithm offers significantly greater efficiency than standard DP-GD under FHE, while maintaining comparable utility. Our key contributions are as follows:

\noindent
\textbf{Theoretical Convergence of Approximate Gradient Descent in FHE:}
While prior work has primarily focused on implementation-level designs for approximate gradient descent trainning under FHE, we provide the first formal convergence analysis of such training, where activation and loss functions are replaced by polynomial approximations$-$a necessary adaptation for compatibility with FHE. (Sec.~\ref{sec:No-DP_converge})

\noindent
\textbf{Differential Privacy for Outsourced Training:}
Unlike existing FHE-based methods that focus solely on securely outsourcing standard training algorithms, we propose integrating differential privacy to mitigate information leakage risks from end users. Additionally, we offer the first comprehensive convergence analysis of DP training under approximate gradient descent. (Sec.~\ref{sec:dp_converge})

\noindent
\textbf{Technical Innovations for DP-ML in FHE:}
We propose a method for provably DP gradient descent under FHE. We modify the objective function to allow strictly upper-bounding the norm of the iterates $\|\vec{w}_i\|$, which fixes the input range for polynomial approximations and ensures their approximation error is controlled. This allows us to bound the gradient sensitivity at each iteration and calibrate the noise correctly, while avoiding the costly gradient-clipping step, which is particularly expensive under FHE.

Crucially, naively replacing DP-GD components with polynomial approximations, regardless of whether clipping is used, can break DP guarantees, because the gradient sensitivity cannot be rigorously bounded. Polynomial approximations are only accurate over a prescribed input interval and can exhibit unbounded error outside it. Clipping does not resolve this issue, as the clipping operations themselves rely on polynomial approximations that can diverge beyond the interval. Prior work selected these intervals heuristically, with no guarantee that all inputs remain within bounds. In contrast, our approach provides a rigorous DP guarantee without heuristic assumptions (Sec.~\ref{sec:no_clip}).

Our approach enables scalable DP training under FHE with markedly reduced computational cost, formal privacy guarantees, and comparable utility. 
Relative to a standard (unapproximated) DP-GD baseline trained on plaintext data, our FHE-trained models incur at most a 1.82\% drop in accuracy and a 0.78\% drop in AUC, while reducing per-iteration multiplicative depth by over $2.5\times$. 
In fact, despite the fact that we use a modified objective function, our training time is essentially the same as that of Output-GD, which runs standard (no-DP) GD under FHE
(Sec.~\ref{sec:fhe_experiments}, Table~\ref{tb:no-clipping_fhe} and \ref{tb:no-clipping_nofhe}).

\noindent
\textbf{Polynomial Approximation Strategies:}
Building on our convergence analysis, we derive practical guidelines for selecting polynomial approximations for activation and loss functions. We validate our findings through extensive experiments. In our experiments, we observe that incorrectly setting the approximation intervals can lead to divergence in both non-DP and DP settings. Our techniques provide rigorous guarantees for these intervals, whereas prior algorithms require data-dependent tuning. We further illustrate the theoretical and empirical trade-offs across different polynomial approximation strategies. (Sec.~\ref{sec:guidance_of_approx}, Appendix~\ref{sec:guidance})

\noindent
\textbf{Data-Independent Hyperparameter Selection:}
Our analysis enables the principled selection of key hyperparameters (e.g.,learning rate), without requiring access to the underlying dataset. 
Prior work did not address how to set these parameters in practice in a data-independent manner while guaranteeing convergence. (Sec.~\ref{sec:guidance_of_approx}, Appendix~\ref{sec:guidance})

\section{Notation and Background}
\label{sec:background}

We use boldface to represent vectors.
$\|\vec{w}\|$ denotes the $\ell_2$ norm of a
vector $\vec{w}$. 
$\vec{I}_m$ denotes the identity matrix of dimension $m$.
$[t]$ denotes the set of natural numbers less than or equal to $t$. $\log$ denotes natural log. A database \(D\) is a multiset of records \((\vec{x}, y)\), where
\(\vec{x} \in [-1,1]^m\) and \(y \in \{0,1\}\). $N$ denotes its size.

\subsection{Properties of Functions}\label{sec:objetive_properties}

\begin{definition}[$\beta$-smooth] A function $f$ is $\beta$-smooth, if $\langle \grad f(\vec{w})-\grad f(\vec{v}), \vec{w}-\vec{v} \rangle \leq \beta \|\vec{w}-\vec{v}\|^2$ 
\end{definition}

\begin{definition}[$\mu$-strong convexity] A function $f$ is $\mu$-strongly convex, if $\langle \grad f(\vec{w})-\grad f(\vec{v}), \vec{w}-\vec{v} \rangle \geq \mu \|\vec{w}-\vec{v}\|^2$     
\end{definition}

\begin{definition}[Expected Curvature $\nu$~\cite{yu2019gradient}] \label{def:expected_curvature}
This notion yields parameters more closely aligned with experimental behavior.
$f$ achieves $\nu_{\sigma^2}$-expected curvature around the optimum $\vec{w}^*$ with variance $\sigma^2$ if
for all $\vec{w}$,
\[\mathbb{E}[
\langle \grad f(\vec{w} + \vec{\chi}), \vec{w} + \vec{\chi} - \vec{w}^* \rangle] \geq  \nu_{\sigma^2} \mathbb{E}[ \|\vec{w} + \vec{\chi} - \vec{w}^*\|^2],
\]
where the expectation is with respect to $\vec{\chi} \sim \mathcal{N}(\vec{0}, \sigma^2 \vec{I}_m)$.
\end{definition}


\subsection{Differentially-Private Gradient Descent}
\label{sec:Differential_Privacy}



\begin{definition}[$(\epsilon, \delta)$-Differential Privacy]
A randomized mechanism $\mathcal{M:D \rightarrow R}$ satisfies $(\epsilon, \delta)$-differential privacy if for any two adjacent databases $D, D' \in \mathcal{D}$ (i.e.~databases that differ in exactly one record) and for any subset of outputs $\mathcal{S} \subseteq \mathcal{R}$  it holds that
$Pr[\mathcal{M}(D)\in \mathcal{S}]\leq e^{\epsilon} Pr[\mathcal{M}(D')\in \mathcal{S}]+\delta$.
\end{definition}

\noindent
DP learning is commonly implemented via
Differentially Private Gradient Descent (DP-GD) and its stochastic variants
(DP-SGD) \cite{Abadi_2016,private_SGD_1,private_SGD_2}. 
Given the objective
$f(\vec{w}) = \frac{1}{N} \sum_{(\vec{x}, y) \in D} f_{(\vec{x},y)}(\vec{w})$,
the per-sample gradient norm is $\|\nabla f_{(\vec{x},y)}(\vec{w})\|$. Since this
quantity may be large or difficult to bound, DP-GD clips each per-sample
gradient to a fixed threshold \(C\), which determines the noise scale \(\sigma\) required for DP. 
Clipping
is defined as

\vspace{1mm}

$
\mathsf{Cp}(\nabla f_{(\vec{x}, y)}(\vec{w}), C)
= \min\!\left\{1, \frac{C}{\|\nabla f_{(\vec{x}, y)}(\vec{w})\|}\right\}
\nabla f_{(\vec{x}, y)}(\vec{w}).
$

\vspace{1mm}

\noindent
\sloppy
Thus, in DP-GD, noise \(\boldsymbol{\chi}_i \sim \mathcal{N}(0,\sigma^2\vec{I}_m)\) is sampled at each step and
model updates become:
$\vec{w}_{i+1}
= \vec{w}_i
- \eta \left (\frac{1}{N} \sum_{(\vec{x}, y) \in D}\mathsf{Cp}(\nabla f_{(\vec{x}, y)}(\vec{w}_i), C) + \boldsymbol{\chi}_i\right)$.

\noindent
\textbf{Breakdown of DP guarantees under polynomial approximation.}
Even when an upper bound on
\(\|\nabla f_{(\vec{x}, y)}(\vec{w})\|\) is available, no such guarantee holds for
the polynomially approximated gradient
\(\|\nabla g_{(\vec{x}, y)}(\vec{w})\|\).
Polynomial approximations have bounded error only over fixed input
intervals and may behave arbitrarily outside them.
Prior to this work, there were no formal guarantees that the inputs encountered
during gradient descent remained within these intervals.

\subsection{Secure Machine Learning via FHE}
\label{sec:FHE}
Fully homomorphic encryption (FHE)~\cite{gentry2009fully,cryptoeprint:2011/277,cryptoeprint:2012/144,cryptoeprint:2014/816,cheon2017heaan} 
allows computations to be performed on encrypted data.
Given encryptions $\enc_{\mathsf{pk}}(m_1)$ and $\enc_{\mathsf{pk}}(m_2)$, FHE allows computation of
$\enc_{\mathsf{pk}}(m_1 + m_2)$ and $\enc_{\mathsf{pk}}(m_1 \cdot m_2)$ and, by composition, arbitrary arithmetic circuits, all \emph{without decryption}.
FHE-based training is computationally expensive, with cost dominated by the circuit’s multiplicative depth.

\subsection{Polynomial approximation of Neural Networks} \label{sec:poly_grad}
FHE supports 
addition and multiplication over a native ring, allowing the computation of polynomials. 
However, the true gradient
\(\nabla f\) typically contains non-polynomial components that cannot be evaluated directly
under FHE. Consequently, training under FHE replaces the true gradient with a
polynomial approximation.

\sloppy
\noindent
\textbf{The approximate gradient is itself a gradient.}
Recall that for neural net training, objective function $f$ has the form 
$f(\vec{w}) = \frac{1}{N} \sum_{(\vec{x}, y) \in D} f_{(\vec{x}, y)}(\vec{w})$,
where
$f_{(\vec{x}, y)}(\vec{w}) =  L_y(\psi^{L}(\vec{W}^{L}\cdots \boldsymbol\psi(\vec{W}^{1}\vec{x}))\cdots ))),$
$\vec{w} = \vec{W}^1, \ldots, \vec{W}^L$ and
matrices $\vec{W}^\ell$ are the weights in layer $\ell$,
and each $\boldsymbol\psi$ is a vector of activations, $\psi$, of appropriate dimension,
applied in the hidden layers,
and $L_y$ is a loss function that depends on the label $y$. Then
\begin{align*}
& \grad f(\vec{W}^\ell) 
= \frac{1}{N} \sum_{(\vec{x},y) \in D} [\boldsymbol \psi'(\vec{z}_\ell) \odot (\vec{W}^{\ell+1})^T \cdot \\  
&\boldsymbol \psi'(\vec{z}_{\ell+1})
\odot \cdots \odot (\vec{W}^L)^T \cdot \psi'(z_L) \cdot L'_{y}(a_L)] \vec{a}_{\ell-1}^T,
\end{align*}
where for $j \in [L-1]$, $\vec{z}_{j}$
(resp.~$\vec{a}_j$) are the inputs 
(resp.~activations)
of layer $j$,
$z_L$ (resp.~$a_L$) is the input
(resp.~activation) of layer $L$,
and
$\odot$ denotes a Hadamard product.
Replacing the activation functions $\boldsymbol \psi$
with polynomial approximations $\vec{p}$, replacing $\psi'(z_L) \cdot L'_{y}(a_L) = (L_{y} \circ \psi)'(z_L)$ with a polynomial approximation
$\tilde{p}'_y$,
and setting the polynomial $\tilde{p}_y$
to be an antiderivative 
of $\tilde{p}'_y$ 
yields the gradient of another function $g$:
\begin{align*}
\grad g(\vec{W}^l)
= &\frac{1}{N} \sum_{(\vec{x},y) \in D} [\vec{p}'(\widetilde{\vec{z}}_\ell) \odot (\vec{W}^{\ell+1})^T \cdot \vec{p}'(\widetilde{\vec{z}}_{\ell+1})\\
&\odot \cdots \odot (\vec{W}^L)^T \cdot \tilde{p}'_y( \widetilde{z}_L)] \widetilde{\vec{a}}_{\ell-1}^T,
\end{align*}
where
$g(\vec{w}) = \frac{1}{N} \sum_{(\vec{x},y) \in D} g_{(\vec{x},y)}(\vec{w})$
and
$g_{(\vec{x},y)}(\vec{w}) = \tilde{p}_y(\vec{W}^{L}\cdots \vec{p}(\vec{W}^{1}\vec{x}))\cdots ))).$

For single-layer neural networks,
$\grad f$ has the form:
$\grad f(\vec{w}) = \frac{1}{N} \sum_{(\vec{x}, y) \in \mathcal{D}}
(L_y \circ \psi)'(\langle \vec{w}, \vec{x} \rangle)\cdot \vec{x}$
%
and $\grad g$ has the form:
$
\grad g(\vec{w}) = \frac{1}{N} \sum_{(\vec{x}, y) \in \mathcal{D}}\tilde{p}'_y(\langle \vec{w}, \vec{x} \rangle) \cdot \vec{x}$.

\section{Convergence of No-DP Approximate GD}
\label{sec:No-DP_converge}
 Recall that we replace the gradient vector, $\grad f(\vec{w})$, corresponding to objective function $f$ with a vector of polynomial approximations, and that 
for functions $f$ of interest, this approximate gradient is the \emph{exact} gradient of a different function $g$
(See Section~\ref{sec:FHE}).
Consequently, we effectively run gradient descent with objective $g$.
See Algorithm~\ref{fig:grad_desc}.

Let convex set $\mathcal{S} \subseteq \mathbb{R}^m$ be the domain of $f, g, \grad f, \grad g$.
We assume $\grad g$ and $g$ satisfy the following:
\begin{assumption}[Bounded Gradient Distance]\label{ass:grad_dist}
\leavevmode\par
{\centering
$\displaystyle
\forall \vec{w} \in \mathcal{S},
\|\grad f(\vec{w}) - \grad g(\vec{w}) \| \leq \zeta.$
\par}
\end{assumption}

\begin{assumption}[Bounded Objective Distance]\label{ass:obj_dist}
\leavevmode\par
{\centering
$\displaystyle
\forall \vec{w} \in \mathcal{S},
| f(\vec{w}) - g(\vec{w}) | \leq \alpha.$
\par}
\end{assumption}

For all \(\vec{w}, \vec{w}' \in \mathcal{S}\), a bound on
\(f(\vec{w}) - f(\vec{w}') - (g(\vec{w}) - g(\vec{w}'))\) follows directly
from Assumption~\ref{ass:grad_dist}.
Consequently, Assumption~\ref{ass:obj_dist} 
can be omitted,
see Appendix~\ref{sec:alpha}.

\begin{algorithm}[tb]
  \caption{Approximate Gradient Descent.}
  \label{fig:grad_desc}
  \begin{algorithmic}
    \STATE {\bfseries Input:} Running steps $T$, learning rate
   $\eta_{\ref{fig:grad_desc}}$.
    \STATE Initialize $\vec{w}_{g,0} = \vec{0}$.
    \FOR{$i = 0$ to $T-1$}
    \STATE $\vec{w}_{g, i+1} \gets \vec{w}_{g,i} - \eta_{\ref{fig:grad_desc}}\grad g(\vec{w}_{g,i})$, 
    \ENDFOR
    \STATE {\bfseries Onput:}$\vec{w}_{g,T}$
  \end{algorithmic}
\end{algorithm}

\begin{restatable}{theorem}{ConvergenceNoDP} \label{thm:npdp_converge}
Suppose objective functions $f$ and $g$ satisfy Assumptions~\ref{ass:grad_dist} and 
\ref{ass:obj_dist}, $f$ is $\mu$-strongly convex, $g$ is $\beta_g$-smooth, and the learning rate $\eta_{\ref{fig:grad_desc}} = \frac{1}{\beta_g}$.
Fix $\rho > 0$ and let $T = \frac{2\beta_g(g(\vec{w}_{g,0}) - L)}{\rho^2}$, where $L \leq g(\vec{w}_{g, T})$.
Then we have the following guarantees on convergence of the output of Algorithm~\ref{fig:grad_desc}, $\vec{w}_{g, T}$, to the optimum, $\vec{w}^*_f$, of 
$f$:
\[
f(\vec{w}_{g,T}) - f(\vec{w}^*_f) \leq 2\alpha + \frac{\rho^2 + 2\zeta \rho + \zeta^2}{2\mu}, \text{ and,}
\] 
    \[
\|\vec{w}_{g,T} - \vec{w}^*_f\|^2 \leq  \frac{4\alpha}{\mu} + \frac{\rho^2 + 2\zeta \rho + \zeta^2}{\mu^2}.
\]
\end{restatable}

\begin{remark}
If objective function $g$ has a global minimum, $\vec{w}^*_g$, then we can set $L = g(\vec{w}^*_g)$ in the above Theorem.
\end{remark}

The proof of Theorem~\ref{thm:npdp_converge} can be found in Appendix~\ref{sec:no_dp_converge}.

\section{Convergence of DP-Approximate GD}
\label{sec:dp_converge}

We consider DP gradient descent 
with ``approximate'' objective function $g$,
as specified in Algorithm~\ref{fig:DP_grad_desc}.
Since random noise is added to the gradient in each iteration, we employ an ``average-case'' analogue of strong convexity---called \emph{expected curvature}~\cite{yu2019gradient} (see Definition~\ref{def:expected_curvature}) and denoted by $\nu$---in the analysis. 
As in Section~\ref{sec:No-DP_converge}, we employ Assumption~\ref{ass:grad_dist}.

\begin{restatable}{theorem}{ConvergenceDP} \label{thm:dp_converge}
Suppose objective functions $f$ and $g$ satisfy Assumption~\ref{ass:grad_dist}, and $f$ is $\beta_f$-smooth with $\nu$ expected curvature, then the output of Algorithm~\ref{fig:DP_grad_desc}, $\vec{w}_T$, satisfies: 
\begin{align*}
&\mathbb{E}[\|\vec{w}_{g,T} - \vec{w}^*_f\|^2] \\ \leq &\left ( 1 - \frac{\eta_{\ref{fig:DP_grad_desc}} \nu}{2} \right )^T \|\vec{w}_{g,0} - \vec{w}^*_f\|^2 + \frac{2\zeta^2}{\nu} \left ( \frac{1}{\nu} + 1 \right) +
\frac{2\eta_{\ref{fig:DP_grad_desc}} m \sigma^2}{\nu}.
\end{align*}
\end{restatable}

The proof appears in Appendix~\ref{sec:dp-converge-proof}.

\begin{algorithm}[tb]
    \caption{Approximate Gradient Descent with DP-noise.}
    \label{fig:DP_grad_desc}
  \begin{algorithmic}
    \STATE {\bfseries Input:} Running steps $T$, learning rate
   $\eta_{\ref{fig:DP_grad_desc}}$, noise standard deviation $\sigma$.
    \STATE Initialize $\vec{w}_{g, 0} = \vec{0}$
    \FOR{$i = 0$ to $T-1$}
    \STATE $\vec{w}_{g, i+1} \gets \vec{w}_{g, i} - \eta_{\ref{fig:DP_grad_desc}}(\grad g(\vec{w}_{g, i})+ \boldsymbol\chi_i)$,
    \STATE where $\boldsymbol\chi_i \sim \mathcal{N}(0, \sigma^2 \vec{I}_m)$.
    \ENDFOR
    \STATE {\bfseries Ouput:}$\vec{w}_{g,T}$
  \end{algorithmic}
\end{algorithm}

\section{DP-Approximate Gradient Descent Algorithm without Clipping} \label{sec:no_clip}


We 
consider here 
smooth, convex cost functions
of the form
$f(\vec{w}) = \frac{1}{N} \sum_{(\vec{x}, y) \in D} \phi(\langle \vec{w}, \vec{x} \rangle, y) = \frac{1}{N} \sum_{(\vec{x}, y) \in D} y \phi_{1}(\langle \vec{w}, \vec{x} \rangle)
+ (1-y)\phi_{0}(\langle \vec{w}, \vec{x} \rangle)$.
\footnote{For notational simplicity, we write $\phi$ in place of $L_y \circ \psi$ as defined in Section~\ref{sec:background}.}
We denote by $\phi'(z,y)$ 
the function
$\phi'(z,y) := y \cdot \frac{d \phi_1(x)}{dx}(z) + 
 (1-y) \cdot \frac{d \phi_0(x)}{dx}(z)$.
We assume $\phi'_{max} := \sup \{|\phi'(z, y)| : z \in (-\infty, \infty), y \in \{0,1\}\} < \infty$.


We begin by modifying the original objective \(f\) to
$f^{+B}_\kappa(\vec{w}) \coloneqq f(\vec{w}) - \lambda B_\kappa(\Theta - \|\vec{w}\|^2)$,
where the additional term is a \emph{barrier function}, a standard tool for
enforcing constraints in optimization.
We define
\(B_\kappa(\Theta - \|\vec{w}\|^2) \coloneqq \ln(\Theta - \|\vec{w}\|^2)\) on the
domain \(\{\vec{w} : \|\vec{w}\|^2 \le (1-\kappa)\Theta\}\),
and use it to control the magnitude of the iterates \(\vec{w}_i\) in
Algorithm~\ref{fig:modified_grad_desc}.
Specifically, the barrier term causes the modified objective to approach $\infty$ as $\|\vec{w}\|^2 \to \Theta$.

Bounding \(\|\vec{w}_i\|\) immediately yields a bound on the hypothesis values,
since \(\langle \vec{w}_i, \vec{x} \rangle \le \sqrt{m}\,\|\vec{w}_i\|\).
This bound is essential for differential privacy: it allows us to fix an
approximation interval for \(\phi'(z,y)\) and guarantees that all
\(z=\langle \vec{w}_i, \vec{x} \rangle\) encountered during gradient descent lie
within this interval.
Consequently, we can bound the sensitivity of the \emph{approximate} gradient
using the known sensitivity of the true gradient together with the maximum
approximation error over the interval; this sensitivity, in turn, determines
the noise required to achieve DP.

Because we require a bound on \(\|\vec{w}_i\|\) at every iteration, not merely at
the optimum, and because polynomial approximations and DP noise are present, the
weight norm cannot be bounded directly by \(\sqrt{\Theta}\).
Instead, we derive an explicit upper bound on the iterates \(\|\vec{w}_i\|\) in
Lemma~\ref{lem:weight_magnitude}.

In Algorithm~\ref{fig:modified_grad_desc}, we define
$\tilde{p}'(z,y) := y\,\tilde{p}'_1(z) + (1-y)\,\tilde{p}'_0(z)$,
where \(\tilde{p}'_0\) and \(\tilde{p}'_1\) are polynomial approximations of
\(\phi'_0\) and \(\phi'_1\), respectively.
The polynomial \(P_\kappa\) approximates
\(F_\kappa(x) := \frac{d}{dx} B_\kappa(x)\) over the interval
\([\kappa\Theta, \Theta]\); note that \(F_\kappa(x) = 1/x\) on this interval.
All polynomials \(\tilde{p}'_0\), \(\tilde{p}'_1\), and \(P_\kappa\) are defined
on \((-\infty, \infty)\).

\begin{algorithm}[tb]
   \caption{Approximate Gradient Descent with Modified objective function $f^{+B}_\kappa$.}
\label{fig:modified_grad_desc}
\begin{algorithmic}
   \STATE {\bfseries Input:} Running steps $T$, learning rate
   $\eta_{\ref{fig:modified_grad_desc}}$.
   \STATE set 
   $\sigma = \frac{2\Delta_2 \sqrt{ T\log(3/\delta)}}{\epsilon N}$, 
$\Delta_2 = 2(\phi'_{max} + e_f) \sqrt{m}$
\STATE Initialize $\vec{w}_0 = \vec{0}$\\
   \FOR {$i = 0$ to $T-1$}
    \STATE
$
\begin{aligned}
&\vec{w}_{i+1} \gets  \vec{w}_i
- \eta_{\ref{fig:modified_grad_desc}} \Bigl ( 2 \lambda P_\kappa(\Theta - \|\vec{w}_i\|^2)\vec{w}_i
        \\
&+ \Bigl ( \frac{1}{N} \sum_{(\vec{x}, y) \in D}
        \tilde{p}'(\langle \vec{w}_i, \vec{x} \rangle, y) \vec{x} \Bigr )
 + \boldsymbol{\chi}_i  \Bigr )
\end{aligned}
$

    \STATE where $\boldsymbol\chi_i \sim \mathcal{N}(0, \sigma^2 \vec{I}_m)$.
    \ENDFOR
\STATE {\bfseries Output:} $\vec{w}_T$
\end{algorithmic}
\end{algorithm}

\begin{restatable}{theorem}{DiffPriv} [Differential Privacy of Algorithm~\ref{fig:modified_grad_desc}] \label{th:dp_no_clipping}
Fix constants $e_f \in [0,\phi'_{max}], e_B \geq 0, \kappa \in (0,1)$. Let
$\zeta_f = e_f \sqrt{m}$,
$d \geq \frac{\|\grad f(\vec{0})\|}{\sqrt{m}}$, $\mathsf{c}_\delta := \sqrt{2 \ln(\frac{3T}{\delta})}$, and
$R := \sqrt{(1-\kappa)\Theta} + \eta_{\ref{fig:modified_grad_desc}}\left(  \phi'_{max}\sqrt{m} + \zeta_f +
2\lambda e_B \sqrt{\Theta}
+ (\sqrt{m} + \mathsf{c}_\delta)\sigma \right)$.
Let $\mathsf{m}_P$, (resp. $\mathsf{M}_P$) be the minimum
(resp.~maximum) value of $P_\kappa$ on the interval
$[\Theta - R^2, \kappa \Theta]$.
If the following hold:
{\setlength\topsep{0pt}
\setlength\itemsep{0.1pt}
\setlength\parskip{0pt}
\begin{itemize}
\item For $y \in \{0,1\}$,
$\tilde{p}'(z,y)$
has error at most $e_f$ 
with respect to $\phi'(z,y)$
on the interval $z \in [-\sqrt{m}R, \sqrt{m}R]$.
\item $P_\kappa$
has error at most
$e_B$ with respect to $F_\kappa$ on the interval
$[\kappa \Theta, \Theta]$. and is monotonically \emph{decreasing}
on the interval $[\Theta - R^2, \kappa \Theta]$.
\item $\mathsf{m}_P \geq 0$ and $\eta_{\ref{fig:modified_grad_desc}}$ 
satisfies the following inequality:
\begin{equation} \label{eq:eta_rest}
\eta_{\ref{fig:modified_grad_desc}} \leq 
\min \left \{\frac{\kappa \Theta}{\lambda}, \frac{1}{\lambda(\mathsf{M}_P+\mathsf{m}_P) + \frac{(\phi'_{max} - d) \sqrt{m}}{2\sqrt{(1-\kappa)\Theta}}} \right \}.
   \end{equation}
\item $\kappa$ satisfies the following inequality:
\begin{equation} \label{eq:kappa_constraint}
\sqrt{(1-\kappa)\Theta}  \geq \frac{-B + \sqrt{B^2-4AC}}{2A},
\end{equation}
where 
$\alpha = 2\eta_{\ref{fig:modified_grad_desc}}\lambda\mathsf{m}_P$,
$A = (2\alpha - \alpha^2)$,
$B = -2\eta_{\ref{fig:modified_grad_desc}}((1-\alpha)(d\sqrt{m} + \mathsf{c}_\delta \cdot \sigma) + \zeta_f)$,
and $C = -\eta_{\ref{fig:modified_grad_desc}}^2((\sqrt{m}\phi'_{max}+(\sqrt{m} + \mathsf{c}_\delta)\sigma)^2-\zeta_f^2)$.
\end{itemize}}
then
the model, $\vec{w}_T$, outputted 
by Algorithm~\ref{fig:modified_grad_desc} achieves
$(\epsilon, \delta)$-differential privacy.
\end{restatable}

\vspace{3mm}

\begin{remark}[Role of monotonicity]
Monotonicity of $P_\kappa$ on
$[\Theta - R^2, \Theta]$ is not required for DP.
However, enforcing monotonicity simplifies the choice of
$\kappa$ and $\eta_{\ref{fig:modified_grad_desc}}$ satisfying
(\ref{eq:kappa_constraint}) and (\ref{eq:eta_bound}), and is used in our
convergence analysis.
\end{remark}

\begin{remark}[Feasibility of parameters]
For fixed $(\varepsilon,\delta), e_f, e_B, d$, 
constraints~(\ref{eq:eta_rest}) and~(\ref{eq:kappa_constraint}) can always be
satisfied by choosing $\kappa$ and
$\eta_{\ref{fig:modified_grad_desc}}$ sufficiently small.
The required polynomials $\tilde{p}'$ and $P_\kappa$ can likewise be constructed with
sufficiently high degree to achieve the desired error (and monotonicity) over
the specified intervals~\cite{DeVore1977,DeVoreYu1985}.
See Appendix~\ref{sec:choosing_parameters} for details.
\end{remark}

\begin{remark}[Data-dependent refinement]
The Client may compute
$d = \|\nabla f(\vec{0})\|/\sqrt{m}$ and release it to the Server in a
differentially private manner.
When $d \ll \max\{\phi'_0(0), \phi'_1(0)\}$, as observed empirically, this yields
significantly improved parameter choices.
We adopt this refinement in our experiments, and account for releasing it in our privacy budget, which essentially has the effect of increasing the number of GD iterations by 1.
\end{remark}

To prove Theorem~\ref{th:dp_no_clipping}, we first establish a bound on the magnitude of 
$\|\vec{w}_i\|$ in each iteration
of Algorithm~\ref{fig:modified_grad_desc}.
\begin{restatable}{lemma}{WeightNorm} \label{lem:weight_magnitude}
    (Bounding magnitude of weights)
Let $\vec{w}_i$ denote the $i$-th iterate of Algorithm~\ref{fig:modified_grad_desc}. Under the parameter conditions of Theorem~\ref{th:dp_no_clipping}, with probability at least $1-\tfrac{2\delta}{3}$, the following holds for all $i \in \{0,\ldots,T\}$:
\vspace{-2mm}
\[
\begin{aligned}
\|\vec{w}_i\| \leq &
\sqrt{(1-\kappa)\Theta} \\ &+ \eta_{\ref{fig:modified_grad_desc}}\left(  \phi'_{max}\sqrt{m} + \zeta_f +
2\lambda e_B \sqrt{\Theta}
+ (\sqrt{m} + \mathsf{c}_\delta)\sigma \right).
\end{aligned}
\]
\vspace{-4mm}
\end{restatable}
We note that Lemma~\ref{lem:weight_magnitude}---under the parameter settings of
Theorem~\ref{th:dp_no_clipping}---guarantees that the weight bound fails with probability at most $\frac{2\delta}{3}$, and hence holds with probability at least $1-\frac{2\delta}{3}$. Conditioned on this event, Algorithm~\ref{fig:modified_grad_desc} satisfies $(\epsilon,\delta/3)$-DP. Hence, by a union bound, the overall algorithm satisfies $(\epsilon,\delta)$-DP. The proof of Lemma~\ref{lem:weight_magnitude} is deferred to Appendix \ref{sec:proof_Lemma_bounding_weights}.

Lemma~\ref{lem:weight_magnitude} implies that
$|\langle \vec{w}_i, \vec{x} \rangle| \leq \sqrt{m}R$, which allows us to fix the approximation interval for $\tilde{p}'(z,y)$ and bound the sensitivity.

Appendix~\ref{sec:sensitivy_bound} shows that the $\ell_2$ sensitivity of
$\nabla g(\vec{w}_i)$ is
$\frac{\Delta_2}{N}$, where
$\Delta_2 \le 2(\phi'_{\max}+e_f)\sqrt{m}$.
Using the zCDP-based composition analysis of
\cite{jayaraman2018distributed,bun2016concentrated} and the inequality
$\sqrt{\log(3/\delta)+\epsilon}-\sqrt{\log(3/\delta)}
\ge \epsilon/(2\sqrt{2\log(3/\delta)})$ for
$\epsilon \le \log(3/\delta)$,
Algorithm~\ref{fig:modified_grad_desc} is $(\epsilon,\delta)$-DP provided
$\sigma \ge \frac{2\Delta_2 \sqrt{T\log(3/\delta)}}{\epsilon N}$.

\sloppy
\vspace{-0.5em}\paragraph{Convergence of Algorithm~\ref{fig:modified_grad_desc}.}
To apply Theorem~\ref{thm:dp_converge}, which establishes convergence of
approximate gradient descent with DP noise, to
Algorithm~\ref{fig:modified_grad_desc}, we extend the objective
$f^{+B}_\kappa$ to the domain $\{\vec{w} : \|\vec{w}\| \le R\}$.
We do so by extending the domain of $F_\kappa(x)$ to
$x \in [\Theta - R^2, \Theta]$ and defining $B_\kappa(x)$ as its antiderivative.
We further extend $F_\kappa$ so that the resulting objective
$f^{+B}_\kappa$ is strongly convex and has bounded smoothness, enabling the
application of Theorem~\ref{thm:dp_converge}.
Finally, the maximum approximation error between $F_\kappa$ and its polynomial
approximation $P_\kappa$ over this enlarged interval remains bounded by $e_B$.
As a result, the gradient approximation error satisfies
$\zeta \le \zeta_f + 2\lambda e_B R$.
See Appendix~\ref{sec:modified_converge}.

\section{Approximating Polynomials and Hyperparameter Selection}
\label{sec:guidance_of_approx}

We first consider the non-DP setting, where training follows
Algorithm~\ref{fig:grad_desc}. By Theorem~\ref{thm:npdp_converge}, the number of
iterations \(T\) required for convergence grows with the smoothness parameter
\(\beta_g\) of the approximating polynomial, while the final optimization error
$\|\vec{w}_{g,T} - \vec{w}^*_f\|^2$
is primarily controlled by \(\zeta\), which captures the approximation error.
We empirically illustrate the tradeoff between smoothness/convergence rate and approximation/optimization error in Appendix~\ref{sec:guidance_smooth}, showing that the minimax and least-squares polynomial approximations of a fixed degree exemplify this effect.

Appendix~\ref{sec:guidance_interval} further shows that even mild underestimation of the
approximation interval can cause divergence. In the FHE setting, however, the
Server cannot detect when inputs fall outside this interval. In
Lemma~\ref{lem:weight_magnitude}, we show that modifying the objective function
(as in Algorithm~\ref{fig:modified_grad_desc}) allows upperbounding the magnitude of the weights so that intervals can be fixed
\emph{a priori} with provable correctness; the non-DP case follows by setting
\(\sigma=0\) in Algorithm~\ref{fig:modified_grad_desc}.

Beyond illustrating these tradeoffs empirically, Appendix~\ref{sec:guide_modified_GD} provides data-independent guidance for selecting the polynomial approximation (e.g., degree and approximation interval) and hyperparameters (e.g., $\eta_{\ref{fig:modified_grad_desc}}, \Theta, \lambda, \kappa$) such that all conditions of Theorem~\ref{th:dp_no_clipping} are satisfied. This procedure enables DP training under FHE via Algorithm~\ref{fig:modified_grad_desc}.

\section{Experimental Results on DP FHE Training}
\label{sec:fhe_experiments}

\begin{figure}[tb]
    \centering
    \includegraphics[width=1\linewidth]
{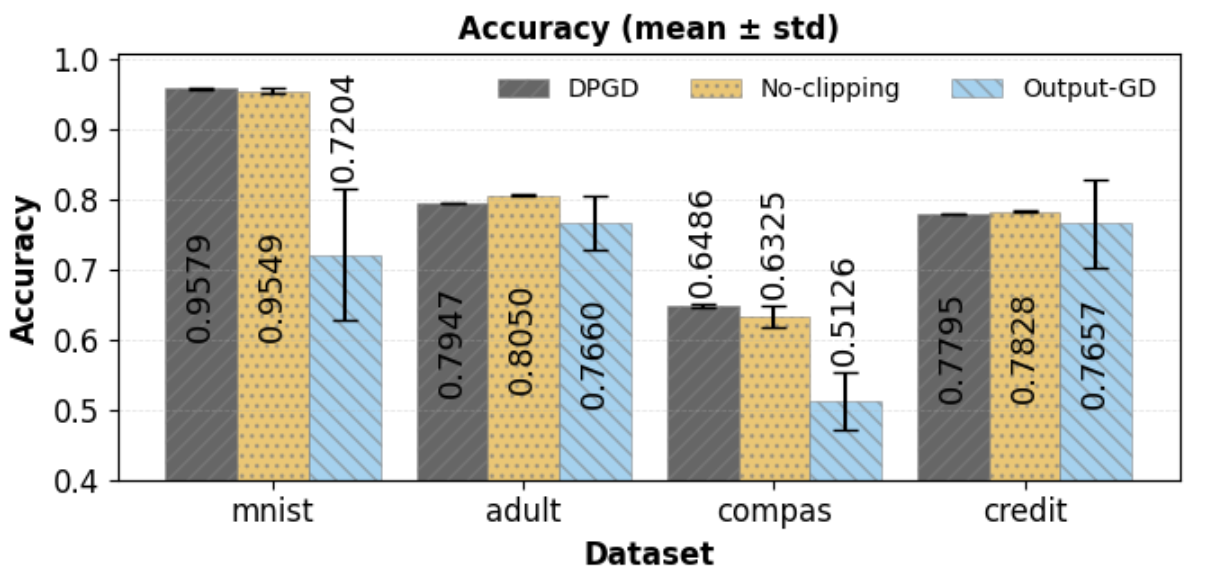}
    \vskip 0.1in
    \caption{Visual comparison of model accuracy with error bars over 100 runs across four datasets under a privacy budget of ($\epsilon=1$, $\delta=10^{-5}$). Model 1: standard DP-GD; 
    Model 2: augmented no-clipping DP-GD, where the sigmoid and barrier functions are respectively replaced by a 7th or 9th-degree polynomial and a 4th-degree polynomial; 
    Model 3: DP model using output perturbation (Algorithm 2 in \cite{wu2017bolt}).
    Error bars represent standard deviation across runs. (see Table \ref{tb:no-clipping_nofhe} for numerical values.)}
    \vspace{-1em}
\label{fig:dp_no_fhe_accuracy}
\end{figure}


We evaluate our proposed no-clipping approach (Algorithm~\ref{fig:modified_grad_desc}) against standard DP-(S)GD, which enforces differential privacy via per-sample gradient clipping. 
For training under FHE, we use an SGD-based update for efficiency and compute the required noise scale using the subsampled DP analysis of~\cite{wang2019subsampled}. We first assess both methods in a plaintext setting to verify that the no-clipping approach maintains comparable model accuracy. We then implement both algorithms in the FHE domain (as Algorithm~\ref{fig:modified_DP-GD_fhe} vs Algorithm~\ref{fig:trad_DP-GD_fhe} in Section~\ref{sec:protocols}) to compare runtime and accuracy. In the FHE setting, we conduct experiments under two security configurations with different ring dimensions (RingDim = 32768 and RingDim = 131072). While the two methods achieve similar accuracy, our approach reduces the multiplicative depth per iteration from 24 to 9, resulting in roughly a $3\times$ reduction (RingDim = 32768) and a $5\times$ reduction (RingDim = 131072) in computational cost (Tables~\ref{tb:no-clipping_fhe}, \ref{tb:no-clipping_fhe_128}, and \ref{tab:depth-cost_analysis}). 


\textbf{Data Preprocessing:} We applied standard preprocessing steps to handle missing values, encode categorical variables numerically, and normalize all features to the range $[-1, 1]$. To mitigate resource and computational constraints, high-dimensional datasets, such as MNIST, were compressed via Principal Component Analysis (PCA)~\cite{PCA, Auto_dimension_PCA}, using the scikit-learn implementation~\cite{Scikit-learn}. We do not include PCA in the privacy accounting of our DP guarantee. Since it is applied uniformly to both our method and the DP-SGD baseline, DP-PCA is orthogonal to our analysis and does not affect comparisons.


\textbf{Approximation Algorithms:} 
 We approximate the sigmoid and square-root functions using either a minimax polynomial, which minimizes
$\sup_{x\in[a,b]} |f(x)-g(x)|$, or a least-squares polynomial, which minimizes
$\int_a^b (f(x)-g(x))^2\,dx$ over the interval $[a,b]$.
Efficient algorithms for computing optimal minimax and least-squares polynomials of fixed degree on a given interval are well known~\cite{remez1934,bjorck1996numerical}.
The inverse operation is approximated using Taylor expansions or least-squares polynomials.
Per-sample clipping requires a comparison operation, which we approximate following~\cite{mazzone2025efficient}.


\begin{table*}[tb]  
\centering
\caption{Model performance under FHE on an AMD EPYC 9534 with multi-thread parallelism ($\epsilon=1, \delta=10^{-5}$, RingDim = 32768)}
\vskip 0.1in
\label{tb:no-clipping_fhe}
\begin{tabular}{l l c c c c c c}
\toprule
\cmidrule(lr){5-8}
\textbf{Data} & \textbf{Training Model} & \textbf{ACC} & \textbf{AUC} & \textbf{10-threads} & \textbf{20-threads} & \textbf{30-threads} & \textbf{50-threads} \\
\midrule
\multirow{3}{*}{mnist} 
& Algo~\ref{fig:trad_DP-GD_fhe} (DP-SGD) 
& 93.62\% & 98.21\% & 600.1 sec/iter & 338.0 sec/iter & 283.0 sec/iter & 187.2 sec/iter \\
& Algo~\ref{fig:modified_DP-GD_fhe} (No Clip) 
& 93.99\% & 98.22\% & 148.2 sec/iter & 93.0 sec/iter & 86.9 sec/iter & 58.2 sec/iter \\
& Output-SGD
& 74.18\% & 88.59\% & 147.5 sec/iter & 90.1 sec/iter& 82.9 sec/iter & 55.7 sec/iter \\
\midrule
\multirow{3}{*}{credit} 
& Algo~\ref{fig:trad_DP-GD_fhe} (DP-SGD) 
& 78.00\% & 68.88\% & 446.1 sec/iter & 258.7 sec/iter & 227.5 sec/iter & 164.0 sec/iter \\
& Algo~\ref{fig:modified_DP-GD_fhe} (No Clip) 
& 77.99\% & 68.69\% & 132.2 sec/iter & 79.4 sec/iter & 70.4 sec/iter &  53.7 sec/iter \\
& Output-SGD
& 76.76\% & 54.82\% & 133.9 sec/iter & 82.7 sec/iter & 75.7 sec/iter & 53.9 sec/iter \\
\bottomrule
\end{tabular}
\end{table*}

\begin{table*}[tb]  
\centering
\caption{Model performance under FHE on an AMD EPYC 7502 with multi-thread parallelism ($\epsilon=1, \delta=10^{-5}$, RingDim = 131072, 128-bit security)}
\vskip 0.1in
\label{tb:no-clipping_fhe_128}
\begin{tabular}{l l c c c c c c}
\toprule
\cmidrule(lr){5-8}
\textbf{Data} & \textbf{Training Model} & \textbf{ACC} & \textbf{AUC} & \textbf{10-threads} & \textbf{20-threads} & \textbf{30-threads} & \textbf{50-threads} \\
\midrule
\multirow{3}{*}{mnist} 
& Algo~\ref{fig:trad_DP-GD_fhe} (DP-SGD) 
& 93.77\% & 98.20\% & 7996.8 sec/iter & 4240.0 sec/iter & 3504.4 sec/iter & 1813.5 sec/iter \\
& Algo~\ref{fig:modified_DP-GD_fhe} (No Clip) 
& 93.97\% & 98.20\% & 566.7 sec/iter & 442.2 sec/iter & 409.8 sec/iter & 350.5 sec/iter \\
& Output-SGD
& 73.87\% & 86.68\% & 524.5 sec/iter & 434.9 sec/iter& 397.4 sec/iter & 351.1 sec/iter \\
\midrule
\multirow{3}{*}{credit} 
& Algo~\ref{fig:trad_DP-GD_fhe} (DP-SGD) 
& 77.96\% & 69.18\% & 8028.1 sec/iter & 3910.5 sec/iter & 3236.3 sec/iter & 1843.3 sec/iter \\
& Algo~\ref{fig:modified_DP-GD_fhe} (No Clip) 
& 77.99\% & 68.85\% & 521.9 sec/iter & 449.2 sec/iter & 394.1 sec/iter &  343.1 sec/iter \\
& Output-SGD
& 76.76\% & 55.61\% & 551.4 sec/iter & 415.8 sec/iter & 379.9 sec/iter & 342.0 sec/iter \\
\bottomrule
\end{tabular}
\end{table*}
\subsection{Empirical Results}
\label{sec: test_result_fhe}


We first evaluate both algorithms in the plaintext setting to enable a direct comparison of accuracy without confounding errors introduced by FHE. Results in this setting are still subject to randomness from DP noise, data splits, and other sources; therefore, each experiment was repeated 100 times and the results were averaged. We conducted this evaluation on four datasets: Adult~\cite{adult}, Compas~\cite{compas}, Credit~\cite{default_of_credit_card_clients_350}, and a restructured version of MNIST~\cite{mnist, han2018_git}. For MNIST, we considered a binary classification task distinguishing digits 3 and 8. As shown in Figure~\ref{fig:dp_no_fhe_accuracy} and Table~\ref{tb:no-clipping_nofhe} (in Appendix), the no-clipping approach achieves accuracy comparable to standard DP-GD (e.g., a maximum drop of 1.61\% in accuracy and 1.05\% in AUC.)

We next analyze both methods in the FHE setting.
As shown in Table~\ref{tb:no-clipping_fhe} and \ref{tb:no-clipping_fhe_128}, our proposed method attains accuracy comparable to standard DP-GD while matching the runtime efficiency of Output-GD, yielding at least a $3\times$ and $5\times$ speedup over DP-GD for RingDim = 32768 and 131072, respectively. Unlike Output-GD, which trades accuracy for speed, our approach preserves strong model performance. 
To accommodate limited computational resources, each iteration trains on a random subsample of 100 data points for DP-GD and Output-GD (98 for our method). Notably, as fewer parallel threads are used (i.e., each thread handles more data), the relative speedup over DP-GD increases, demonstrating superior scalability.\footnote{Experiments can be reproduced via: \href{https://github.com/dpfhe096-design/dp_fhe}{github.com/dpfhe096-design/dp\_fhe}(including e.g., $T, \eta$, and FHE-related settings). 
}

\textbf{Multiplicative Depth Analysis}
Table~\ref{tb:no-clipping_fhe} shows that our method trains significantly faster than standard DP-GD. To explain this, we analyze the computational complexity in terms of multiplicative depth, a key metric drives overhead and noise growth in FHE (see Appendix Table~\ref{tab:depth-cost_analysis}). Standard DP-GD (Algorithm~\ref{fig:trad_DP-GD_fhe}) is bottlenecked by costly per-sample gradient clipping under FHE. Our algorithm (Algorithm~\ref{fig:modified_DP-GD_fhe}) removes clipping and instead uses a lightweight "Barrier Calculation" (depth 6) running once per iteration in parallel with simplified gradient (depth 8). The critical path depth is thus $\max(8,6)=8$. Including the weight update, total depth per iteration drops from 24 to 9, a $2.5\times$ reduction that explains the observed FHE speedup.

\textbf{Observations.} Our empirical comparison shows that our proposed no-clipping
algorithm outperforms the traditional clipping-based approach for DP-ML under FHE.
Both methods achieve comparable accuracy, with the no-clipping model at most a 0.01\% accuracy drop and a 0.33\% AUC decrease.
Compared to output perturbation (Output-GD), which runs standard (no-DP) GD under FHE and adds noise to the output, our method essentially matches its computation time while delivering significantly higher accuracy.

Beyond efficiency, our approach addresses fundamental limitations of
DP-GD under FHE. The standard approach requires a polynomial
approximation of \(\phi'\) over an input interval determined by the
data-dependent hypothesis value \(z=\langle \vec{w}, \vec{x} \rangle\), which
varies dynamically during training. If this interval is chosen incorrectly, the
approximation of \(\phi'\), and thus also the approximate squared gradient norm, may become unbounded, leading to divergence or severe
errors in training. For example, during MNIST training, using a
square-root approximation over \([0,10]\) caused divergence after only 10
iterations. Under FHE, such failures cannot be remedied by iteratively retuning
approximation intervals, since interactive corrections are not possible. In
contrast, our method explicitly bounds \(z\) and provides formal guarantees on
the required approximation intervals (see
Theorem~\ref{th:dp_no_clipping}).

These issues affect not only convergence but also DP. In
particular, clipping involves 
normalization of the gradient,
which relies on polynomial approximations of \(1/x\).
If the gradient norm falls outside the approximation interval, this normalization
can become arbitrarily inaccurate, and the effective clipping threshold can no
longer be assumed to equal \(C\). As a result, the DP sensitivity bound may be
violated. Our no-clipping approach explicitly accounts for all approximation
errors and therefore yields provable differential privacy guarantees (see
Theorem~\ref{th:dp_no_clipping} and Algorithm~\ref{fig:modified_grad_desc}).

Finally, the no-clipping method offers substantial practical gains under FHE,
reducing multiplicative depth by more than 2.5$\times$ and cutting training time
by more than one-third, making it both theoretically rigorous and practically
viable for privacy-preserving ML.

\section{Discussion and Limitations}

This work provides the first formal analysis of DP for gradient descent with polynomial approximations in the FHE setting. We introduce an efficient training algorithm that is provably DP even in the presence of polynomial approximation error. We establish convergence and end-to-end privacy under encryption, and show that polynomial smoothness and approximation error govern convergence and model utility, enabling principled choices of approximation scheme, degree, and interval. Together, these results strengthen the theory and practice of private, encrypted learning.

Our combined convergence and DP analysis, however, is currently limited to smooth, convex objectives\footnote{We note that adding the barrier function to the objective induces strong convexity over the interior of the feasible region.}, where iterate norms can be tightly controlled to enable sensitivity analysis. Extending these results to general neural networks is an important direction for future work. The core mechanism, augmenting the objective with a barrier, may extend beyond convex settings by helping bound weight norms, stabilize activations, and support sensitivity control under polynomial approximations, though formal guarantees remain open for future work.

\section*{Impact Statement}
\label{sec:impact_statement}
This paper advances machine learning by enabling fully private model training through the combination of Fully Homomorphic Encryption (FHE) and Differential Privacy (DP). We provide the first theoretical convergence analysis for gradient descent with polynomial approximations compatible with FHE, together with a differentially private training algorithm designed for this encrypted setting. Our approach enables learning directly on encrypted data while also producing models that are themselves differentially private, protecting individuals not only during training but also against information leakage in downstream use.

By removing the need for costly per-sample gradient clipping, our methods substantially improve the practicality of secure, outsourced learning for data owners who lack computational resources, particularly in sensitive domains such as healthcare, finance, and government services. Although FHE-based training incurs higher computational costs than plaintext learning, our work reduces these overheads and improves scalability, bringing end-to-end private learning closer to real-world deployment. We believe the societal benefits of enabling verifiable, confidential training and privacy-preserving model release, without requiring trust in the training infrastructure, outweigh the remaining computational costs and any modest impact on model utility.

\begin{ack}
The authors thank the anonymous reviewers for their insightful comments and valuable feedback on this work.  Dana Dachman-Soled is supported in part by NSF grants CNS-2154705, CNS-1933033, and IIS2147276 is jointly provided with Min Wu.
We gratefully acknowledge support from the JP Morgan Chase Faculty Research Award.
\vspace{-0.5em}\paragraph{Disclaimer.}
This paper was prepared for informational purposes in part by the Artificial Intelligence Research group of JPMorgan Chase \& Co. and its affiliates (``JP Morgan'') and is not a product of the Research Department of JP Morgan. JP Morgan makes no representation and warranty whatsoever and disclaims all liability, for the completeness, accuracy or reliability of the information contained herein. This document is not intended as investment research or investment advice, or a recommendation, offer or solicitation for the purchase or sale of any security, financial instrument, financial product or service, or to be used in any way for evaluating the merits of participating in any transaction, and shall not constitute a solicitation under any jurisdiction or to any person, if such solicitation under such jurisdiction or to such person would be unlawful.
\end{ack}

\bibliographystyle{plainnat}

\bibliography{refs}

\newpage
\appendix

\onecolumn

\section{Proof of Theorem~\ref{thm:npdp_converge}}
\label{sec:no_dp_converge}

We begin by restating Theorem~\ref{thm:npdp_converge}.

\ConvergenceNoDP*

To prove Theorem~\ref{thm:npdp_converge}, we analyze gradient descent on the
(non-convex) approximate objective $g$. By
Claim~\ref{claim:critical_point}, within $T$ iterations there exists an iterate
$\vec{w}_{g,\tilde t}$, for some $\tilde t \in [T]$, such that
$\|\nabla g(\vec{w}_{g,\tilde t})\| \le \rho$.
Assumption~\ref{ass:grad_dist} then implies
$\|\nabla f(\vec{w}_{g,\tilde t})\| \le \rho + \zeta$.

The same claim guarantees $g(\vec{w}_{g,T}) \le g(\vec{w}_{g,\tilde t})$, and
Assumption~\ref{ass:obj_dist} yields
$f(\vec{w}_{g,T}) \le f(\vec{w}_{g,\tilde t}) + 2\alpha$.
Using the bound on $\|\nabla f(\vec{w}_{g,\tilde t})\|$ and the strong convexity
of $f$, we upper bound
$f(\vec{w}_{g,\tilde t}) - f(\vec{w}_f^*)$, where $\vec{w}_f^*$ is the minimizer
of $f$.
Combining these inequalities gives a bound on
$f(\vec{w}_{g,T}) - f(\vec{w}_f^*)$, which directly implies a bound on
$\|\vec{w}_{g,T} - \vec{w}_f^*\|^2$.
We now present the detailed step-by-step proof.

To prove Theorem~\ref{thm:npdp_converge} we begin with the following claim, restated from~\cite{cme323_lec8}. 

\begin{claim}[In \cite{cme323_lec8}, Theorem 8.3] \label{claim:critical_point}
Fix $\rho >0$ and let $T = \frac{2\beta_g(g(\vec{w}_{g,0}) - L)}{\rho^2}$. Then there exists $\tilde{t} \in [T]$ such that
$g(\vec{w}_{g,T}) \leq g(\vec{w}_{g, \tilde{t}})$ and
$\|\grad g(\vec{w}_{g, \tilde{t}}) \|\leq \rho$,
where $\vec{w}_{g,\tilde{t}}, \vec{w}_{g,T}$ are defined as in Algorithm~\ref{fig:grad_desc} with $\eta_{\ref{fig:grad_desc}} = \frac{1}{\beta_{g}}$.
\end{claim}

We now complete the proof of Theorem~\ref{thm:npdp_converge}, given Claim~\ref{claim:critical_point}.

\begin{proof}[Proof of Theorem~\ref{thm:npdp_converge}]
Let $T$, $\tilde{t}$ be as in Claim~\ref{claim:critical_point}.
Then
\begin{align}
&f(\vec{w}_{g,T}) - f(\vec{w}^*_f) =
f(\vec{w}_{g,T}) - f(\vec{w}_{g, \tilde{t}}) + f(\vec{w}_{g, \tilde{t}}) - f(\vec{w}^*_f)\nonumber \\
&=
g(\vec{w}_{g,T}) - g(\vec{w}_{g, \tilde{t}}) +
\left( f(\vec{w}_{g,T}) - f(\vec{w}_{g, \tilde{t}}) - (g(\vec{w}_{g,T}) - g(\vec{w}_{g, \tilde{t}})) \right )+ f(\vec{w}_{g, \tilde{t}}) - f(\vec{w}^*_f)\nonumber \\
&\quad \leq g(\vec{w}_{g,T}) - g(\vec{w}_{g, \tilde{t}}) + 2\alpha + 
f(\vec{w}_{g, \tilde{t}}) - f(\vec{w}^*_f) \label{eq:remove_alpha} \\
&\quad \leq 2\alpha + f(\vec{w}_{g, \tilde{t}}) - f(\vec{w}^*_f), \label{eq:reduce}
\end{align}
where (\ref{eq:remove_alpha}) follows from Assumption~\ref{ass:obj_dist} and
(\ref{eq:reduce}) follows from the fact that 
$g(\vec{w}_{g,T}) \leq g(\vec{w}_{g, \tilde{t}})$.

Since $\|\grad g(\vec{w}_{g, \tilde{t}})\| \leq \rho$,
we have by Assumption~\ref{ass:grad_dist}
that $\|\grad f(\vec{w}_{g, \tilde{t}})\|^2 \leq \rho^2 + 2\zeta \rho + \zeta^2$ and by
the
Polyak–Lojasiewicz condition that
\[
f(\vec{w}_{g, \tilde{t}}) - f(\vec{w}^*_{f}) \leq \frac{\rho^2 + 2\zeta \rho + \zeta^2}{2\mu}.
\]
Substituting into
(\ref{eq:reduce}), we obtain
\[
f(\vec{w}_{g,T}) - f(\vec{w}^*_f) \leq 2\alpha + \frac{\rho^2 + 2\zeta \rho + \zeta^2}{2\mu}.
\]

By the Quadratic-Growth condition, we can also bound,
\begin{align*}
\|\vec{w}_{g,T} - \vec{w}^*_f\|^2 &\leq \frac{2}{\mu} \cdot (f(\vec{w}_{g,T}) - f(\vec{w}^*_f))\\
&= \frac{4\alpha}{\mu} + \frac{\rho^2 + 2\zeta \rho + \zeta^2}{\mu^2}.
\end{align*}
\end{proof}

\subsection{Removing Assumption~\ref{ass:obj_dist}} \label{sec:alpha}

In this section, we will show the following bound:
\begin{align} \label{eq:bound_alpha}
f(\vec{w}_{g,T}) - f(\vec{w}_{g, \tilde{t}})
- (g(\vec{w}_{g,T}) - g(\vec{w}_{g, \tilde{t}})) \leq
\zeta \cdot \|\vec{w}_{g,T} -\vec{w}_{g, \tilde{t}} \|.
\end{align}
This bound allows us to eliminate the additional assumption that
\[
\forall \vec{w} \in \mathcal{D},
| f(\vec{w}) - g(\vec{w}) | \leq \alpha.
\]
Specifically, in the analysis in Appendix~\ref{sec:no_dp_converge} above, (\ref{eq:remove_alpha}) can be replaced with
\[
g(\vec{w}_{g,T}) - g(\vec{w}_{g, \tilde{t}}) + \zeta \cdot \|\vec{w}_{g,T} -\vec{w}_{g, \tilde{t}} \| + 
f(\vec{w}_{g, \tilde{t}}) - f(\vec{w}^*_f),\]
yielding the final bounds of
\[
f(\vec{w}_{g,T}) - f(\vec{w}^*_f) \leq \zeta \cdot \|\vec{w}_{g,T} -\vec{w}_{g, \tilde{t}} \| + \frac{\rho^2 + 2\zeta \rho + \zeta^2}{2\mu}
\]
and
\[
\|\vec{w}_{g,T}) - \vec{w}^*_f\|^2 \leq \frac{2\zeta \cdot \|\vec{w}_{g,T} -\vec{w}_{g, \tilde{t}} \|}{\mu} + \frac{\rho^2 + 2\zeta \rho + \zeta^2}{\mu^2}.
\]

We now prove the upperbound given in (\ref{eq:bound_alpha}):
let $C$ be the line that goes from 
$\vec{w}_{g,\tilde{t}}$ to $\vec{w}_{g, T}$.
This line can be parameterized in terms of a single variable $\tau$ such that $\vec{\ell}(\tau)= \vec{w}_{g,\tilde{t}} + \tau \cdot (\vec{w}_{g, T}-\vec{w}_{g,\tilde{t}})$, where $\tau \in [0, 1]$. 
We have that
\begin{align}
    \int_C \langle \grad f(\vec{w}) - \grad g(\vec{w}), d\vec{w} \rangle &= \int_0^1 \langle \grad f(\vec{\ell}(\tau)) - \grad g(\vec{\ell}(\tau)), \frac{d\vec{\ell}(\tau)}{d\tau} \rangle d\tau \nonumber \\
    &=\int_0^1 \langle \grad f(\vec{\ell}(\tau)) - \grad g(\vec{\ell}(\tau)), \vec{w}_{g,T}-\vec{w}_{g,\tilde{t}} \rangle d\tau \nonumber \\
    &\leq \int_0^1 ||\grad f(\vec{\ell}(\tau)) - \grad g(\vec{\ell}(\tau)))|| \cdot ||\vec{w}_{g,T}-\vec{w}_{g,\tilde{t}}|| d\tau \nonumber \\
    &\leq \int_0^1 \zeta \cdot ||\vec{w}_{g,T}-\vec{w}_{g,\tilde{t}}|| d\tau \label{eq:convex} \\
    &= \zeta \cdot ||\vec{w}_{g,T}-\vec{w}_{g,\tilde{t}})||, \label{eq:line_int_1}
\end{align}
where~(\ref{eq:convex}) follows from Assumption~\ref{ass:grad_dist} holding on the convex set $\mathcal{S}$. Since $\vec{w}_{g,\tilde{t}}, \vec{w}_{g,T} \in \mathcal{S}$, the line segment $\ell(\tau) = \vec{w}_{g,\tilde{t}} + \tau \cdot (\vec{w}_{g, T}-\vec{w}_{g,\tilde{t}})$ lies entirely in $\mathcal{S}$ for all $\tau \in [0,1]$, implying
\[
\|\nabla f(\ell(\tau)) - \nabla g(\ell(\tau))\| \le \zeta \quad \text{for all } \tau \in [0,1].
\]

On the other hand, by the gradient theorem for line integrals
\begin{align}
\int_C \langle \grad f(\vec{w}) - \grad g(\vec{w}), d\vec{w} \rangle &=
f(\vec{w}_{g,T}) - g(\vec{w}_{g,T}) - (f(\vec{w}_{g, \tilde{t}}) -
g(\vec{w}_{g, \tilde{t}}) ) \nonumber \\
&= f(\vec{w}_{g,T}) - f(\vec{w}_{g, \tilde{t}})
- (g(\vec{w}_{g,T}) - g(\vec{w}_{g, \tilde{t}})) \label{eq:line-int}
\end{align}

Combining (\ref{eq:line_int_1}) and (\ref{eq:line-int})we obtain:
\[
f(\vec{w}_{g,T}) - f(\vec{w}_{g, \tilde{t}})
- (g(\vec{w}_{g,T}) - g(\vec{w}_{g, \tilde{t}})) \leq \zeta \cdot ||\vec{w}_{g,T}-\vec{w}_{g,\tilde{t}})||,
\]
as desired.

\section{Proof of Theorem~\ref{thm:dp_converge}}
\label{sec:dp-converge-proof}

We begin by restating Theorem~\ref{thm:dp_converge}.

\ConvergenceDP*

To prove Theorem~\ref{thm:dp_converge}, we analyze the evolution of the expected
squared distance between the iterate $\vec{w}_{g,i}$, obtained by optimizing the
approximate objective $g$, and the optimal solution $\vec{w}_f^*$ of the true
objective $f$.

We begin by expanding $\vec{w}_{g,i+1}$ as the sum of $\vec{w}_{g,i}$, the gradient
update step, and the injected DP noise, and express
$\mathbb{E}\!\left[\|\vec{w}_{g,i+1}-\vec{w}_f^*\|^2\right]$ in terms of
$\vec{w}_{g,i}$. Using Assumption~\ref{ass:grad_dist}, we relate the gradient of
the approximate objective to that of the true objective, replacing
$\nabla g(\vec{w}_{g,i})$ with $\nabla f(\vec{w}_{g,i})$ plus the error term, $\zeta$.

This reduction allows us to apply standard convexity, smoothness, and expected curvature
properties of $f$ to bound the expected progress made in each iteration,
following a classical gradient descent convergence argument. Iterating this
bound for $i=0,\ldots,T-1$ yields a convergence guarantee on
$\mathbb{E}\!\left[\|\vec{w}_{g,T}-\vec{w}_f^*\|^2\right]$, completing the proof.

\begin{proof}
Below, we set $\eta_{\ref{fig:DP_grad_desc}} \leq \frac{1}{2 \beta_{f}}$. 

We first bound $\mathbb{E}[\|\vec{w}_{g, i+1} - \vec{w}^*_f\|^2]$ in terms of $\mathbb{E}[\|\vec{w}_{g, i} - \vec{w}^*_f\|^2]$. We begin by expanding
$\mathbb{E}[\|\vec{w}_{g, i+1} - \vec{w}^*_f\|^2]$:
\begin{align}
\mathbb{E}[\|\vec{w}_{g, i+1} - \vec{w}^*_f\|^2] &= \mathbb{E}[\| \vec{w}_{g, i} - \eta_{\ref{fig:DP_grad_desc}} \grad g(\vec{w}_{g, i}) - \eta_{\ref{fig:DP_grad_desc}} \boldsymbol\chi_i - \vec{w}^*_f\|^2] \nonumber \\
&= \mathbb{E}[\|\vec{w}_{g, i} - \vec{w}^*_f \|^2] - 2\eta_{\ref{fig:DP_grad_desc}} \mathbb{E}[\langle \grad g(\vec{w}_{g, i}), \vec{w}_{g, i} - \vec{w}^*_f \rangle] + 
\eta_{\ref{fig:DP_grad_desc}}^2\mathbb{E}[\|\grad g(\vec{w}_{g, i})\|^2] + \eta_{\ref{fig:DP_grad_desc}}^2 m \sigma^2 \label{eq:expand}
\end{align}
We next
replace the term
$- 2\eta_{\ref{fig:DP_grad_desc}} \mathbb{E}[\langle \grad g(\vec{w}_{g, i}), \vec{w}_{g, i} - \vec{w}^*_f \rangle]$ 
in (\ref{eq:expand})
with $- 2\eta_{\ref{fig:DP_grad_desc}} \mathbb{E}[\langle \grad f(\vec{w}_{g, i}), \vec{w}_{g, i} - \vec{w}^*_f \rangle] + 
2 \eta_{\ref{fig:DP_grad_desc}} \mathbb{E}[\langle \vec{w}_{g, i} - \vec{w}^*_f, \grad f(\vec{w}_{g, i}) - \grad g(\vec{w}_{g, i}) \rangle]$
to obtain:
\begin{align}
(\ref{eq:expand}) &= \mathbb{E}[\|\vec{w}_{g, i} - \vec{w}^*_f \|^2] - 2\eta_{\ref{fig:DP_grad_desc}} \mathbb{E}[\langle \grad f(\vec{w}_{g, i}), \vec{w}_{g, i} - \vec{w}^*_f \rangle]\nonumber \\ 
&\quad \quad
+ \eta_{\ref{fig:DP_grad_desc}}^2\mathbb{E}[\|\grad g(\vec{w}_{g, i})\|^2] + 
2 \eta_{\ref{fig:DP_grad_desc}} \mathbb{E}[\langle \vec{w}_{g, i} - \vec{w}^*_f, \grad f(\vec{w}_{g, i}) - \grad g(\vec{w}_{g, i}) \rangle]
+ \eta_{\ref{fig:DP_grad_desc}}^2 m \sigma^2 \label{eq:switch_f}
\end{align}
We next use expected curvature, convexity, and smoothness to upper bound the term $-2\eta_{\ref{fig:DP_grad_desc}} \mathbb{E}[\langle \grad f(\vec{w}_{g, i}), \vec{w}_{g, i} - \vec{w}^*_f\rangle]$ as follows:
\begin{align}
-2\eta_{\ref{fig:DP_grad_desc}} \mathbb{E}[\langle \grad f(\vec{w}_{g, i}), \vec{w}_{g, i} - \vec{w}^*_f\rangle] &=  - \eta_{\ref{fig:DP_grad_desc}} \mathbb{E}[\langle \grad f(\vec{w}_{g, i}), \vec{w}_{g, i} - \vec{w}^*_f \rangle] - \eta_{\ref{fig:DP_grad_desc}} \mathbb{E}[\langle \grad f(\vec{w}_{g, i}), \vec{w}_{g, i} - \vec{w}^*_f \rangle] \nonumber \\
&\leq 
- \eta_{\ref{fig:DP_grad_desc}} \nu \mathbb{E}[\|\vec{w}_{g, i} - \vec{w}^*_f\|^2] -
\eta_{\ref{fig:DP_grad_desc}} \mathbb{E}[\langle \grad f(\vec{w}_{g, i}), \vec{w}_{g, i} - \vec{w}^*_f \rangle] \label{eq:strong_convexity} \\
&\leq 
- \eta_{\ref{fig:DP_grad_desc}} \nu \mathbb{E}[\|\vec{w}_{g, i} - \vec{w}^*_f\|^2] -
\eta_{\ref{fig:DP_grad_desc}} \mathbb{E}[f(\vec{w}_{g, i}) - f(\vec{w}^*_f)]] \label{eq:convexity}  \\
&\leq - \eta_{\ref{fig:DP_grad_desc}} \nu \mathbb{E}_[\|\vec{w}_{g, i} - \vec{w}^*_f\|^2]
- \frac{\eta_{\ref{fig:DP_grad_desc}}}{2\beta_f} \mathbb{E}[\|\grad f(\vec{w}_{g, i})\|^2] 
\label{eq:smoothness} \\
&\leq
- \eta_{\ref{fig:DP_grad_desc}} \nu \mathbb{E}[\|\vec{w}_{g, i} - \vec{w}^*_f\|^2]
- 2 \eta_{\ref{fig:DP_grad_desc}}^2 \mathbb{E}[\|\grad f(\vec{w}_{g, i})\|^2] \label{eq:eta_bound},
\end{align}
where (\ref{eq:strong_convexity}) holds due to the expected curvature of $f$ being at least $\nu$, 
(\ref{eq:convexity}) holds due to convexity of $f$, 
(\ref{eq:smoothness}) holds due to  the $\beta_f$-smoothness condition, 
and (\ref{eq:eta_bound}) holds due to setting $\eta_{\ref{fig:DP_grad_desc}} \leq \frac{1}{2 \beta_f}$, which implies that  $- \frac{\eta_{\ref{fig:DP_grad_desc}}}{2\beta_f} \mathbb{E}[\|\grad f(\vec{w}_{g, i})\|^2]\leq - 2 \eta_{\ref{fig:DP_grad_desc}}^2 \mathbb{E}[\|\grad f(\vec{w}_{g, i})\|^2]$.

Replacing $-2\eta_{\ref{fig:DP_grad_desc}} \mathbb{E}[\langle \grad f(\vec{w}_{g, i}), \vec{w}_{g, i} - \vec{w}^*_f\rangle]$ 
in
(\ref{eq:switch_f})
with (\ref{eq:eta_bound})  yields:
\begin{align}
(\ref{eq:switch_f}) 
&\leq
\mathbb{E}[\|\vec{w}_{g, i} - \vec{w}^*_f \|^2] - \eta_{\ref{fig:DP_grad_desc}} \nu \mathbb{E}[\|\vec{w}_{g, i} - \vec{w}^*_f\|^2]
- 2 \eta_{\ref{fig:DP_grad_desc}}^2 \mathbb{E}[\|\grad f(\vec{w}_{g, i})\|^2] \nonumber \\
&\quad\quad + 
\eta_{\ref{fig:DP_grad_desc}}^2 \mathbb{E}[\|\grad g(\vec{w}_{g, i})\|^2] + 2 \eta_{\ref{fig:DP_grad_desc}} \mathbb{E}[\langle \vec{w}_{g, i} - \vec{w}^*_f, \grad f(w_{g, i}) - \grad g(w_{g, i}) \rangle] +
\eta_{\ref{fig:DP_grad_desc}}^2 m \sigma^2
\label{eq:eta_bound_2},
\end{align}

Re-arranging terms and using Assumption~\ref{ass:grad_dist} to upperbound the following term 
\begin{align*}
\mathbb{E}[\|\grad g(\vec{w}_{g, i})\|^2] &=
\mathbb{E}[\|\grad f(\vec{w}_{g, i})\|^2] + 2\mathbb{E}[\langle \grad f(\vec{w}_{g, i}), \grad g(\vec{w}_{g, i}) - \grad f(\vec{w}_{g, i}) \rangle] +
\mathbb{E}[\|\grad g(\vec{w}_{g, i}) - \grad f(\vec{w}_{g, i})\|^2]\\
&\leq \mathbb{E}[\|\grad f(\vec{w}_{g, i})\|^2] + 2\mathbb{E}[\langle \grad f(\vec{w}_{g, i}), \grad g(\vec{w}_{g, i}) - \grad f(\vec{w}_{g, i}) \rangle] +
\zeta^2
\end{align*}
we obtain:
\begin{align}
(\ref{eq:eta_bound_2}) 
&\leq \Big (\mathbb{E}[\|\vec{w}_{g, i} - \vec{w}^*_f \|^2] - \eta_{\ref{fig:DP_grad_desc}} \nu \mathbb{E}[\|\vec{w}_{g, i} - \vec{w}^*_f\|^2] \nonumber \\
&\quad \quad + 2 \eta_{\ref{fig:DP_grad_desc}} \mathbb{E}[\langle \vec{w}_{g, i} - \vec{w}^*_f, \grad g(\vec{w}_{g, i}) - \grad f(\vec{w}_{g, i})\rangle] \Big) \nonumber \\
&\quad\quad + 
\Big (- \eta_{\ref{fig:DP_grad_desc}}^2 \mathbb{E}[\|\grad f(\vec{w}_{g, i})\|^2] +
2\eta_{\ref{fig:DP_grad_desc}}^2 \mathbb{E}[\langle \grad f(\vec{w}_{g, i}), \grad g(\vec{w}_{g, i}) - \grad f(\vec{w}_{g, i}) \rangle] \Big ) \nonumber \\
&\quad\quad + \eta_{\ref{fig:DP_grad_desc}}^2 \zeta^2   +
\eta_{\ref{fig:DP_grad_desc}}^2 m \sigma^2 \nonumber \\
&\leq (1 - \frac{\eta_{\ref{fig:DP_grad_desc}} \nu}{2}) \mathbb{E}[\|\vec{w}_{g, i} - \vec{w}^*_f \|^2 +
\frac{2\eta_{\ref{fig:DP_grad_desc}} \zeta^2}{\nu} + 2\eta_{\ref{fig:DP_grad_desc}}^2 \zeta^2 + \eta_{\ref{fig:DP_grad_desc}}^2 m \sigma^2 \label{eq:final_bound} \\ 
&= (1 - \frac{\eta_{\ref{fig:DP_grad_desc}} \nu}{2}) \mathbb{E}[\|\vec{w}_{g, i} - \vec{w}^*_f \|^2] +
2\eta_{\ref{fig:DP_grad_desc}} \zeta^2 \left ( \frac{1}{\nu} + \eta_{\ref{fig:DP_grad_desc}} \right ) +
\eta_{\ref{fig:DP_grad_desc}}^2 m \sigma^2, \label{eq:final_bound_2}
\end{align}
where 
(\ref{eq:final_bound}) holds 
due to the following two claims:
\begin{claim}
    \begin{align*}
        \mathbb{E}[\langle \vec{w}_{g, i} - \vec{w}^*_f, \grad f(\vec{w}_{g, i}) - \grad g(\vec{w}_{g, i})\rangle]
\leq \frac{\nu}{4} \mathbb{E}[\|\vec{w}_{g, i} - \vec{w}^*_f\|^2] + \frac{\zeta^2}{\nu}
    \end{align*}
\end{claim}

\begin{proof}

\begin{align*}
\mathbb{E}[\langle \vec{w}_{g, i} - \vec{w}^*_f, \grad f(\vec{w}_{g, i}) - \grad g(\vec{w}_{g, i})\rangle]
&=
\mathbb{E}[\langle \sqrt{\frac{\nu}{2}} (\vec{w}_{g, i} - \vec{w}^*_f), \sqrt{\frac{2}{\nu}} (\grad f(\vec{w}_{g, i}) - \grad g(\vec{w}_{g, i})) \rangle]\\
&\leq
\frac{\nu}{4} \mathbb{E}[ \|\vec{w}_{g, i} - \vec{w}^*_f\|^2]
+ \frac{1}{\nu} \mathbb{E}[\|\grad f(\vec{w}_{g, i}) - \grad g(\vec{w}_{g, i})\|^2] \rangle]\\
&\leq
\frac{\nu}{4} \mathbb{E}[ \|\vec{w}_{g, i} - \vec{w}^*_f\|^2]
+ \frac{\zeta^2}{\nu},
\end{align*}
where the final inequality holds due to Assumption~\ref{ass:grad_dist}.
\end{proof}

\begin{claim}
    \begin{align*}
        -\mathbb{E}[\|\grad f(\vec{w}_{g, i})\|^2] + 2 \mathbb{E}[\langle \grad f(\vec{w}_{g, i}), \grad g(\vec{w}_{g, i}) - \grad f(\vec{w}_{g, i}) \rangle]
\leq \zeta^2
    \end{align*}
\end{claim}

\begin{proof}

\begin{align*}
&-\mathbb{E}[\|\grad f(\vec{w}_{g, i})\|^2] + 2\mathbb{E}[\langle \grad f(\vec{w}_{g, i}), \grad g(\vec{w}_{g, i}) - \grad f(\vec{w}_{g, i}) \rangle]\\
&\leq
-\mathbb{E}[\|\grad f(\vec{w}_{g, i})\|^2] + 
\mathbb{E}[\|\grad f(\vec{w}_{g, i})\|^2]
+ 
\mathbb{E}[\|\grad g(\vec{w}_{g, i}) - \grad f(\vec{w}_{g, i})\|^2]\\ 
&\leq \zeta^2,
\end{align*}
where the final inequality holds due to Assumption~\ref{ass:grad_dist}.
\end{proof}

The analysis from above culminating with (\ref{eq:final_bound_2}) yields the following bound on
$\mathbb{E}[\|\vec{w}_{g, i+1} - \vec{w}^*_f\|^2]$ in terms of $\mathbb{E}[\|\vec{w}_{g, i} - \vec{w}^*_f\|^2]$:

\begin{equation} \label{eq:final_bound_3}
\mathbb{E}[\|\vec{w}_{g, i+1} - \vec{w}^*_f\|^2] \leq
(1 - \frac{\eta_{\ref{fig:DP_grad_desc}} \nu}{2}) \mathbb{E}[\|\vec{w}_{g, i} - \vec{w}^*_f \|^2] +
2\eta_{\ref{fig:DP_grad_desc}} \zeta^2 \left ( \frac{1}{\nu} + \eta_{\ref{fig:DP_grad_desc}} \right ) +
\eta_{\ref{fig:DP_grad_desc}}^2 m \sigma^2.
\end{equation}

Applying the bound 
in
(\ref{eq:final_bound_3}) 
iteratively and taking expectation with
respect to $\vec{w}_{T}, \vec{w}_{T-1}, \ldots, \vec{w}_0$ yields the desired bound:
\begin{align}
\mathbb{E}[\|\vec{w}_{g,T} - \vec{w}^*_f\|^2]
&\leq \left ( 1 - \frac{\eta_{\ref{fig:DP_grad_desc}} \nu}{2} \right )^T \|\vec{w}_{g,0} - \vec{w}^*_f\|^2 +\left(
2\eta_{\ref{fig:DP_grad_desc}} \zeta^2 \left ( \frac{1}{\nu} +  \eta_{\ref{fig:DP_grad_desc}} \right ) +
\eta_{\ref{fig:DP_grad_desc}}^2 m \sigma^2\right ) \sum_{i=0}^T \left ( 1 - \frac{\eta_{\ref{fig:DP_grad_desc}} \nu}{2} \right )^i \nonumber \\
&\leq \left ( 1 - \frac{\eta_{\ref{fig:DP_grad_desc}} \nu}{2} \right )^T \|\vec{w}_{g,0} - \vec{w}^*_f\|^2 +\left(
2\eta_{\ref{fig:DP_grad_desc}} \zeta^2 \left ( \frac{1}{\nu} +  \eta_{\ref{fig:DP_grad_desc}} \right ) +
\eta_{\ref{fig:DP_grad_desc}}^2 m \sigma^2\right ) \frac{2}{\eta_{\ref{fig:DP_grad_desc}} \nu} \nonumber \\
&= \left ( 1 - \frac{\eta_{\ref{fig:DP_grad_desc}} \nu}{2} \right )^T \|\vec{w}_{g,0} - \vec{w}^*_f\|^2 + \frac{4 \zeta^2}{\nu} \left ( \frac{1}{\nu} + \eta_{\ref{fig:DP_grad_desc}} \right) +
\frac{2\eta_{\ref{fig:DP_grad_desc}} m \sigma^2}{\nu}.\nonumber
\end{align}
\end{proof}

\section{Proof of Theorem~\ref{th:dp_no_clipping} and Missing Analyses from Section~\ref{sec:no_clip}}

We begin by restating Theorem~\ref{th:dp_no_clipping}.

\DiffPriv*

To prove Theorem~\ref{th:dp_no_clipping}, we first bound the maximum magnitude of the weights, $R$, encountered during gradient descent in Lemma~\ref{lem:weight_magnitude} (see Appendix~\ref{sec:proof_Lemma_bounding_weights} for the proof).
We then use this to define the approximation interval $[-\sqrt{m}R, \sqrt{m}R]$ for approximations $\tilde{p}'_0$ and $\tilde{p}'_1$ relative to target functions $\phi'_0$ and $\phi'_1$. This then allows us to bound the sensitivity in
Appendix~\ref{sec:sensitivy_bound}. Finally, by employing standard DP analysis to set the noise variance $\sigma$, we achieve the full DP guarantee.

\subsection{Proof of Lemma~\ref{lem:weight_magnitude}}
\label{sec:proof_Lemma_bounding_weights}

We begin by restating Lemma~\ref{lem:weight_magnitude}.

\WeightNorm*

We prove the lemma by induction on the iteration index \(i\). The base case is immediate, since Algorithm~\ref{fig:modified_grad_desc} is initialized at \(\vec{w}_0 = \vec{0}\).

For the inductive step, we analyze a single update of the form
\[
\vec{w}_{i+1}
= \vec{w}_i
- \eta_{\ref{fig:modified_grad_desc}}
\Bigl(
2\lambda P_\kappa(\Theta - \|\vec{w}_i\|^2)\vec{w}_i
+ \frac{1}{N}\!\sum_{(\vec{x},y)\in D}
\tilde{p}'(\langle \vec{w}_i,\vec{x}\rangle,y)\vec{x}
+ \boldsymbol{\chi}_i
\Bigr).
\]
The update decomposes naturally into three components.
The first term corresponds to the gradient of the barrier function; since
\(P_\kappa(x)\) approximates \(1/x\), this term is approximately
\(\frac{2\lambda}{\Theta-\|\vec{w}_i\|^2}\vec{w}_i\) and acts to reduce the norm of \(\vec{w}_i\) as it approaches the boundary.
The second term approximates the gradient of the original objective \(f\).
The third term accounts for the injected Gaussian noise, whose magnitude we bound over all \(T\) iterations with high probability.

As \(\|\vec{w}_i\|^2\) approaches \(\Theta\), the barrier term increasingly pulls the iterate inward, while the data-dependent gradient and noise may push it outward. We therefore split the analysis into two cases.

In the first case, when \(\|\vec{w}_i\| \le \sqrt{(1-\kappa)\Theta}\), we choose parameters so that the barrier term cannot increase the norm. Bounding the contributions of the objective gradient, approximation error, and noise then yields
\[
\|\vec{w}_{i+1}\|
\le \sqrt{(1-\kappa)\Theta}
+ \eta_{\ref{fig:modified_grad_desc}}
\bigl(
\phi'_{\max}\sqrt{m}
+ \zeta_f
+ 2\lambda e_B\sqrt{\Theta}
+ (\sqrt{m}+\mathsf{c}_\delta)\sigma
\bigr).
\]

In the second case, when \(\|\vec{w}_i\| \ge \sqrt{(1-\kappa)\Theta}\), we show that the inward pull of the barrier term dominates any outward contribution from the objective gradient and noise. This requires a refined decomposition of these terms into components parallel and orthogonal to \(\vec{w}_i\). Under the parameter constraints in
(\ref{eq:eta_rest}) and (\ref{eq:kappa_constraint}), this analysis implies
\(\|\vec{w}_{i+1}\| \le \|\vec{w}_i\|\).

Combining both cases with the inductive hypothesis yields the desired bound on \(\|\vec{w}_i\|\) for all iterations, completing the proof.

We now proceed with the formal proof

\begin{proof}[Proof] 
We first note that due to the setting of $\mathsf{c}_\delta$, by standard analysis, with probability
$1-\frac{2\delta}{3}$,
$\|\boldsymbol\chi_i\| \leq (\sqrt{m} + \mathsf{c}_\delta)\sigma$ and
$\frac{\langle \boldsymbol\chi_i, \vec{w}_i \rangle}{\|\vec{w}_i\|} \leq \mathsf{c}_\delta \cdot \sigma$,
for all $i \in \{0, \ldots, T\}$.

Assuming the above bounds hold for all $i \in \{0, \ldots, T\}$,
we proceed to argue by induction:

\noindent
\underline{Base case:} $\vec{w}_0$ is initialized to the all $0$ vector in so
$\|\vec{w}_{0}\| = 0$.

\noindent
\underline{Inductive case.}
Assume that
$\|\vec{w}_{i}\|
\leq \sqrt{(1-\kappa)\Theta} +  \eta_{\ref{fig:modified_grad_desc}}\left(  \phi'_{max}\sqrt{m} + \zeta_f +
2\lambda e_B \sqrt{\Theta}
+ (\sqrt{m} + \mathsf{c}_\delta)\sigma \right)$.

We further split the inductive case into two sub-cases:

\underline{Sub-Case 1:}
Assume that 
$\|\vec{w}_{i}\|
\leq \sqrt{(1-\kappa)\Theta}$.


We have
\begin{align}
\|\vec{w}_{i+1}\| &=
\|\vec{w}_{i} - \eta_{\ref{fig:modified_grad_desc}} \sum_{(\vec{x}, y) \in D} p'(\langle \vec{w}_i, \vec{x} \rangle)\vec{x} 
-
2\tilde{p}'(z)\eta_{\ref{fig:modified_grad_desc}} \lambda P_\kappa(\Theta - \|\vec{w}_i\|^2)\vec{w}_i
-
\eta_{\ref{fig:modified_grad_desc}} \boldsymbol\chi_i\| \nonumber \\
&\leq \|\vec{w}_{i} - \eta_{\ref{fig:modified_grad_desc}} \grad f^{+B}_\kappa(\vec{w}_i)\| + \eta_{\ref{fig:modified_grad_desc}} \zeta_f+
2\eta_{\ref{fig:modified_grad_desc}} \lambda e_B \sqrt{(1-\kappa)\Theta} +
\eta_{\ref{fig:modified_grad_desc}} \|\boldsymbol\chi_i\| \nonumber \\
&\leq \|\vec{w}_{i} - 2 \eta_{\ref{fig:modified_grad_desc}} \lambda F_\kappa(\Theta - \|\vec{w}_i\|^2)\vec{w}_i\| + \eta_{\ref{fig:modified_grad_desc}} \phi'_{max} \sqrt{m} + \eta_{\ref{fig:modified_grad_desc}} \zeta_f+
2\eta_{\ref{fig:modified_grad_desc}} \lambda e_B \sqrt{(1-\kappa)\Theta} +
\eta_{\ref{fig:modified_grad_desc}} \|\boldsymbol\chi_i\| \nonumber \\
&\leq
\|\vec{w}_{i}\| + \eta_{\ref{fig:modified_grad_desc}} \phi'_{max} \sqrt{m} + \eta_{\ref{fig:modified_grad_desc}} \zeta_f +
2\eta_{\ref{fig:modified_grad_desc}} \lambda e_B \sqrt{(1-\kappa)\Theta}
+\eta_{\ref{fig:modified_grad_desc}} \|\boldsymbol\chi_i\| \label{eq:non-negative} \\
&\leq
\|\vec{w}_{i}\| + \eta_{\ref{fig:modified_grad_desc}}\left(  \phi'_{max} \sqrt{m} + \zeta_f +
2\lambda e_B \sqrt{\Theta}
+ (\sqrt{m} + \mathsf{c}_\delta)\sigma \right), \nonumber
\end{align}
where (\ref{eq:non-negative}) holds since in Sub-Case 1, 
$F_\kappa(\Theta - \|\vec{w}_i\|^2) \leq \frac{1}{\kappa \Theta}$ and since
$\eta_{\ref{fig:modified_grad_desc}} \leq \frac{\kappa \Theta}{\lambda}$,
we have that
$|1-2\eta_{\ref{fig:modified_grad_desc}}\lambda F_\kappa(\Theta - \|\vec{w}_i\|^2)| \leq 1$.

\noindent
\underline{Sub-Case 2:}
Assume that 
$\|\vec{w}_{i}\|
\geq \sqrt{(1-\kappa)\Theta}$.
First recall that 
$\| \grad f(\vec{w}_i) \| \leq \phi'_{max} \cdot \sqrt{m}$.
Further, since $f$ is convex,
we have that 
$\langle \grad f(\vec{w}_i) - \grad f(\vec{0}), \vec{w}_i\rangle \geq 0$.
This implies that

\[
\langle \grad{f}(\vec{w}_i), \vec{w}_i \rangle \geq \langle  \grad f(\vec{0}), \vec{w}_i \rangle
\geq
-\| \grad f(\vec{0})\| \|\vec{w}_i\|.
\]

We can write
\[
\grad{f}(\vec{w}_i) =
a \widetilde{\vec{w}}
+ b \widehat{\vec{w}},
\]
where $\widetilde{\vec{w}} = \frac{\vec{w}_i}{\|\vec{w}_i\|}$,
$\widehat{\vec{w}}$ is a unit vector
orthogonal to $\widetilde{\vec{w}}$,
and 
\begin{equation} \label{eq:a_b_bound}
-\| \grad f(\vec{0})\| \leq a \leq \phi'_{max} \cdot \sqrt{m}
\quad
\mbox{ and }
\quad 
c:= \sqrt{a^2 + b^2} \leq \phi'_{max} \cdot \sqrt{m}.
\end{equation}
Further, 
since we assume
\[
\frac{|\langle \boldsymbol\chi_i, \vec{w}_i \rangle|}{\|\vec{w}_i\|} \leq
\mathsf{c}_\delta \cdot \sigma,
\]
$\boldsymbol\chi_i$ can be written as:
\[
\boldsymbol\chi_i = a' \widetilde{\vec{w}}
+ b' \overline{\vec{w}},
\]
where $\overline{\vec{w}}$
is a unit vector orthogonal to 
$\widetilde{\vec{w}}$
and 
\begin{equation} \label{eq:a_b_prime_bound}
|a'| \leq \mathsf{c}_\delta \cdot \sigma
\quad
\mbox{and}
\quad
c' := \sqrt{(a')^2 + (b')^2} \leq (\sqrt{m} + \mathsf{c}_\delta)\sigma.
\end{equation}

We will use the following fact later in the proof:
\begin{equation}\label{eq:pythagorean}
(a + a')^2 + (|b| + |b'|)^2 \leq (c + c')^2.
\end{equation}
This can be seen since:
\begin{align*}
(a + a')^2 + (|b| + |b'|)^2 &=
a^2 + b^2 + (a')^2 + (b')^2 + 2aa' + 2|b||b'|\\
&= a^2 + b^2 + (a')^2 + (b')^2 + 2\langle (a,|b|), (a', |b'|) \rangle\\
&\leq a^2 + b^2 + (a')^2 + (b')^2 + 2\|(a,|b|)\| \cdot \|(a', |b'|)\|\\
&= c^2 + c'^2 + 2c \cdot c'\\
&= (c + c')^2.
\end{align*}

We thus have
\begin{align}
&\|\vec{w}_{i+1}\| =
\|\vec{w}_{i} - \eta_{\ref{fig:modified_grad_desc}} \sum_{(\vec{x}, y) \in D} p'(\langle \vec{w}_i, \vec{x} \rangle)\vec{x} 
-
2\tilde{p}'(z)\eta_{\ref{fig:modified_grad_desc}} \lambda P_\kappa(\Theta - \|\vec{w}_i\|^2)\vec{w}_i
-
\eta_{\ref{fig:modified_grad_desc}} \boldsymbol\chi_i\| \nonumber \\
&\leq \|\vec{w}_{i} - 
2\tilde{p}'(z)\eta_{\ref{fig:modified_grad_desc}} \lambda P_\kappa(\Theta - \|\vec{w}_i\|^2)\vec{w}_i
-\eta_{\ref{fig:modified_grad_desc}} \grad f(\vec{w}_i) - \eta_{\ref{fig:modified_grad_desc}} \boldsymbol\chi_i\| + \eta_{\ref{fig:modified_grad_desc}} \zeta_f \nonumber \\
&=
\sqrt{\|\vec{w}_{i} - 2\tilde{p}'(z)\eta_{\ref{fig:modified_grad_desc}} \lambda P_\kappa(\Theta - \|\vec{w}_i\|^2)\vec{w}_i - \eta_{\ref{fig:modified_grad_desc}} \cdot a \widetilde{\vec{w}} - \eta_{\ref{fig:modified_grad_desc}} \cdot a' \widetilde{\vec{w}}\|^2 +
\eta^2_{\ref{fig:modified_grad_desc}} \|b \widehat{\vec{w}} + b'\overline{\vec{w}}\|^2}
+
\eta_{\ref{fig:modified_grad_desc}} \zeta_f \nonumber \\
&\leq
\sqrt{\|\vec{w}_{i} - 2\tilde{p}'(z)\eta_{\ref{fig:modified_grad_desc}} \lambda P_\kappa(\Theta - \|\vec{w}_i\|^2)\vec{w}_i - \eta_{\ref{fig:modified_grad_desc}} \cdot a \widetilde{\vec{w}} - \eta_{\ref{fig:modified_grad_desc}} \cdot a' \widetilde{\vec{w}}\|^2 +
\eta^2_{\ref{fig:modified_grad_desc}}(\|b \widehat{\vec{w}}\| + \|b'\overline{\vec{w}}\|)^2}
+
\eta_{\ref{fig:modified_grad_desc}} \zeta_f \nonumber \\
&\leq
\sqrt{\|\vec{w}_{i} - 2\tilde{p}'(z)\eta_{\ref{fig:modified_grad_desc}} \lambda P_\kappa(\Theta - \|\vec{w}_i\|^2)\vec{w}_i - \eta_{\ref{fig:modified_grad_desc}} \cdot a \widetilde{\vec{w}} - \eta_{\ref{fig:modified_grad_desc}} \cdot a' \widetilde{\vec{w}}\|^2 +
\eta^2_{\ref{fig:modified_grad_desc}}(|b| + |b'|)^2}
+
\eta_{\ref{fig:modified_grad_desc}} \zeta_f \nonumber \\
&=
\sqrt{
\left ( \left ( 1 - 2\tilde{p}'(z)\eta_{\ref{fig:modified_grad_desc}} \lambda P_\kappa(\Theta - \|\vec{w}_i\|^2) \right )\|\vec{w}_i\| - \eta_{\ref{fig:modified_grad_desc}}(a + a')\right )^2 +
\eta^2_{\ref{fig:modified_grad_desc}}(|b| + |b'|)^2}
+
\eta_{\ref{fig:modified_grad_desc}} \zeta_f \nonumber \\
&=
\sqrt{
\left ( 1 - 2\tilde{p}'(z)\eta_{\ref{fig:modified_grad_desc}} \lambda P_\kappa(\Theta - \|\vec{w}_i\|^2) \right )^2\|\vec{w}_i\|^2 -2 \eta_{\ref{fig:modified_grad_desc}}\left ( 1 - 2\tilde{p}'(z)\eta_{\ref{fig:modified_grad_desc}} \lambda P_\kappa(\Theta - \|\vec{w}_i\|^2) \right )\|\vec{w}_i\|(a + a') +
\eta^2_{\ref{fig:modified_grad_desc}}(a + a')^2 +
\eta^2_{\ref{fig:modified_grad_desc}}(|b| + |b'|)^2}
+
\eta_{\ref{fig:modified_grad_desc}} \zeta_f \nonumber \\
&\leq
\sqrt{
 \left ( 1 - 2\tilde{p}'(z)\eta_{\ref{fig:modified_grad_desc}} \lambda P_\kappa(\Theta - \|\vec{w}_i\|^2) \right )^2\|\vec{w}_i\|^2 -2 \eta_{\ref{fig:modified_grad_desc}}\left ( 1 - 2\tilde{p}'(z)\eta_{\ref{fig:modified_grad_desc}} \lambda P_\kappa(\Theta - \|\vec{w}_i\|^2) \right )\|\vec{w}_i\|(a + a') +
\eta^2_{\ref{fig:modified_grad_desc}}(c + c')^2}
+
\eta_{\ref{fig:modified_grad_desc}} \zeta_f \label{eq:use_pythagorean} \\
&\leq
\sqrt{
\left ( 1 - 2\tilde{p}'(z)\eta_{\ref{fig:modified_grad_desc}} \lambda P_\kappa(\Theta - \|\vec{w}_i\|^2) \right )^2\|\vec{w}_i\|^2 -2 \eta_{\ref{fig:modified_grad_desc}}\left ( 1 - 2\tilde{p}'(z)\eta_{\ref{fig:modified_grad_desc}} \lambda P_\kappa(\Theta - \|\vec{w}_i\|^2) \right )\|\vec{w}_i\|(a + a') +
\eta^2_{\ref{fig:modified_grad_desc}} \cdot (\sqrt{m}\phi'_{max} + (\sqrt{m} + \mathsf{c}_\delta)\sigma)^2}
+
\eta_{\ref{fig:modified_grad_desc}} \zeta_f \label{eq:final_bound_w_i+1}
\end{align}
where 
(\ref{eq:use_pythagorean}) 
follows from (\ref{eq:pythagorean})
and
(\ref{eq:final_bound_w_i+1}) follows from
the bounds on $c, c'$ given in
(\ref{eq:a_b_bound}) and (\ref{eq:a_b_prime_bound}).

We now aim to upperbound the quantity
\begin{equation} \label{eq:pos_neg}
\left ( 1 - 2\tilde{p}'(z)\eta_{\ref{fig:modified_grad_desc}} \lambda P_\kappa(\Theta - \|\vec{w}_i\|^2) \right )^2\|\vec{w}_i\|^2 -2 \eta_{\ref{fig:modified_grad_desc}}\left ( 1 - 2\tilde{p}'(z)\eta_{\ref{fig:modified_grad_desc}} \lambda P_\kappa(\Theta - \|\vec{w}_i\|^2) \right )\|\vec{w}_i\|(a + a')
\end{equation}
from (\ref{eq:final_bound_w_i+1}).

First, for every fixed value $\|\vec{w}_i\|$
and for every fixed value
$V: = 2\tilde{p}'(z)\eta_{\ref{fig:modified_grad_desc}} \lambda P_\kappa(\Theta - \|\vec{w}_i\|^2)$ such that $1-V$ is positive, (\ref{eq:pos_neg}) is maximized by taking $(a + a')$ to be as small a \emph{negative value} as possible,
which by (\ref{eq:a_b_bound}), is $(a + a') = -(d\sqrt{m} + \mathsf{c}_\delta \cdot \sigma)$.
Substituting this into (\ref{eq:pos_neg}) we obtain:
\begin{equation} \label{eq:pos_neg_2}
\left ( 1 - V \right )^2\|\vec{w}_i\|^2 + 2 \eta_{\ref{fig:modified_grad_desc}}\left ( 1 - V) \right )\|\vec{w}_i\|(d\sqrt{m} + \mathsf{c}_\delta \cdot \sigma).
\end{equation}
Further, for every fixed value $\|\vec{w}_i\|$ and for all
values $V$, such that $(1-V)$ is positive,
(\ref{eq:pos_neg_2}) is maximized by taking
$V$ as small as possible,
which in Sub-Case 2, is $V = 2\tilde{p}'(z)\eta_{\ref{fig:modified_grad_desc}} \lambda \mathsf{m}_P$, since $\eta_{\ref{fig:modified_grad_desc}} \leq \frac{1}{2\mathsf{m}_P}$.
So when $(1-V)$ is positive,
this yields
\begin{equation*} 
\mathsf{Pos} := \left ( 1 - 2\tilde{p}'(z)\eta_{\ref{fig:modified_grad_desc}} \lambda \mathsf{m}_P \right )^2\|\vec{w}_i\|^2 + 2 \eta_{\ref{fig:modified_grad_desc}}\left ( 1 - 2\tilde{p}'(z)\eta_{\ref{fig:modified_grad_desc}} \lambda \mathsf{m}_P) \right )\|\vec{w}_i\|(d\sqrt{m} + \mathsf{c}_\delta \cdot \sigma).
\end{equation*}

Second, for every fixed value $\|\vec{w}_i\|$ and for every fixed value
$V: = 2\tilde{p}'(z)\eta_{\ref{fig:modified_grad_desc}} \lambda P_\kappa(\Theta - \|\vec{w}_i\|^2)$ such that $1-V$ is negative, (\ref{eq:pos_neg}) is maximized by taking $(a + a')$ to be as large a \emph{positive value} as possible,
which by (\ref{eq:a_b_bound}), is $(a + a') = (\phi'_{max} \cdot \sqrt{m} + \mathsf{c}_\delta \cdot \sigma)$.
Substituting this into (\ref{eq:pos_neg}) we obtain:
\begin{equation} \label{eq:pos_neg_2}
\left ( 1 - V \right )^2\|\vec{w}_i\|^2 -2 \eta_{\ref{fig:modified_grad_desc}}\left ( 1 - V) \right )\|\vec{w}_i\|(\phi'_{max} \cdot \sqrt{m} + \mathsf{c}_\delta \cdot \sigma).
\end{equation}
Further, for every fixed value $\|\vec{w}_i\|$ and for all
values $V$, such that $(1-V)$ is negative,
(\ref{eq:pos_neg_2}) is maximized by taking
$V$ as large as possible,
which in Sub-Case 2, is $V = 2\tilde{p}'(z)\eta_{\ref{fig:modified_grad_desc}} \lambda \mathsf{M}_P$.
So when $(1-V)$ is negative,
this yields
\begin{equation*} 
\mathsf{Neg} := \left ( 1 - 2\tilde{p}'(z)\eta_{\ref{fig:modified_grad_desc}} \lambda \mathsf{M}_P \right )^2\|\vec{w}_i\|^2 -2 \eta_{\ref{fig:modified_grad_desc}}\left ( 1 - 2\tilde{p}'(z)\eta_{\ref{fig:modified_grad_desc}} \lambda \mathsf{M}_P) \right )\|\vec{w}_i\|(\phi'_{max} \cdot \sqrt{m} + \mathsf{c}_\delta \cdot \sigma).
\end{equation*}

Thus,
\[
(\ref{eq:final_bound_w_i+1}) \leq
\sqrt{
\max\{\mathsf{Pos}, \mathsf{Neg} \} +
\eta^2_{\ref{fig:modified_grad_desc}} \cdot (\sqrt{m}\phi'_{max} +(\sqrt{m} + \mathsf{c}_\delta)\sigma)^2}
+
\eta_{\ref{fig:modified_grad_desc}} \zeta_f. 
\]

Finally, we show that
$\mathsf{Neg} \leq 
\mathsf{Pos}$.
Since $\phi'_{max} \geq d$,
this is implied by
\[
\mathsf{Neg} +  \eta_{\ref{fig:modified_grad_desc}}^2 (\phi'_{max}\cdot\sqrt{m} + \mathsf{c}_\delta \cdot \sigma)^2  \leq \mathsf{Pos} + 
\eta_{\ref{fig:modified_grad_desc}}^2 (d\sqrt{m} + \mathsf{c}_\delta \cdot \sigma)^2,
\]
or equivalently
\begin{equation}\label{eq:intermediate_pos_neg}
\left (\left ( 1 - 2\eta_{\ref{fig:modified_grad_desc}}\lambda\mathsf{M}_P \right )\|\vec{w}_i\| - \eta_{\ref{fig:modified_grad_desc}}(\phi'_{max}\cdot \sqrt{m} + \mathsf{c}_\delta \cdot \sigma)\right)^2 \leq
\left (\left ( 1 - 2\eta_{\ref{fig:modified_grad_desc}}\lambda\mathsf{m}_P \right )\|\vec{w}_i\| + \eta_{\ref{fig:modified_grad_desc}}(d\sqrt{m} + \mathsf{c}_\delta \cdot \sigma) \right)^2.
\end{equation}
(\ref{eq:intermediate_pos_neg}), in turn,
is implied by
\[
\left ( -1 + 2\eta_{\ref{fig:modified_grad_desc}}\lambda\mathsf{M}_P \right )\|\vec{w}_i\| + \eta_{\ref{fig:modified_grad_desc}}(\phi'_{max}\cdot \sqrt{m} + \mathsf{c}_\delta \cdot \sigma)
\leq
\left ( 1 - 2\eta_{\ref{fig:modified_grad_desc}}\lambda\mathsf{m}_P \right )\|\vec{w}_i\| + \eta_{\ref{fig:modified_grad_desc}}(d\sqrt{m} + \mathsf{c}_\delta \cdot \sigma).
\]
This is implied by
\[
\eta_{\ref{fig:modified_grad_desc}} \leq 
\frac{1}{\lambda(\mathsf{M}_P+\mathsf{m}_P) + \frac{(\phi'_{max} - d) \sqrt{m}}{2\sqrt{(1-\kappa)\Theta}}},
\]
which is the restriction we place on 
$\eta_{\ref{fig:modified_grad_desc}}$ in (\ref{eq:eta_rest}).

Thus, we have that
\[
\| \vec{w}_{i+1}\| \leq 
\sqrt{
\left ( 1 - 2\tilde{p}'(z)\eta_{\ref{fig:modified_grad_desc}} \lambda \mathsf{m}_P \right )^2\|\vec{w}_i\|^2 + 2 \eta_{\ref{fig:modified_grad_desc}}\left ( 1 - 2\tilde{p}'(z)\eta_{\ref{fig:modified_grad_desc}} \lambda \mathsf{m}_P) \right )\|\vec{w}_i\|(d\sqrt{m} + \mathsf{c}_\delta \cdot \sigma) +
\eta^2_{\ref{fig:modified_grad_desc}} \cdot  ( \sqrt{m}\phi'_{max} + (\sqrt{m}+\mathsf{c}_\delta) \sigma)^2}
+
\eta_{\ref{fig:modified_grad_desc}} \zeta_f.
\]

We next argue that
\[
\sqrt{
\left ( 1 - 2\tilde{p}'(z)\eta_{\ref{fig:modified_grad_desc}} \lambda \mathsf{m}_P \right )^2\|\vec{w}_i\|^2 + 2 \eta_{\ref{fig:modified_grad_desc}}\left ( 1 - 2\tilde{p}'(z)\eta_{\ref{fig:modified_grad_desc}} \lambda \mathsf{m}_P) \right )\|\vec{w}_i\|(d\sqrt{m} + \mathsf{c}_\delta \cdot \sigma) +
\eta^2_{\ref{fig:modified_grad_desc}} \cdot  ( \sqrt{m}\phi'_{max} + (\sqrt{m}+\mathsf{c}_\delta) \sigma)^2}
+
\eta_{\ref{fig:modified_grad_desc}} \zeta_f \leq \|\vec{w}_i\|,
\]
which implies that
$\|\vec{w}_{i+1}\| \leq \|\vec{w}_i\|$.

Since $\alpha = 2\eta_{\ref{fig:modified_grad_desc}}\lambda\mathsf{m}_P$ is positive and less than $1$,
the inequality is satisfied as long as 
\[
\|\vec{w}_i\| \geq \frac{-B + \sqrt{B^2-4AC}}{2A},
\]
where
$A = (2\alpha - \alpha^2)$,
$B = -2\eta_{\ref{fig:modified_grad_desc}}((1-\alpha)(d\sqrt{m} + \mathsf{c}_\delta \cdot \sigma) + \zeta_f)$,
and $C = -\eta_{\ref{fig:modified_grad_desc}}^2((\sqrt{m}\phi'_{max}+(\sqrt{m} + \mathsf{c}_\delta)\sigma)^2-\zeta^2_f)$.
Since we are in Sub-Case 2, $\|\vec{w}_i\| \geq \sqrt{(1-\kappa)\Theta}$.
Recall that
we set
$\kappa$ so that:

\begin{align*}
\sqrt{(1-\kappa)\Theta} \geq 
\frac{-B + \sqrt{B^2-4AC}}{2A}.
\end{align*}

Thus, the inequality is indeed satisfied in Sub-Case 2.
By the inductive hypothesis this implies that
$\|\vec{w}_{i+1}\| \leq \|\vec{w}_{i}\| \leq
\sqrt{(1-\kappa)\Theta} +  \eta_{\ref{fig:modified_grad_desc}}\left(  \phi'_{max}\sqrt{m} + \zeta_f +
2\lambda e_B \sqrt{\Theta}
+ (\sqrt{m} + \mathsf{c}_\delta) \sqrt{m}\sigma \right)$.
This concludes the proof of Lemma~\ref{lem:weight_magnitude}.
\end{proof}

\subsection{Bounding the Sensitivity of Gradient Updates in Algorithm~\ref{fig:modified_grad_desc}}
\label{sec:sensitivy_bound}
Since $P_\kappa(\vec{w}_i)$ is data-independent, we can upper bound the $\ell_2$ sensitivity of $\grad g(\vec{w}_i)$ by $\frac{\Delta_2}{N}$, where
\begin{align}
\Delta_2 = \|\grad g_{(\vec{x},y)}(\vec{w}_i) -  \grad g_{(\vec{x}',y')}(\vec{w}_i)\| &= \|\tilde{p}'(\langle \vec{w}_i, \vec{x} \rangle, y) \cdot \vec{x}  - 
\tilde{p}'(\langle \vec{w}_i, \vec{x}' \rangle, y') \cdot \vec{x}' \| \nonumber \\
&\leq
\|\tilde{p}'(\langle \vec{w}_i, \vec{x} \rangle, y) \cdot \vec{x} \| + 
\|\tilde{p}'(\langle \vec{w}_i, \vec{x}' \rangle, y') \cdot \vec{x}' \| \nonumber \\
&=
|\tilde{p}'(\langle \vec{w}_i, \vec{x} \rangle, y)| \cdot \| \vec{x} \| +
|\tilde{p}'(\langle \vec{w}_i, \vec{x}' \rangle, y')| \cdot \|\vec{x}' \|\nonumber \\
&\leq
(|\phi'(\langle \vec{w}_i, \vec{x} \rangle, y)| + e_f) \cdot \| \vec{x} \| +
(|\phi'(\langle \vec{w}_i, \vec{x}' \rangle, y')| + e_f) \cdot \|\vec{x}' \| \label{eq:interval_bound} \\
&\leq 2(\phi'_{max} + e_f) \sqrt{m}. \nonumber
\end{align}
Cruciallly, (\ref{eq:interval_bound}) follows since, via Lemma~\ref{lem:weight_magnitude},
$\langle \vec{w}_i, \vec{x} \rangle$
is guaranteed to be contained in the interval
$[-\sqrt{m}R, \sqrt{m}R]$ and $\tilde{p}'$ is guaranteed to have error at most $e_f$  with respect to $\phi'$ on this interval.

Thus we finally obtain
$\Delta_2 \leq 2(\phi'_{max} + e_f) \sqrt{m}$,
where
$\phi'_{max} = \sup \{|\phi'(a, b)| : a \in (-\infty, \infty), b \in \{0,1\}\}$.

\subsection{Choosing polynomials and parameters for Algorithm~\ref{fig:modified_grad_desc}}\label{sec:choosing_parameters}

Choosing parameters is delicate since the constraint on $\kappa$ in (\ref{eq:kappa_constraint}) depends on the setting of $\eta_{\ref{fig:modified_grad_desc}}$ but the constraint on $\eta_{\ref{fig:modified_grad_desc}}$ in (\ref{eq:eta_rest})
depends on the settings of $\kappa$ and $P_\kappa$.

If one just wants to find \emph{some} feasible setting of parameters, this can be done as follows:
First, the constraint (\ref{eq:kappa_constraint}) on $\kappa$ given in Theorem~\ref{th:dp_no_clipping} is implied by the following constraint
(assuming $\alpha \leq 1$, which occurs when $\eta_{\ref{fig:modified_grad_desc}} \leq \frac{1}{2 \lambda \mathsf{m}_P}$ and is implied by (\ref{eq:eta_rest})):
\begin{equation}\label{eq:stronger_constraint}
\sqrt{(1-\kappa)\Theta} \geq
\frac{2(d\sqrt{m} + \mathsf{c}_\delta \cdot \sigma + \zeta_f) + \sqrt{8(d\sqrt{m} + \mathsf{c}_\delta \sigma + \zeta_f)^2 + 4((\sqrt{m}\phi'_{max} + (\sqrt{m} + \mathsf{c}_\delta)\sigma)^2 - \zeta^2_f)}}{2\lambda \mathsf{m}_P}.
\end{equation}
Note that this constraint is now independent of $\eta_{\ref{fig:modified_grad_desc}}$.
Furthermore, the right hand side of (\ref{eq:stronger_constraint}) decreases as $\kappa$ decreases.
This is because the value $\mathsf{m}_P$, which appears in the denominator of the right hand side, is equal to the value of
$P_\kappa(\kappa \Theta)$ (since $P_\kappa$ will be chosen such that it is decreasing on the interval $[\Theta - R^2, \kappa \Theta]$. By the approximation error bound, we can lower bound 
$P_\kappa(\kappa \Theta) \geq \frac{1}{\kappa \Theta} - e_B$ and this quantity increases as $\kappa$ decreases. 
On the other hand, the left hand side of (\ref{eq:stronger_constraint}) increases as $\kappa$ decreases.
Thus, there must be a setting of $\kappa$ that satisfies the equation, while all remaining parameters are fixed.

Once $\kappa$ is fixed,
we can determine $P_\kappa$ as follows:
First, we find a polynomial approximation, $\tilde{P}_\kappa$, of $1/x$ on the interval $[\frac{\kappa \Theta}{2}, \Theta]$ that preserves monotonicity and has error at most $\frac{e_B}{2}$. Such an approximation can be found by taking a sufficiently high degree polynomial using the works of~\cite{DeVore1977,DeVoreYu1985}. 
Since $\eta_{\ref{fig:modified_grad_desc}}$ is not yet determined, we assume
$\eta_{\ref{fig:modified_grad_desc}} \leq \frac{1}{2\beta_f}$ (which is the estimated learning rate needed for convergence).

To construct $P_\kappa$, we add an adjustment $h(x)$ to $\tilde{P}'_\kappa(x)$
and then set $P_\kappa := \tilde{P}_\kappa(x) + \tilde{h}(x)$ to be an antiderivative of
$\tilde{P}'(x) + h(x)$.
We choose $h(x)$
so that $\tilde{P}'_\kappa(x) + h(x)$ is negative
on the interval $[\Theta - \tilde{R}^2, \frac{\kappa \Theta}{2}]$. Further, we choose an $h(x)$ that is non-positive on $[\Theta - \tilde{R}^2, \Theta]$, $\tilde{P}'(x) + h(x)$ , so that $\tilde{P}'(x) + h(x)$ is negative on $[\frac{\kappa \Theta}{2}, \kappa \Theta]$.
Taken together, these ensure that
$P_\kappa$ is monotonically decreasing on the interval $[\Theta - \tilde{R}^2, \kappa \Theta]$.
To ensure that the error between $P_\kappa$ and $1/x$ does not increase by more than $\frac{e_B}{2}$ beyond the error of at most $\frac{e_B}{2}$ between
$\tilde{P}_\kappa$ and $1/x$ on the interval
$[\kappa \Theta, \Theta]$, we choose $h(x)$ such that
$|\tilde{h}(\kappa \Theta)|$ is decreasing on
$[\kappa \Theta, \Theta]$ and such that
$|\tilde{h}(\kappa \Theta)| \leq \frac{e_B}{2}$. Details follow. 

Let $M := \max \{|\Theta - \tilde{R}^2|, (1-\kappa)\Theta \}$.
Let $q \geq 0$ be a sufficiently large integer such that
for all $x \in [\Theta - \tilde{R}^2, \frac{\kappa \Theta}{2}]$,
\quad \quad 
\begin{equation} \label{eq:adjust_negative}
\tilde{P}'_\kappa(x) -\frac{e_B(q+1)}{2M} \left ( 1 - \frac{x - \kappa \Theta}{M} \right)^q \leq -1.
\end{equation}
Define 
\[
h(x) := -\frac{e_B(q+1)}{2M} \left ( 1 - \frac{x - \kappa \Theta}{M} \right)^q.
\]
Note that (\ref{eq:adjust_negative}) implies that $\tilde{P}'_\kappa(x) + h(x) < 0$ on the interval
$[\Theta - \tilde{R}^2, \frac{\kappa \Theta}{2}]$.
The fact that $\tilde{P}_\kappa$ is monotonically decreasing on
$[\frac{\kappa \Theta}{2}, \Theta]$ and $h(x)$ is non-positive on
$[\frac{\kappa \Theta}{2}, \Theta]$ implies that
$\tilde{P}'_\kappa(x) + h(x) < 0$ on the interval
$[\frac{\kappa \Theta}{2}, \Theta]$.
Let 
\[
\tilde{h}(x) := \frac{e_B}{2}\left(1-\frac{x-\kappa \Theta}{M}\right)^{q+1}
\]
be an antiderivative of $h(x)$.
Then $P_\kappa(x) := \tilde{P}_\kappa(x) + \tilde{h}(x)$ is monotonically decreasing on the interval $[\Theta - \tilde{R}^2, \kappa \Theta]$.
Moreover, since $|\tilde{h}(x)|$ is decreasing on the interval
$[\kappa \Theta, \Theta]$, the maximum error between $P_\kappa$ and $\frac{1}{x}$ is 
at most $\frac{e_B}{2} + |\tilde{h}(\kappa \Theta)| \leq e_B$
on the interval $[\kappa \Theta, \Theta]$.

Now that $P_\kappa$ is fixed, we set $\eta_{\ref{fig:modified_grad_desc}}$
sufficiently small so that it satisfies (\ref{eq:eta_rest}) and is at most $\frac{1}{2\beta_f}$. Note that this also fixes the value of $R$.

Finally, approximations $p'_0$ and $p'_1$ to $\phi'_0, \phi'_1$ with error at most $e_f$ over the interval $[-R\sqrt{m}, R\sqrt{m}]$ can be found by taking sufficiently high degree
using the works of~\cite{bernstein1912proof,Roulier70}.

The above always leads to a feasible setting of parameters, but it may not be the best practical choice. In Appendix~\ref{sec:guide_modified_GD}, we describe a heuristic procedure that leads to better parameters in practice.

\subsection{On the Convergence of Algorithm~\ref{fig:modified_grad_desc}}
\label{sec:modified_converge}

In order to apply Theorem~\ref{thm:dp_converge}, we must define $F_\kappa$ on the entire interval $[\Theta - R^2, \Theta]$, whereas it is currently only defined on $[\kappa \Theta, \Theta]$.
Once we do this, we can set $B_\kappa$ to be an antiderivative of $F_\kappa$ so that $f^{+B}_\kappa(\vec{w}) = f(\vec{w}) - \lambda B_\kappa(\Theta - \| \vec{w}\|^2)$ is fully defined on the domain
$\{\vec{w} : \| \vec{w}\| \leq R\}$.
We then analyze the strong convexity and smoothness of $f^{+B}_\kappa$ over the domain. Finally, we determine an upperbound on $\zeta$ by bounding the norm of the difference between the gradient of $f^{+B}_\kappa$:
\[
\grad f^{+B}_\kappa(\vec{w}) =
\frac{1}{N} \sum_{(\vec{x}, y) \in D} y \phi'_1(\langle \vec{w}, \vec{x} \rangle) +
 (1-y) \phi'_0(\langle \vec{w}, \vec{x} \rangle) 
+ 2\lambda F_\kappa(\Theta - \|\vec{w}\|^2)
\]
and the approximate gradient:
\[
\frac{1}{N}\sum_{(\vec{x}, y) \in D} y \tilde{p}'_1(\langle \vec{w}, \vec{x} \rangle) +
 (1-y) \tilde{p}'_0(\langle \vec{w}, \vec{x} \rangle) 
+ 2\lambda P_\kappa(\Theta - \|\vec{w}\|^2).
\]

\paragraph{Defining $F_\kappa$.} 
Recall that $F_\kappa(x) = \frac{1}{x}$ for $x \in [\kappa \Theta, \Theta]$ and that $P_\kappa(x)$ is defined on the entire interval $[\Theta - R^2, \kappa \Theta]$.
Recall that we further require $P_\kappa$ to be monotonically decreasing on the interval $[\Theta-R^2, \kappa \Theta]$.
For $x \in [\Theta - R^2, \kappa \Theta]$, we set
$F_\kappa(x) := (\frac{1}{\kappa \Theta} - P(\kappa \Theta)) + P_\kappa(x)$.

\paragraph{The Hessian of $-\lambda B_\kappa(\Theta - \|\vec{w}\|^2)$.}
The 
Hessian of the function $-\lambda B_\kappa(\Theta - \|\vec{w}\|^2)$ has the form
\[
4\lambda B''(\Theta - \|\vec{w}\|^2) \cdot \vec{w} \cdot \vec{w}^T - 2\lambda B'(\Theta - \|\vec{w}\|^2) \cdot \vec{I}_m.
\]
Since the first term is an outer product, it has a single non-zero eigenvalue which is equal to $4\lambda B''(\Theta - \|\vec{w}\|^2) \| \vec{w}\|^2$.
The second term has $m$ eigenvalues
$- 2\lambda B'(\Theta - \|\vec{w}\|^2)$.
Re-written in terms of $F_\kappa$, these are equal to
$-4\lambda F'_\kappa(\Theta - \|\vec{w}\|^2) \| \vec{w}\|^2$.
and
$2\lambda F_\kappa(\Theta - \|\vec{w}\|^2)$.

\paragraph{Strong Convexity of $f^{+B}_\kappa$.}
Given the above, the strong convexity parameter of
$-\lambda B_\kappa$ is equal to the minimum value of
$-4\lambda F'_\kappa(\Theta - \|\vec{w}\|^2) \| \vec{w}\|^2 + 2\lambda F_\kappa(\Theta - \|\vec{w}\|^2)$
over the domain
$\{\vec{w} : \| \vec{w} \|^2 \leq R\}$.
Since $F_\kappa$ is strictly decreasing, $-4\lambda F'_\kappa(\Theta - \|\vec{w}\|^2)$ is always non-negative, therefore
the lower bound of the first term is $0$ and occurs when $\|\vec{w}\|^2 = 0$.
The second term is also non-negative and the lower bound is $\frac{2\lambda}{\Theta}$ and occurs when $\|\vec{w}\|^2 = 0$.
Thus, the total strong convexity parameter of $f^{+B}_\kappa$ is equal to $\mu_{f^{+B}_\kappa} = \mu_f + \frac{2\lambda}{\Theta}$.

\paragraph{Smoothness of $f^{+B}_\kappa$.}
The smoothness parameter of
$-\lambda B_\kappa$ is equal to the maximum value of
$-4\lambda F'_\kappa(\Theta - \|\vec{w}\|^2) \| \vec{w}\|^2 + 2\lambda F_\kappa(\Theta - \|\vec{w}\|^2)$
over the domain
$\{\vec{w} : \| \vec{w}\| \leq R\}$.
Let $\mathsf{m}_{F'}$ denote the minimum of
$F'_\kappa(x)$ over the interval $[\Theta-R^2, \kappa \Theta]$. 
The minimum of $F'_\kappa(x)$ over the interval $[\kappa \Theta, \Theta]$ is $-1/(\kappa\Theta)^2$.
Since $F_\kappa(x)$ is decreasing on $[\Theta - R^2, \Theta]$, the maximum of the second term is
$F_\kappa(\Theta - R^2) = (\frac{1}{\kappa \Theta} - P(\kappa \Theta)) + P_\kappa(\Theta - R^2)$ on the interval
$[\Theta - R^2, \kappa \Theta]$
and $\frac{1}{\kappa \Theta}$ on the interval
$[\kappa \Theta, \Theta]$.
Thus, the total smoothness parameter of $f^{+B}_\kappa$ is equal to $\beta_{f^{+B}_\kappa} = \beta_f + \max \{ -4\lambda \mathsf{m}_{F'} + (\frac{1}{\kappa \Theta} - P(\kappa \Theta)) + P_\kappa(\Theta - R^2), 
\frac{4 \lambda (1-\kappa) \Theta}{(\kappa \Theta)^2} + \frac{1}{\kappa \Theta}\}$.

\paragraph{Setting $\zeta$.}
By the definition of
$F_\kappa$ on the interval
$[\Theta-R^2, \kappa \Theta]$, the difference between 
$|F_\kappa(x)-P_\kappa(x)| = |F_\kappa(\kappa \Theta) - P_\kappa(\kappa \Theta)|$.
On the other hand,
$|F_\kappa(\kappa \Theta) - P_\kappa(\kappa \Theta)| \leq e_B$.
Thus, the maximum error between
$F_\kappa$ and $P_\kappa$ on the entire interval
$[\Theta - R^2, \Theta]$ is at most 
$e_B$.
The difference in gradients is therefore
at most $2\lambda e_B \| \vec{w} \| \leq 2\lambda e_B R$.
Thus, the total value of $\zeta$
is at most $e_f \sqrt{m} + 2\lambda e_B R =
\zeta_f + 2\lambda e_B R$.

\section{Guidance on Polynomial Approximation and Hyperparameter Selection}\label{sec:guidance}

For concreteness, let us assume that the objective funciton $f$
has the form
\[
f(\vec{w}) = \sum_{(\vec{x}, y) \in D} y \phi_1(\langle \vec{w}, \vec{x} \rangle) + (1-y)\phi_0(\langle \vec{w}, \vec{x} \rangle).
\]
In particular, our experiments pertain to the logistic regression cost function, which has the above form.
In this case, the approximate objective function $g$ has the form
\[
g(\vec{w}) = \sum_{(\vec{x}, y) \in D} y \tilde{p}_1(\langle \vec{w}, \vec{x} \rangle)
+ (1-y) \tilde{p}_0(\langle \vec{w}, \vec{x} \rangle),
\]
where for $y \in \{0,1\}$,
$\tilde{p}'_y$ is an $e_f$-error approximation to 
$\phi'_y$ over some interval
$[-z_{min}, z_{max}]$.

The quality and effectiveness of training under the approximate objective function $g$ are governed by three critical parameters:

\begin{enumerate}
    \item The smoothness constant $\beta_g$ of the function $g$.
    \item The approximation error $e_f$ incurred when replacing the true gradient term $\phi'_y$ with its polynomial approximation $\tilde{p}'_y$.
    \item The approximation interval $[-z_{\min}, z_{\max}]$ over which the guarantees of the polynomial approximation hold.
\end{enumerate}

These parameters impact the training outcome in the following ways:

\begin{enumerate}
    \item The smoothness constant $\beta_g$ directly influences the convergence rate of the gradient descent algorithm. Specifically, a smaller $\beta_g$ yields faster convergence.
    
    \item The approximation error $e_f$ affects the proximity of the output weights $\vec{w}_{g,T}$, produced by approximate gradient descent, to the optimal weights $\vec{w}^*_f$ that minimize the original objective function $f$. A smaller $e_f$ leads to outputs that more closely approximate the true optimum.
    
    \item The choice of approximation interval is crucial for ensuring convergence. If the interval $[-z_{\min}, z_{\max}]$ is not selected appropriately, the guarantees provided by the polynomial approximation may not hold. In such cases, the approximate gradient descent algorithm can diverge, producing weight vectors with unbounded magnitude (i.e., $\|\vec{w}_{g,T}\| \to \infty$), and leading the approximate objective value $g(\vec{w}_{g,T})$ to diverge toward $-\infty$.
    This does not contradict Theorem~\ref{thm:npdp_converge}/Claim~\ref{claim:critical_point}, because the theorem/claim assumes the existence of a lower bound $L$ on the objective. In the cases considered here, $g$ has no global minimum and therefore no such lower bound exists.
\end{enumerate}

We will illustrate each of these effects in section \ref{sec:guidance_smooth} and \ref{sec:guidance_interval}, including an example that demonstrates the consequences of an improperly chosen approximation interval. Motivated by these considerations, the next subsection \ref{sec:guide_modified_GD} provides guidance for selecting appropriate polynomial approximations and training hyperparameters, enabling stable, privacy-preserving training under fully homomorphic encryption.

\subsection{Smoothness, Approximation Error, and Convergence Behavior}\label{sec:guidance_smooth}

In the setting of outsourced machine learning under fully homomorphic encryption (FHE), eliminating the need for data-dependent hyperparameter tuning is essential. Unlike conventional training pipelines, where convergence can be monitored and hyperparameters such as the learning rate or the number of training steps can be adjusted dynamically, outsourced training under non-interactive FHE must operate entirely over encrypted data and without interaction.

To support such a non-interactive protocol, one must either (1) enable the client to privately determine appropriate hyperparameters in a lightweight \emph{pre-processing} phase and communicate them in a differentially-private manner to the server, or (2) allow the server to select hyperparameters in a manner that is fully data-independent. In both cases, hyperparameters must be fixed in advance to enable one-shot execution by the server, without revealing intermediate computation or requiring feedback from the client.

Importantly, improperly selected hyperparameters may prevent convergence entirely. This is precisely where our theoretical results provide critical guidance. Theorem~\ref{thm:npdp_converge}/Claim~\ref{claim:critical_point} establish that setting the learning rate to $\eta = \frac{1}{\beta_g}$ and the number of gradient descent iterations to
\begin{equation} \label{eq:num_iter}
T = \frac{2\beta_g\left(g(\vec{w}_{g,0}) - L)\right)}{\rho^2},
\end{equation}
for a fixed $\rho > 0$, ensures convergence. These values depend on the properties of the polynomial approximation used for training and can be determined without reference to the underlying data distribution.

Moreover, our experimental results demonstrate that these hyperparameters can be estimated in a data-independent manner and still yield effective training performance in practice. This enables efficient, private, and non-interactive outsourced training entirely on the encrypted data.

In our experiments, we consider the original objective function $f$ to be the cost function for logistic regression and we set
\begin{itemize}
    \item Learning rate as
$\eta_{\ref{fig:grad_desc}}  \approx \frac{1}{\beta_\text{approx}}$, where $\beta_\text{approx}=\max\{\tilde{p}'(z)\}\cdot m$, is an upperbound on the smoothness constant $\beta_g$, and $\max\{\tilde{p}'(z)\}$ is the maximum value of $\tilde{p}'_0(z), \tilde{p}'_1(z)$ over the interval $[-z_{min}, z_{max}]$. 
\footnote{Note that one could instead use the potentially smaller value of $\frac{\max\{\tilde{p}'(z)\}}{N} \sum_{\vec{x}}\vec{x}^T \vec{x}$ as an upperbound on the smoothness. This quantity is data-dependent, but the Client can easily compute $\frac{1}{N} \sum_{\vec{x}}\vec{x}^T \vec{x}$ in a pre-processing step and release it to the Server in a differentially private manner.}
\item Training steps $T\approx \frac{2\beta_\text{approx} (g(\vec{w}_{g,0}))}{\rho^2}$ (taking $L = 0$).
\end{itemize}

\begin{table}[htbp]
\centering
\caption{Summary of Approximate-GD Training with Fixed Hyperparameters (7-Degree Polynomial, $\rho = 0.1$ for MNIST, $\rho = 0.05$ for others) using either minimax (MM) or least squares (LS) approximation. LR stands for learning rate.}
\begin{tabular}{@{}clllrrrrrr@{}}
\toprule
\textbf{Dataset} & \textbf{Pair} & \textbf{model} & \textbf{Method} & \textbf{Error} & \textbf{Smoothness} & \textbf{LR Est.} & \textbf{$T \cdot \rho^2$ Est.} & \textbf{Steps@$\rho$} & \textbf{Accuracy (\%)} \\
\midrule
\multirow{6}{*}{MNIST} 
  & \multirow{2}{*}{1} & model-1 & MM@[-20, 20]      & 0.1920 & 19.1185 & 0.0523 & 26.50 & 2650 & 96.3529 \\
  &                    & model-2 & LS@[-20, 20]      & 0.1439 & 35.1523 & 0.0284 & 48.73 & 4873 & 96.3937 \\
  \cmidrule{2-10}
  & \multirow{2}{*}{2} & model-1 & MM@[-25, 25]      & 0.2335 & 15.3325 & 0.0652 & 21.26 & 2125 & 96.2609 \\
  &                    & model-2 & LS@[-25, 25]      & 0.1834 & 27.4613 & 0.0364 & 38.07 & 3806 & 96.3328 \\
  \cmidrule{2-10}
  & \multirow{2}{*}{3} & model-1 & MM@[-30, 30]      & 0.3904 & 13.6093 & 0.0735 & 18.87 & 1886 & 96.1875 \\
  &                    & model-2 & LS@[-30, 30]      & 0.4250 & 16.8512 & 0.0593 & 23.36 & 2336 & 96.2694 \\
\midrule
\multirow{6}{*}{Adult}
  & \multirow{2}{*}{1} & model-1 & MM@[-10, 10]      & 0.0632 & 2.3223  & 0.4306 & 3.22  & 1287 & 84.4332 \\
  &                    & model-2 & LS@[-10, 10]      & 0.0498 & 2.5594  & 0.3907 & 3.55  & 1419 & 84.3991 \\
  \cmidrule{2-10}
  & \multirow{2}{*}{2} & model-1 & MM@[-15, 15]      & 0.1356 & 1.6680  & 0.5995 & 2.31  & 924  & 84.3650 \\
  &                    & model-2 & LS@[-15, 15]      & 0.0955 & 2.6682  & 0.3748 & 3.70  & 1479 & 84.3309 \\
  \cmidrule{2-10}
  & \multirow{2}{*}{3} & model-1 & MM@[-20, 20]      & 0.1920 & 1.2681  & 0.7886 & 1.76  & 703  & 84.3309 \\
  &                    & model-2 & LS@[-20, 20]      & 0.1439 & 2.3315  & 0.4289 & 3.23  & 1292 & 84.2285 \\
\midrule
\multirow{6}{*}{Compas}
  & \multirow{2}{*}{1} & model-1 & MM@[-7, 7]        & 0.0224 & 1.7384  & 0.5752 & 2.41  & 963  & 67.6212 \\
  &                    & model-2 & LS@[-7, 7]        & 0.0230 & 1.8121  & 0.5518 & 2.51  & 1004 & 67.6212 \\
  \cmidrule{2-10}
  & \multirow{2}{*}{2} & model-1 & MM@[-10, 10]      & 0.0632 & 1.4291  & 0.6997 & 1.98  & 792  & 67.4827 \\
  &                    & model-2 & LS@[-10, 10]      & 0.0498 & 1.5750  & 0.6349 & 2.18  & 873  & 67.5751 \\
  \cmidrule{2-10}
  & \multirow{2}{*}{3} & model-1 & MM@[-15, 15]      & 0.1356 & 1.0264  & 0.9742 & 1.42  & 569  & 67.2055 \\
  &                    & model-2 & LS@[-15, 15]      & 0.0955 & 1.6420  & 0.6090 & 2.28  & 910  & 67.4436 \\
\midrule
\multirow{6}{*}{Credit}
  & \multirow{2}{*}{1} & model-1 & MM@[-10, 10]      & 0.0632 & 4.1087  & 0.2434 & 5.70  & 2278 & 80.6229 \\
  &                    & model-2 & LS@[-10, 10]      & 0.0498 & 4.5281  & 0.2208 & 6.28  & 2510 & 80.6452 \\
  \cmidrule{2-10}
  & \multirow{2}{*}{2} & model-1 & MM@[-15, 15]      & 0.1356 & 2.9510  & 0.3389 & 4.09  & 1636 & 80.4378 \\
  &                    & model-2 & LS@[-15, 15]      & 0.0955 & 4.7207  & 0.2118 & 6.54  & 2617 & 80.4341 \\
  \cmidrule{2-10}
  & \multirow{2}{*}{3} & model-1 & MM@[-20, 20]      & 0.1920 & 2.2435  & 0.4457 & 3.11  & 1244 & 80.2970 \\
  &                    & model-2 & LS@[-20, 20]      & 0.1439 & 4.1250  & 0.2424 & 5.72  & 2287 & 80.2822 \\
\bottomrule
\end{tabular}
\label{tab:no-dp}
\end{table}

This work is the first to explicitly address the challenge of hyperparameter selection in the context of fully homomorphic encryption (FHE)-based training by proposing a principled method for estimating both the learning rate and the number of training iterations \emph{in advance}, based on the properties of the chosen polynomial approximation.

After establishing how to select the learning rate and training steps given a fixed polynomial approximation, we now turn to the choice of the polynomial itself.
In contrast to prior work, which either relies on manual, data-dependent, tuning or left them unspecified, we provide analytically grounded estimates. This advance enables fully automated, non-interactive, and privacy-preserving training under FHE that is both practical and efficient. Notably, prior FHE-based training works typically adopted a single polynomial approximation scheme, such as minimax or least-squares, without a principled justification of the resulting trade-offs. Our analysis shows that the choice of polynomial approximation has a direct impact on training behavior such as convergence rate and the optimization error. Even with same predefined approximation interval and with the same fixed degree, different approximation schemes can obtain polynomials with different smoothness and approximation error, further influencing overall training behavior.

\begin{figure}[htbp]
  \centering
  \begin{tabular}{ccc}
    \begin{subfigure}{0.3\textwidth}
      \includegraphics[width=\linewidth]{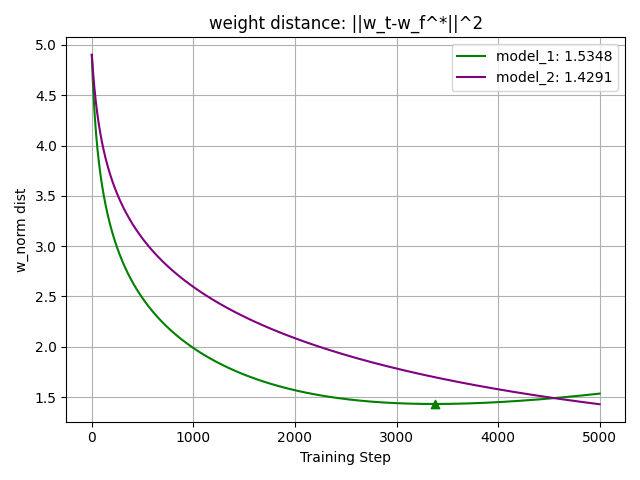}
      \caption{Mnist Pair-1}
    \end{subfigure} &
    \begin{subfigure}{0.3\textwidth}
      \includegraphics[width=\linewidth]{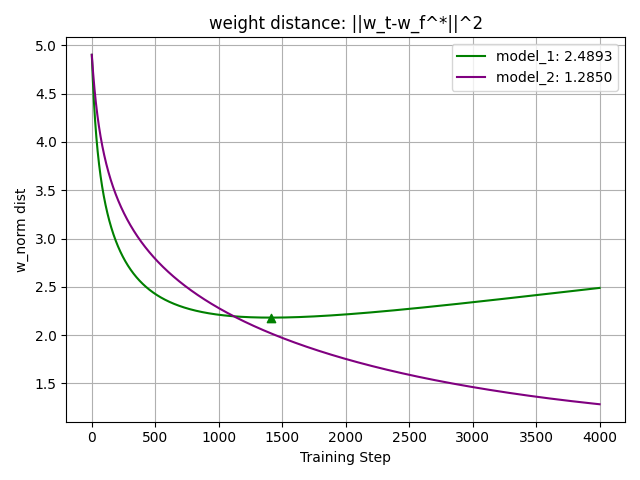}
      \caption{Mnist Pair-2}
    \end{subfigure} &
    \begin{subfigure}{0.3\textwidth}
      \includegraphics[width=\linewidth]{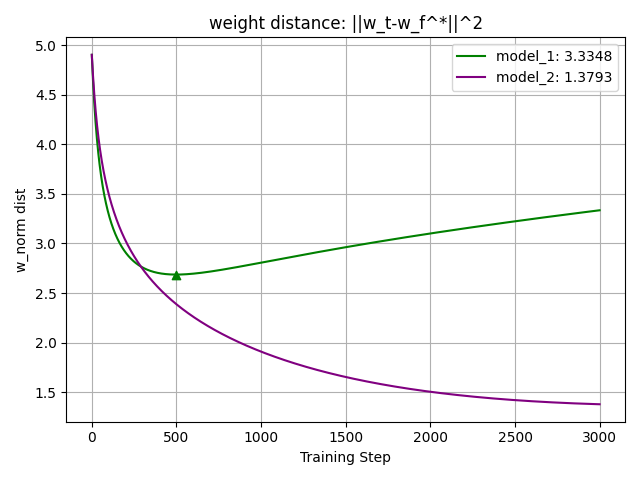}
      \caption{Mnist Pair-3}
    \end{subfigure} \\
    
    \begin{subfigure}{0.3\textwidth}
      \includegraphics[width=\linewidth]{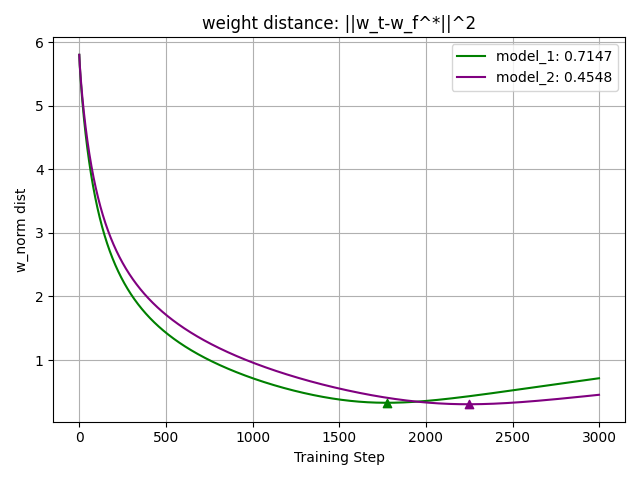}
      \caption{Adult Pair-1}
    \end{subfigure} &
    \begin{subfigure}{0.3\textwidth}
      \includegraphics[width=\linewidth]{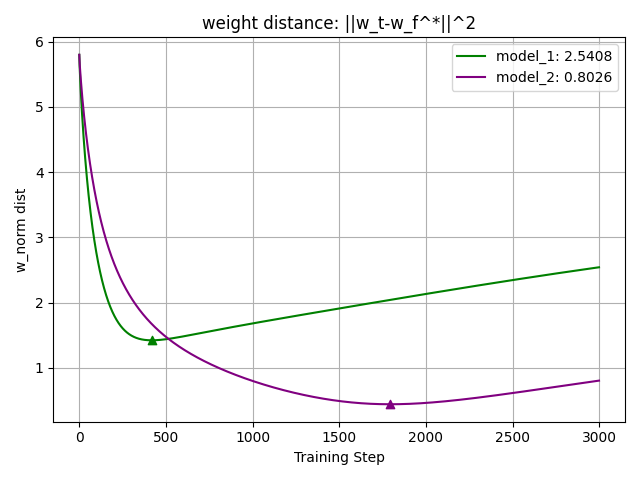}
      \caption{Adult Pair-2}
    \end{subfigure} &
    \begin{subfigure}{0.3\textwidth}
      \includegraphics[width=\linewidth]{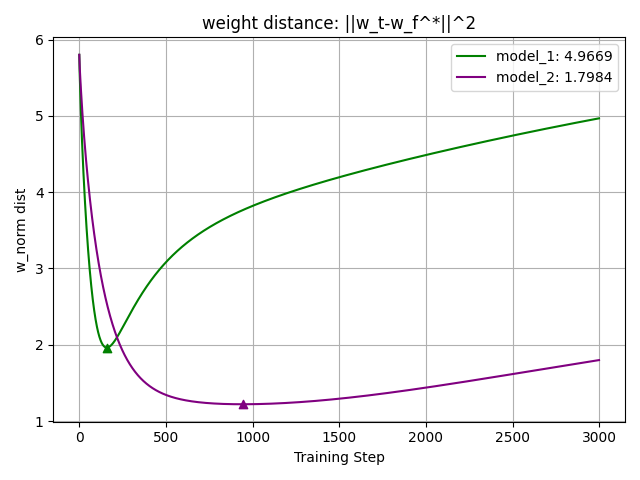}
      \caption{Adult Pair-3}
    \end{subfigure} \\
    
    \begin{subfigure}{0.3\textwidth}
      \includegraphics[width=\linewidth]{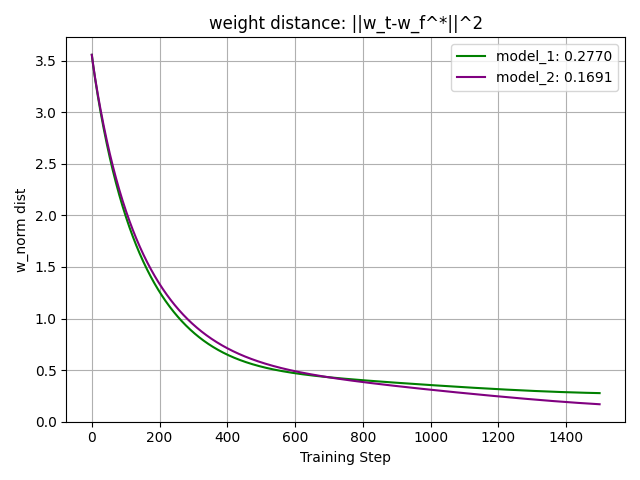}
      \caption{Compas Pair-1}
    \end{subfigure} &
    \begin{subfigure}{0.3\textwidth}
      \includegraphics[width=\linewidth]{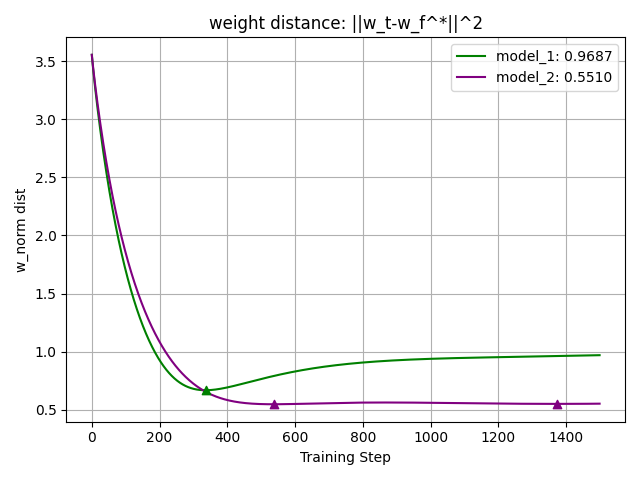}
      \caption{Compas Pair-2}
    \end{subfigure} &
    \begin{subfigure}{0.3\textwidth}
      \includegraphics[width=\linewidth]{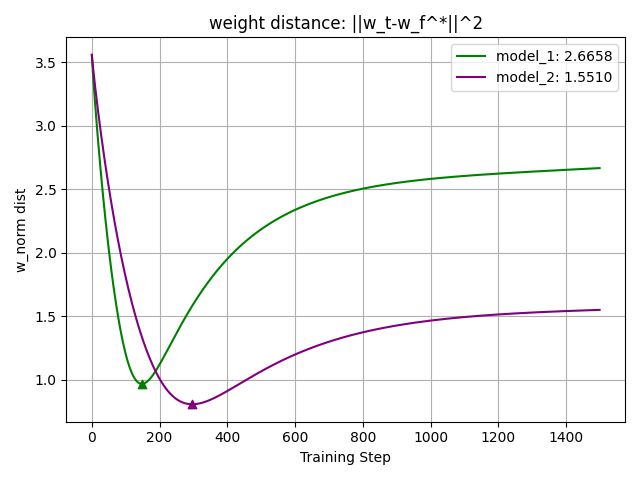}
      \caption{Compas Pair-3}
    \end{subfigure}\\

    \begin{subfigure}{0.3\textwidth}
      \includegraphics[width=\linewidth]{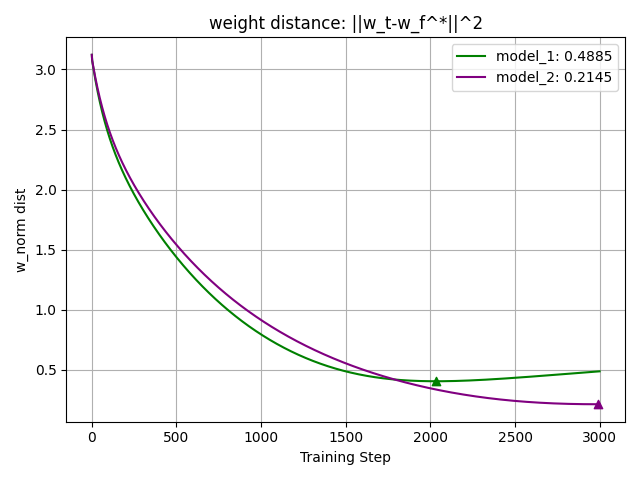}
      \caption{Credit Pair-1}
    \end{subfigure} &
    \begin{subfigure}{0.3\textwidth}
      \includegraphics[width=\linewidth]{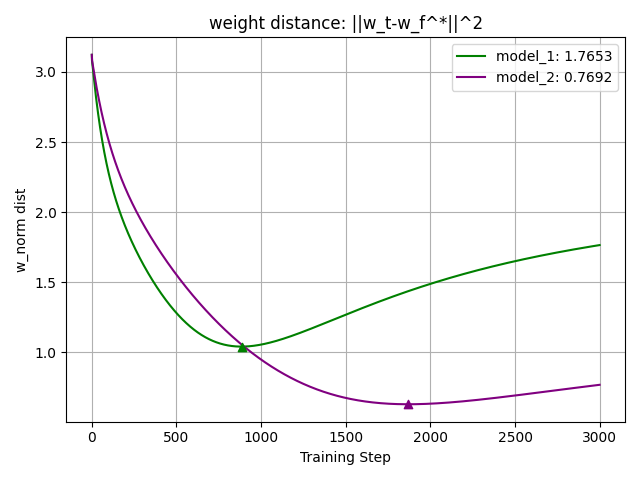}
      \caption{Credit Pair-2}
    \end{subfigure} &
    \begin{subfigure}{0.3\textwidth}
      \includegraphics[width=\linewidth]{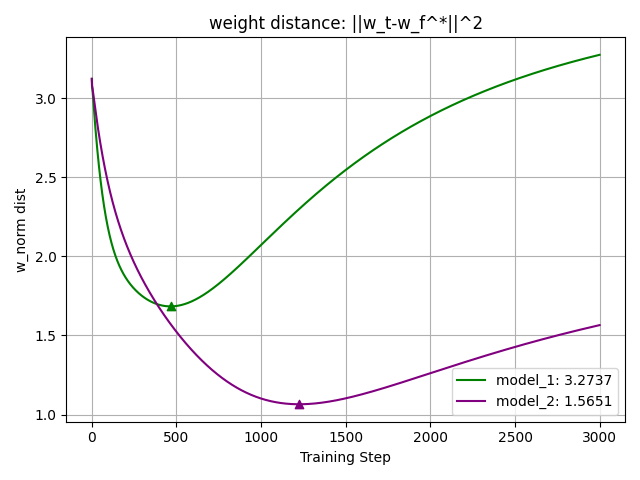}
      \caption{Credit Pair-3}
    \end{subfigure}
  \end{tabular}
  \caption{Pairwise comparisons of $\|\vec{w}_T - \vec{w}^*\|^2$ during training for two polynomial approximations.
Model~1 uses a minimax polynomial approximation, while Model~2 uses a least-squares approximation.
In each subfigure, the boxed values report the final optimization error $\|\vec{w}_T - \vec{w}^*\|^2$ after $T$ iterations.
We observe that Model~1 has a smaller $T\rho^2$ value (see Table~\ref{tab:no-dp}) and converges more rapidly, consistent with its stronger smoothness properties.
In contrast, Model~2 has lower approximation error (Table~\ref{tab:no-dp}) and, in most cases, achieves a smaller final optimization error.}
  \label{fig:no-dp test result}
\end{figure}

Our convergence analysis provides practical guidance for selecting the polynomial approximation used in training. In particular, the theoretical results indicate that convergence is faster when the smoothness parameter $\beta_g$ is smaller. As illustrated in Figure~\ref{fig:no-dp test result} and summarized in Table~\ref{tab:no-dp}, 
we compare training using the minimax and least-squares polynomial approximations of a fixed degree.
Interestingly, the experimental results establish that the two approximation schemes exhibit a clear and systematic trade-off. The minimax approximation consistently yields smaller smoothness (see Table~\ref{tab:no-dp}), and indeed exhibits faster convergence, and this is the case even when the learning rate is tuned to the data, as opposed to being set a priori (see Figure~\ref{fig:no-dp test result}). In contrast, the least squares approximation consistently achieves smaller approximation error (see Table~\ref{tab:no-dp}) and typically also achieves smaller optimization error $\|\vec{w}_{g,T} - \vec{w}^*_f \|^2$ (see Figure~\ref{fig:no-dp test result}).

Specifically, training with polynomial approximations that yield smaller values of $T \cdot \rho^2$, where $T$ and $\rho$ are defined in terms of $\beta_g$, results in more rapid convergence toward the optimal solution $\vec{w}_f^*$. Notably, this quantity can be approximated using a data-independent estimate $\beta_{\text{approx}}$, which may be computed by the server. This implies that selecting a polynomial approximation with smaller $\beta_\text{approx}$ allows similar training performance to be achieved using fewer gradient steps, thereby improving overall efficiency.

On the other hand, Theorem~\ref{thm:npdp_converge} shows that a larger value of $\zeta$, which captures the distance between the gradient of the approximate objective $g$ and that of the true objective $f$, corresponds to a greater deviation between the weights output by approximate gradient descent and the true minimizer $\vec{w}_f^*$. 

Recall that the gradient of the original objective function $f$ takes the form:
\[
\nabla f(\vec{w}) = \sum_{(\vec{x}, y) \in D} \phi'_y(\langle \vec{w}, \vec{x} \rangle) \cdot \vec{x},
\]
while the gradient of the approximated objective $g$ is:
\[
\nabla g(\vec{w}) = \sum_{(\vec{x}, y) \in D} \tilde{p}'_y(\langle \vec{w}, \vec{x} \rangle) \cdot \vec{x}.
\]
We adopt the heuristic that the approximation error $e_f$, defined as the maximum deviation between $\phi'_y$ and $\tilde{p}'_y$, contributes directly to $\zeta$. Importantly, $e_f$ is a data-independent quantity that can be computed by the server prior to training.

Empirical results support this relationship: in nearly all tested cases, larger values of $e_f$ lead to final output weights that lie farther from the optimum (see Figure~\ref{fig:no-dp test result}). These findings highlight the importance of careful polynomial selection, guided by both theoretical analysis and pre-computable approximation metrics.

\subsection{Approximation Interval and Training Stability}\label{sec:guidance_interval}

As discussed in Section~\ref{sec:guidance}, selecting an appropriate approximation interval is critical for ensuring both the accuracy and stability of training under polynomial approximations. If the interval is chosen too large, the approximation error increases for a fixed polynomial degree. This phenomenon is clearly demonstrated in Table~\ref{tab:no-dp}, where smaller intervals consistently yield lower approximation errors for the same degree of polynomial.

On the other hand, if the interval is chosen too small, there is a significant risk that the quantity $|\langle \vec{w}_i, \vec{x} \rangle|$ will exceed the approximation range during training. While in the case of logistic regression loss function is lower-bounded by zero regardless of the input magnitude, the polynomial approximation $g$ may be \emph{unbounded below} outside the interval. Specifically, as $|\langle \vec{w}_i, \vec{x} \rangle| \to \infty$, the value of the loss $g(\vec{w})$ can diverge to $-\infty$, resulting in the lower bound parameter $L$ in Equation~(\ref{eq:num_iter}) becoming unbounded. This can lead to instability in training: the approximate gradient descent procedure may diverge, producing weight vectors $\vec{w}_{g,T}$ with unbounded magnitude and corresponding loss values that tend toward $-\infty$.

An illustrative example of this behavior arises when training on the Adult dataset with the approximation interval set to $[-7, 7]$. As shown in Figure~\ref{fig:outbound_No-DP}, during later iterations, the value of $|\langle \vec{w}_i, \vec{x} \rangle|$ exceeds the interval, causing the magnitude of the weights to diverge and the loss to descend toward negative infinity. This ultimately leads to numerical instability, resulting in NaN values in both the training loss and gradient computations.

It is unclear how one can determine a suitable approximation interval for $|\langle \vec{w}_i, \vec{x} \rangle|$ in a non-dynamic or data-independent manner. In such cases, training can instead be performed using the modified objective function $f^{+B}_\kappa$, which incorporates a barrier function as defined in Section~\ref{sec:no_clip}. If differential privacy is not required, the barrier can be employed without injecting noise. 

This approach enforces a strict upper bound on the quantity $|\langle \vec{w}_i, \vec{x} \rangle|$ throughout training, regardless of the number of iterations $T$ (See Lemma~\ref{lem:weight_magnitude}). Consequently, with appropriately chosen approximation functions $\tilde{p}'$ and $P_\kappa$ (See Section~\ref{sec:no_clip}), the objective $g(\vec{w}_T)$ remains bounded from below by a constant $L$, and convergence is guaranteed via Theorem~\ref{thm:npdp_converge}. Figure~\ref{fig:inbound_No-DP} illustrates successful convergence using this barrier-augmented objective.

\begin{figure}[htbp]
    \centering
    \begin{subfigure}[t]{0.3\textwidth}
        \includegraphics[width=\linewidth]{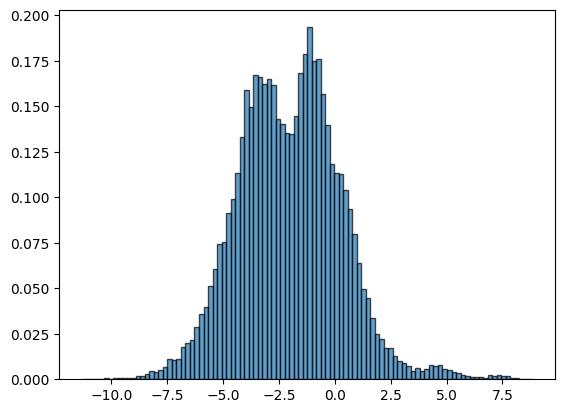}
        \caption{z values at t=1000}
        \label{fig:outbound_sub1}
    \end{subfigure}
    \hfill
    \begin{subfigure}[t]{0.3\textwidth}
        \includegraphics[width=\linewidth]{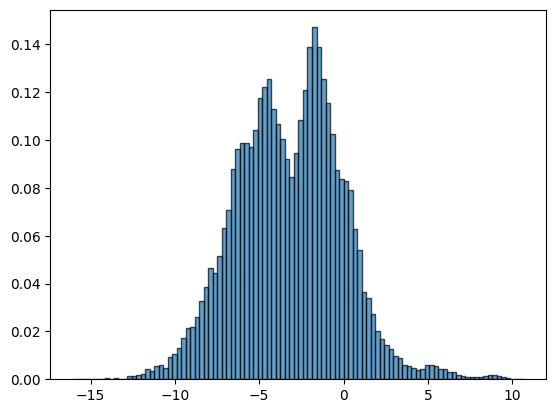}
        \caption{z values at t=1080}
        \label{fig:outbound_sub2}
    \end{subfigure}
    \hfill
    \begin{subfigure}[t]{0.3\textwidth}
        \includegraphics[width=\linewidth]{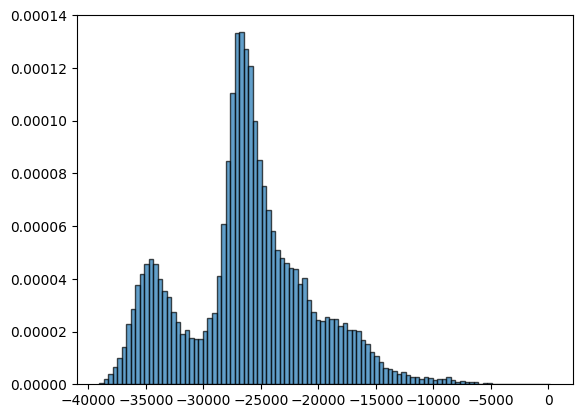}
        \caption{z values at t=1086}
        \label{fig:outbound_sub3}
    \end{subfigure}

    \vspace{0.5cm}

    \begin{subfigure}[t]{0.35\textwidth}
        \includegraphics[width=\linewidth]{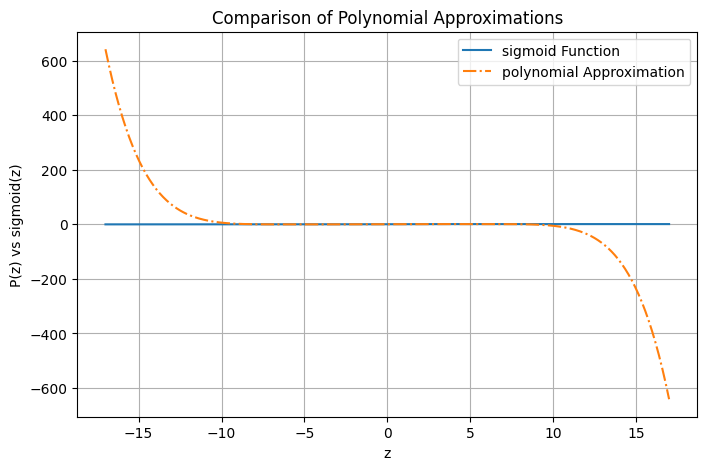}
        \caption{Comparison of sigmoid function and its polynomial approximation over a wide range. They exhibit close on $[-10, 10]$,  but diverge outside this range. Polynomial grows unbounded as $z\rightarrow \pm\infty$, while sigmoid remains bounded between 0 and 1.}
        \label{fig:outbound_sub4}
    \end{subfigure}
    \hfill
    \begin{subfigure}[t]{0.6\textwidth}
        \includegraphics[width=\linewidth]{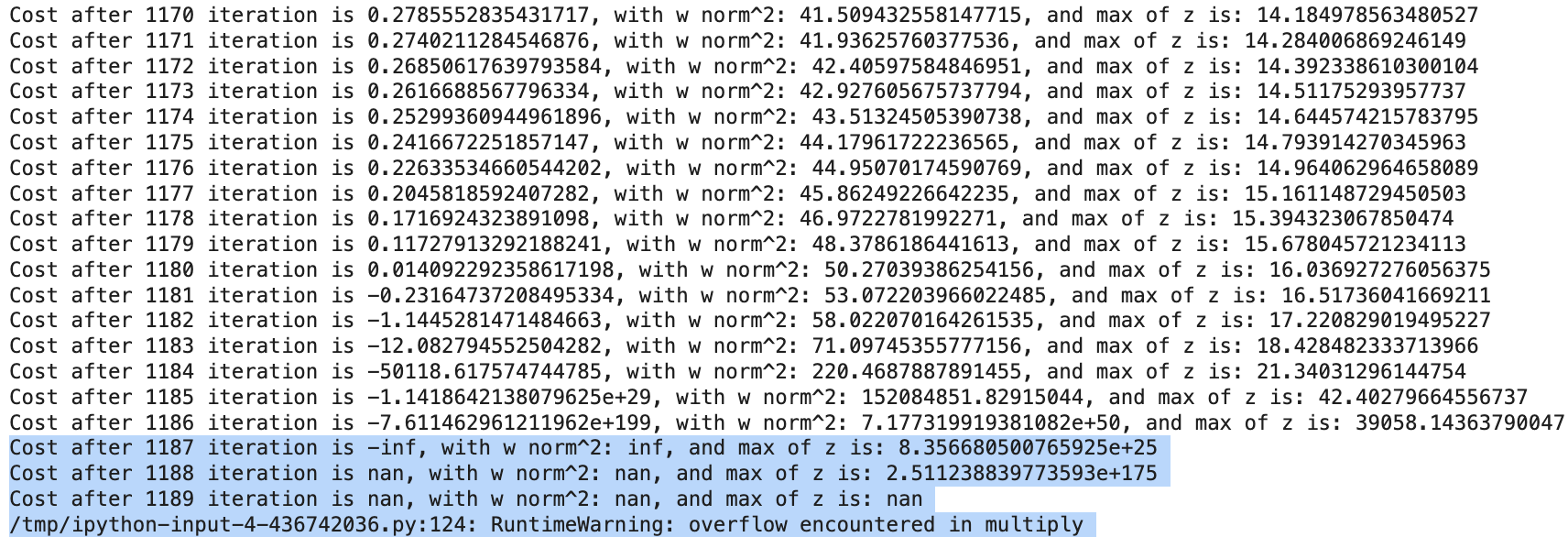}
        \caption{Due to the unbounded growth of the linear hypothesis value $z$ in the later training iteration, value $\tilde{p}'(z)$ diverges, and so does gradient and training loss, leading to a runtime error with \texttt{nan}.}       \label{fig:outbound_sub5}
    \end{subfigure}
    \caption{We trained on the Adult dataset using a polynomial approximation over the range $[-7, 7]$. However, as training progresses, the linear hypothesis value $z = \langle \vec{w}, \vec{x} \rangle$ sometimes exceeds this range. Since the polynomial $\tilde{p}'(z)$ is only accurate within $[-7, 7]$, evaluating it outside this range causes it to diverge, eventually leading to runtime errors as $\tilde{p}'(z) \to \infty$.
}
    \label{fig:outbound_No-DP}
\end{figure}

\begin{figure}[htbp]
    \centering
    \begin{subfigure}[t]{0.3\textwidth}
        \includegraphics[width=\linewidth]{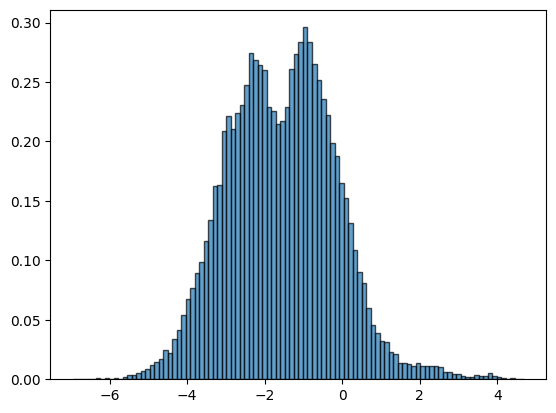}
        \caption{z values at t=1080}
        \label{fig:inbound_sub1}
    \end{subfigure}
    \hfill
    \begin{subfigure}[t]{0.3\textwidth}
        \includegraphics[width=\linewidth]{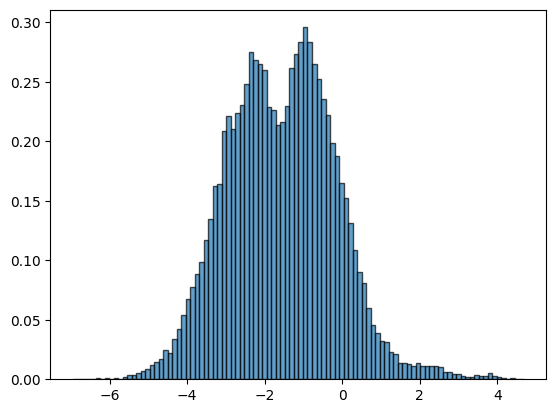}
        \caption{z values at t=1086}
        \label{fig:inbound_sub2}
    \end{subfigure}
    \hfill
    \begin{subfigure}[t]{0.3\textwidth}
        \includegraphics[width=\linewidth]{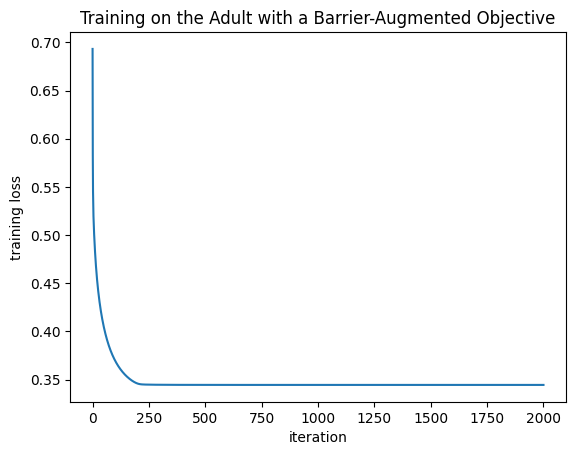}
        \caption{training loss converge over training with barrier-augmented objective}
    \label{fig:inbound_sub3}
    \end{subfigure}
\caption{We trained on the Adult dataset using the same polynomial approximation, but with a modified objective function incorporating a barrier function (as in Algorithm~\ref{fig:modified_grad_desc}). Throughout the whole training process, the linear hypothesis value $z$ remained bounded and and the training successfully converged.}
\label{fig:inbound_No-DP}
\end{figure}

A comparison between Figure~\ref{fig:outbound_No-DP} and Figure~\ref{fig:inbound_No-DP} clearly demonstrates the advantage of training with a barrier-augmented objective function (as in Algorithm~\ref{fig:modified_grad_desc}). Both experiments were conducted using the same dataset, polynomial approximation, learning rate, and number of training steps. The only distinction is that the model in Figure~\ref{fig:inbound_No-DP} was trained using an objective that incorporates a barrier function.
The results show that the barrier-augmented objective ensures stable convergence. As illustrated in Figure~\ref{fig:inbound_sub3}, the training process converges reliably under this formulation. Specifically, at training steps 1080 and 1086, the inner product values $z = \langle \vec{w}_i, \vec{x} \rangle$ remain bounded within the interval $[-7, 7]$, as shown in Figures~\ref{fig:inbound_sub1} and~\ref{fig:inbound_sub2}, respectively.
In contrast, training on the standard (non-barrier) objective exhibits unbounded growth in $z$. At step 1080, the values expand to the range $[-15, 10]$ (Figure~\ref{fig:outbound_sub2}), and by step 1086, they diverge dramatically, reaching nearly $-4000$ (Figure~\ref{fig:outbound_sub3}). 

\newpage
\subsection{Data-Independent Selection of Polynomial Approximations and Hyperparameters}
\label{sec:guide_modified_GD}

We now describe a fully data-independent procedure for selecting the polynomial approximations and training hyperparameters for the training procedure in Algorithm~\ref{fig:modified_grad_desc}, for our target application of logistic regression. This procedure ensures that all theoretical conditions are satisfied and can be executed \emph{a priori}, without access to the underlying dataset.
The sequence of steps presented in this section is just one possible way to select the parameters and approximations.  While Theorem~\ref{th:dp_no_clipping} provide the constraints that determine feasible values for guaranteeing DP, alternative orders or strategies for adjusting $\Theta$, $\lambda$, and $\kappa$, and the polynomial approximations are equally valid, as long as all conditions in theorem are eventually satisfied. This flexibility allows practitioners to adopt different sequences or heuristics for parameter selection without affecting the correctness or differential privacy guarantees of Algorithm~\ref{fig:modified_grad_desc}.

We begin by initializing the learning rate following smoothness-based analysis, we initialize the learning rate as
$
\eta_{\ref{fig:modified_grad_desc}} \approx \frac{1}{2\beta_{\text{approx}}},
$
where $\beta_{\text{approx}} = m/4$. We also fix a target approximation tolerance $E_f$ for $\phi'$, which is the sigmoid function for the case of logistic regression (e.g. in practice, we initialize $E_f = 0.05$.)

Next, we select the parameters governing the barrier function, namely the threshold $\Theta$ and the penalty coefficient $\lambda$. The threshold $\Theta$ determines the maximum admissible magnitude of 
$\|\vec{w}_i\|$, and therefore the maximum magnitude of the
inner products $|\langle \vec{w}_i, \vec{x} \rangle|$. Thus, $\Theta$ defines the approximation interval required by Theorem~\ref{th:dp_no_clipping}, while $\lambda$ controls the strength of the barrier penalty. Larger values of $\Theta$ and smaller values of $\lambda$ generally improve model accuracy, whereas smaller $\Theta$ and larger $\lambda$ enforce stricter stability. In practice, we initialize $\Theta$ to the feature dimension and set $\lambda$ to a small value (e.g., $0.001$). We then verify whether this choice satisfies constraint~(\ref{eq:kappa_constraint}) in Theorem~\ref{th:dp_no_clipping}. A larger $\Theta$ increases the approximation interval, which in turn raises the polynomial approximation error for a fixed degree. Conversely, if $\lambda$ is too small, the RHS of~(\ref{eq:kappa_constraint}) increases and may violate the constraint. If the constraint is not satisfied, we iteratively decrease $\Theta$ and increase $\lambda$ until the inequality holds. This adjustment process is entirely data-independent and requires no training.

We select the parameter $\kappa$ simultaneously with $\Theta$ and $\lambda$. The parameter $\kappa$ specifies the interval on which the polynomial $P_\kappa$ must approximate the derivative of the barrier function, $1/x$. Although decreasing $\kappa$ increases the LHS and decreases the RHS (by increasing $\mathsf{m}_P$) of~(\ref{eq:kappa_constraint}) and facilitates feasibility, doing so may increase approximation error. To balance these effects, we adopt an adaptive strategy: we initialize $\kappa$ to a relatively large value (e.g., $0.01$), which yields smaller approximation error $(e_B)$, and then progressively decrease $\kappa$ until constraint~(\ref{eq:kappa_constraint}) is satisfied.

With $\kappa$ fixed, we select the corresponding polynomial $P_\kappa$ and compute the true approximation error $e_B$, defined as the maximum approximation error of $P_\kappa$ from $1/x$ over the interval $[\kappa\Theta, \Theta]$. Several methods can be used to construct $P_\kappa$, including Newton-Raphson, Taylor series, or least-squares approximation. The goal is to approximate $1/x$ accurately over $[\kappa\Theta, \Theta]$, keeping $e_B$ small. This $e_B$ error is then used to determine the approximation interval for the sigmoid function, whose formula is defined in the first point in Theorem~\ref{th:dp_no_clipping}. A smaller $e_B$ narrows the required approximation interval for 
$\phi'$, corresponding to
the sigmoid function in logistic regression, which in turn reduces the approximation error $e_f$ for a given polynomial degree. Additionally, the polynomial should be monotone decreasing over the interval $[\Theta - R^2, \kappa\Theta]$, with the left tail kept as low as possible. This ensures that the maximum value, denoted by $\mathsf{M}_P$, over this interval does not become excessively large, which allows a larger learning rate while maintaining stability. The Newton-Raphson algorithm can achieve very high accuracy but often requires a higher polynomial degree, whereas Taylor series approximations can work with low-degree polynomials but sacrifice the accuracy. The choice of method can be made based on the desired trade-off between computational efficiency and precision. For simplicity and sufficient precision, we adopt a least-squares approach. In our experiments, a polynomial of degree 4 obtained via the least-squares method was adequate across all datasets. Using $P_\kappa$, we compute the maximum and minimum values attained over the interval $[\Theta - R^2, \kappa\Theta]$, denoted by $\mathsf{M}_P$ and $\mathsf{m}_P$, respectively. These quantities are used to verify whether the current learning rate satisfies condition (\ref{eq:eta_rest}) in Theorem~\ref{th:dp_no_clipping}. If the condition is not satisfied, the learning rate is reduced, and the procedure restarts from the step of computing the approximation interval. In our experiments, the initially chosen learning rates were typically sufficient to satisfy this condition, so few adjustments were required.

Finally, we select the polynomial degree for the sigmoid approximation. For each candidate degree, we verify three conditions: (i) the resulting approximation error $e_f$ does not exceed the preset tolerance $E_f$; (ii) the stability parameter $\alpha = 2\eta_{\ref{fig:modified_grad_desc}}\lambda\mathsf{m}_P$ satisfies $\alpha \leq 1$; and (iii) the constraint (\ref{eq:kappa_constraint})
holds. If any of these conditions are violated, we return to the barrier-parameter selection step-adjusting $\kappa, \Theta, \lambda$ as necessary, and repeat the procedure.

These three parameters can be adjusted in combination: increasing 
$\lambda$, or decreasing $\kappa$ and $\Theta$ pushes the system closer to satisfying the constraint (\ref{eq:kappa_constraint}). In practice, when the LHS and RHS of (\ref{eq:kappa_constraint}) are already close (e.g., within a ratio of 100), it is generally sufficient to keep $\lambda$ and $\Theta$ fixed and adjust only $\kappa$, which simplifies the tuning procedure.

This process yields a complete set of polynomial approximations and hyperparameters that can be fixed \emph{a priori}. As a result, it enables stable, non-interactive training under fully homomorphic encryption while satisfying all theoretical constraints.

\newpage

\section{Additional Experimental Results}

\begin{table}[th]
\caption{Accuracy and AUC for different models over 100 runs  across four (4) datasets under a privacy budget of ($\varepsilon=1$, $\delta=10^{-5}$). Results are reported as mean $\pm$ standard deviation. 
Model 1: standard DP-GD; 
Model 2: augmented no-clipping DP-GD, where the sigmoid and barrier functions are respectively replaced by a 7th or 9th-degree polynomial and a 4th-degree polynomial; 
Model 3: DP model using output perturbation.}
\label{tb:no-clipping_nofhe}
\centering
\begin{tabular}{llcc}
\toprule
\textbf{Dataset} & \textbf{Model} & \textbf{Accuracy} & \textbf{AUC} \\
\midrule
\multirow{3}{*}{mnist} & model1\_dpgd        & 0.9579 $\pm$ 0.0004 & 0.9898 $\pm$ 0.0001 \\
                        & model2\_noclipping  & 0.9549 $\pm$ 0.0038 & 0.9890 $\pm$ 0.0010 \\
                        & model3\_output\_GD  & 0.7204 $\pm$ 0.0937 & 0.8687 $\pm$ 0.0716 \\
\midrule
\multirow{3}{*}{adult} & model1\_dpgd        & 0.7947 $\pm$ 0.0001 & 0.8692 $\pm$ 0.0002 \\
                        & model2\_noclipping  & 0.8050 $\pm$ 0.0015 & 0.8750 $\pm$ 0.0014 \\
                        & model3\_output\_GD  & 0.7660 $\pm$ 0.0384 & 0.8191 $\pm$ 0.0222 \\
\midrule
\multirow{3}{*}{compas} & model1\_dpgd       & 0.6486 $\pm$ 0.0028 & 0.7034 $\pm$ 0.0015 \\
                         & model2\_noclipping & 0.6325 $\pm$ 0.0150 & 0.6929 $\pm$ 0.0121 \\
                         & model3\_output\_GD & 0.5126 $\pm$ 0.0418 & 0.5490 $\pm$ 0.0741 \\
\midrule
\multirow{3}{*}{credit} & model1\_dpgd       & 0.7795 $\pm$ 0.0000 & 0.6411 $\pm$ 0.0003 \\
                         & model2\_noclipping & 0.7828 $\pm$ 0.0016 & 0.6501 $\pm$ 0.0022 \\
                         & model3\_output\_GD & 0.7657 $\pm$ 0.0623 & 0.6012 $\pm$ 0.0355 \\
\bottomrule
\end{tabular}
\end{table}

\begin{table}[!htbp]
\centering
\captionof{table}{Multiplicative depth cost per iteration for the DP training approach with and without gradient clipping and for Output-GD, which runs no-DP under FHE}
\label{tab:depth-cost_analysis}
\begin{tabular}{l c c c}
\toprule
\textbf{Computational Step} 
& \textbf{DP-GD (w/ Clip)} 
& \textbf{DP-GD (No Clip)} 
& \textbf{Output-GD} \\
\midrule
\multicolumn{4}{l}{\textit{Path A: Gradient Calculation (Per-sample)}} \\
\quad $\vec{z} = \langle \vec{w}, \vec{x}\rangle$ (Inner Product) & +1 & +1 & +1 \\
\quad Sigmoid Poly & +4 & +4 & +4 \\
\quad Gradient $((y_{pred}-y)*\vec{x})$ & +2 & +2 & +2 \\
\quad \textit{--- Per-Sample Clipping ---} & & & \\
\quad \quad $||\text{grad}_i||^2$ (Norm) & +1 & \text{---} & \text{---} \\
\quad \quad $\sqrt{\text{norm}}$ (Sqrt Poly) & +4 & \text{---} & \text{---} \\
\quad \quad $\text{max} \{C, ||\text{grad}_i||\}$ & +6 & \text{---} & \text{---} \\
\quad \quad Inverse $\sqrt{\text{norm}}$ (using TS) & +3 & \text{---} & \text{---} \\
\quad \quad Scaling Logic ($\text{grad}_i$*scal) & +1 & \text{---} & \text{---} \\
\quad Average Gradient & +1 & +1 & +1 \\
\midrule
\multicolumn{4}{l}{\textit{Path B: Barrier Calculation (Parallel to A)}} \\
\quad $||\vec{w}||^2$ (Squared Norm) & \text{---} & +1 & \text{---} \\
\quad Inverse Term (using LS) & \text{---} & +3 & \text{---} \\
\quad Barrier Term & \text{---} & +2 & \text{---} \\
\midrule
\quad Critical Path Depth & 23 & $\max(8, 6)=8$ & 8 \\
\quad Weight Update (grad * lr) & +1 & +1 & +1 \\
\midrule
\textbf{Total Cost per Iteration} & \textbf{24} & \textbf{9} & \textbf{9} \\
\bottomrule
\end{tabular}
\end{table}

\clearpage
\section{Protocols}
\label{sec:protocols}
\begin{algorithm}[!htbp]
\caption{Traditional DP-(S)GD Protocol using FHE}
\label{fig:trad_DP-GD_fhe}
\begin{algorithmic}
\STATE {\bfseries Input:} Private data $(X, y)$, learning rate $\eta$, epochs $T$, clipping norm $C = 1$, DP noise vectors $\{z_t\}_{t=1}^T$, sigmoid approximation function $P_{\text{sigmoid}}$, BatchSize $n$\\
\STATE \textbf{// Client Side: Key Generation}
\STATE $(\texttt{pk}, \texttt{sk}) \gets \texttt{HE.KeyGen()}$

\STATE \textbf{// Client Side: Data Encryption}
\STATE $\texttt{Enc\_X} \gets \texttt{HE.Encrypt}(\texttt{pk}, X)$
\STATE $\texttt{Enc\_y} \gets \texttt{HE.Encrypt}(\texttt{pk}, y)$
\STATE $\texttt{Enc\_z} \gets \texttt{HE.Encrypt}(\texttt{pk}, \{z_t\}_{t=1}^T)$
\STATE \texttt{Send } $(\texttt{Enc\_X}, \texttt{Enc\_y}, \texttt{Enc\_z})$ \texttt{ to Server}
\STATE \textbf{// Server Side: Encrypted Training}
\STATE $\texttt{Enc\_w} \gets \texttt{HE.Encrypt}(\texttt{pk}, 0^d)$

\FOR{$t = 1$ to $T$}
    \STATE $\texttt{Enc\_grad\_sum} \gets \texttt{HE.Encrypt}(\texttt{pk}, 0^d)$\\
    \STATE $I_t \gets \text{SampleRandomIndices}(N, n)$ \\
    \FOR{$i$ in $I_t$}
        \STATE $\texttt{Enc\_logit}_i \gets \langle \texttt{Enc\_X}[i], \texttt{Enc\_w} \rangle$
        \STATE $\texttt{Enc\_pred}_i \gets P_{\text{sigmoid}}(\texttt{Enc\_logit}_i)$
        \STATE $\texttt{Enc\_error}_i \gets \texttt{Enc\_pred}_i - \texttt{Enc\_y}[i]$
        \STATE $\texttt{Enc\_grad}_i \gets \texttt{Enc\_X}[i] \cdot \texttt{Enc\_error}_i$
        \STATE $\texttt{Enc\_grad}_i \gets P_{\text{clip}}(\texttt{Enc\_grad}_i, C)$
        \STATE $\texttt{Enc\_grad\_sum} \gets \texttt{Enc\_grad\_sum} + \texttt{Enc\_grad}_i$
    \ENDFOR
    \STATE $\texttt{Enc\_grad} \gets \frac{1}{n} \cdot \texttt{Enc\_grad\_sum} + \texttt{Enc\_z}[t]$
    \STATE $\texttt{Enc\_w} \gets \texttt{Enc\_w} - \eta \cdot \texttt{Enc\_grad}$
\ENDFOR
\STATE \textbf{// Server sends } $\texttt{Enc\_w}$ \texttt{ to Client}

\STATE \textbf{// Client Side: Decryption}
\STATE $w \gets \texttt{HE.Decrypt}(\texttt{sk}, \texttt{Enc\_w})$

\STATE {\bfseries Output:} Trained model weights $w$
\end{algorithmic}
\end{algorithm}

\clearpage
\begin{algorithm}[!htbp]
\caption{Proposed non-clipping DP-(S)GD protocol}
\label{fig:modified_DP-GD_fhe}
\begin{algorithmic}
\STATE {\bfseries Input:} Private data $(X, y)$, learning rate $\eta$, epochs $T$, DP noise vectors $\{z_t\}_{t=1}^T$, sigmoid approximation function $P_{\text{sigmoid}}$;  barrier function's parameters: $\lambda$ and $\Theta$, inverse approximation function $P_{\text{rep}}$, BatchSize $n$ \\
\textbf{Output:} Trained model weights $w$
\STATE \textbf{// Client Side: Key Generation}
\STATE $(\texttt{pk}, \texttt{sk}) \gets \texttt{HE.KeyGen()}$

\STATE \textbf{// Client Side: Data Encryption}
\STATE $\texttt{Enc\_X} \gets \texttt{HE.Encrypt}(\texttt{pk}, X)$
\STATE $\texttt{Enc\_y} \gets \texttt{HE.Encrypt}(\texttt{pk}, y)$
\STATE $\texttt{Enc\_z} \gets \texttt{HE.Encrypt}(\texttt{pk}, \{z_t\}_{t=1}^T)$
\STATE \texttt{Send } $(\texttt{Enc\_X}, \texttt{Enc\_y}, \texttt{Enc\_z})$ \texttt{ to Server}

\STATE \textbf{// Server Side: Encrypted Training}
\STATE $\texttt{Enc\_w} \gets \texttt{HE.Encrypt}(\texttt{pk}, 0^d)$

\FOR{$t = 1$ to $T$}
    \STATE $\texttt{Enc\_grad\_sum} \gets \texttt{HE.Encrypt}(\texttt{pk}, 0^d)$\\
    \STATE $I_t \gets \text{SampleRandomIndices}(N, n)$ \\
    \FOR{$i$ in $I_t$}
        \STATE $\texttt{Enc\_logit}_i \gets \langle \texttt{Enc\_X}[i], \texttt{Enc\_w} \rangle$
        \STATE $\texttt{Enc\_pred}_i \gets P_{\text{sigmoid}}(\texttt{Enc\_logit}_i)$
        \STATE $\texttt{Enc\_error}_i \gets \texttt{Enc\_pred}_i - \texttt{Enc\_y}[i]$
        \STATE $\texttt{Enc\_grad}_i \gets \texttt{Enc\_X}[i] \cdot \texttt{Enc\_error}_i$
        \STATE $\texttt{Enc\_grad\_sum} \gets \texttt{Enc\_grad\_sum} + \texttt{Enc\_grad}_i$
    \ENDFOR
    \STATE  $\texttt{Enc\_barrier} \gets P_{rep}(\Theta-\| \texttt{Enc\_w}\|^2)$
    \STATE $\texttt{Enc\_grad} \gets \frac{1}{n} \cdot \texttt{Enc\_grad\_sum}+2\cdot \lambda \cdot \texttt{Enc\_w}\cdot \texttt{Enc\_barrier}+\texttt{Enc\_z}[t]$
    \STATE $\texttt{Enc\_w} \gets \texttt{Enc\_w} - \eta \cdot \texttt{Enc\_grad}$
\ENDFOR

\STATE \textbf{// Server sends } $\texttt{Enc\_w}$ \texttt{ to Client}

\STATE \textbf{// Client Side: Decryption}
\STATE $w \gets \texttt{HE.Decrypt}(\texttt{sk}, \texttt{Enc\_w})$

\STATE {\bfseries Output:} $w$
\end{algorithmic}
\end{algorithm}


\end{document}